\documentclass{article}

\usepackage{amsthm}
\usepackage{amssymb}
\usepackage{amstext}
\usepackage{amsmath}
\usepackage{mathrsfs}
\usepackage{url}
\usepackage[dvipsnames]{xcolor}

\usepackage[colorlinks, bookmarks=false]{hyperref}
\hypersetup{
linkcolor=blue,
citecolor=blue,
pdfauthor={Name},
}
\usepackage[margin = 40pt,font=small,labelfont=bf]{caption}
\usepackage{arydshln}
\usepackage{pdfsync}
\usepackage{inputenc}
\usepackage[T1]{fontenc}
\usepackage{color}
\usepackage{algorithmic} 
\usepackage[ruled,vlined]{algorithm2e}
\usepackage{enumerate}
\usepackage{bbm}
\usepackage{amsthm}
\usepackage{amssymb}
\usepackage{amstext}
\usepackage{amsmath}
\usepackage{amsthm}
\usepackage{amssymb}
\usepackage{amstext}
\usepackage{amsmath}
\usepackage{mathrsfs}
\usepackage{url}
\usepackage[dvipsnames]{xcolor}

\usepackage{mathrsfs}
\usepackage{url}
\usepackage[dvipsnames]{xcolor}

\usepackage{pdfpages}
\usepackage{epstopdf}
\usepackage{graphicx}
\usepackage{epsfig}
\usepackage{subfigure}
\usepackage{authblk}
\usepackage{listings}
\usepackage[numbers]{natbib}

\newlength{\dhatheight}

\numberwithin{equation}{section}
\theoremstyle{plain}
\newtheorem{thm}{Theorem}[section]
\newtheorem{rem}{Remark}[section]
\newtheorem{prop}{Proposition}[section]
\newtheorem{cor}{Corollary}[section]
\newtheorem{lem}{Lemma}[section]
\newtheorem{deff}{Definition}[section]
\newtheorem{hyp}{Assumption}[section]

\def\build#1_#2^#3{\mathrel{\mathop{\kern 0pt#1}\limits_{#2}^{#3}}}

\newcommand{\KL}{\mathop{\mathrm{KL}}}

\DeclareMathOperator{\Id}{Id}
\def\videbox{\mathbin{\vbox{\hrule\hbox{\vrule height1.4ex \kern.6em\vrule height1.4ex}\hrule}}}

\newcommand{\ds}{\displaystyle}

\def\argmin{\mathop{\rm arg \; min}\limits}%

\newcommand{\thefont}[2]{\fontsize{#1}{#2}\fontshape{n}\selectfont}
\newcommand{\1}{\rlap{\thefont{10pt}{12pt}1}\kern.16em\rlap{\thefont{11pt}{13.2pt}1}\kern.4em}

\title{On the potential benefits of entropic regularization for smoothing Wasserstein estimators \vspace{1ex}}

\author[1,2]{J\'er\'emie Bigot}
\author[1,2,5]{Paul Freulon \footnote{correspondence: paul.freulon@epfl.ch}}
\author[1,3,4]{Boris P.\ Hejblum}
\author[1,2,6]{Arthur Leclaire}

\affil[1]{\small Universit\'{e} de Bordeaux, Bordeaux, 33000, France.}
\affil[2]{Institut de Math\'ematiques de Bordeaux et CNRS  (UMR 5251), 33400 Talence, France.}
\affil[3]{Bordeaux Population Health Research Center Inserm U1219, Inria SISTM, 33000 Bordeaux, France.}
\affil[4]{Vaccine Research Institute (VRI), 94010 Cr\'eteil, France.}
\affil[5]{EPFL, Institute of Mathematics,  CH-1015 Lausanne, Switzerland.}
\affil[6]{LTCI, Télécom Paris, IP Paris, 19 place Marguerite Perey, 91120 Palaiseau, France.}
\date{}

\begin{document}

\maketitle

\begin{abstract}
	This paper focuses on the study of entropic regularization in optimal transport as a smoothing method for Wasserstein estimators, through the prism of the classical trade-off between approximation and estimation errors in statistics. Wasserstein estimators are defined as solutions of variational problems whose objective function involves the use of an optimal transport cost between probability measures. Such estimators can be regularized by replacing the optimal transport cost  by its regularized version using an entropy penalty on the transport plan. The use of such a regularization has a potentially significant smoothing effect on the resulting estimators. In this work, we investigate its potential benefits on the approximation and estimation properties of regularized Wasserstein estimators. Our main contribution is to discuss how entropic regularization may reach, at a lower computational cost,  statistical performances that are comparable to those of un-regularized  Wasserstein estimators in statistical learning problems involving distributional data analysis. To this end, we present new theoretical results on the convergence of regularized Wasserstein estimators. We also study their numerical performances using simulated and real data in the supervised learning problem of proportions estimation in mixture models using optimal transport.
\end{abstract}


\section{Introduction}\label{sec:Introduction}

Wasserstein estimators are defined as solutions of variational problems whose objective function involves the use of an optimal transport (OT) cost between probability measures. Such estimators typically arise in statistical problems involving the minimization of a Wasserstein distance (or more generally an OT cost) between the empirical measure of the data and a distribution belonging to a parametric model (see \cite{bernton2019parameter}), and this class of estimators has found important applications in generative adversarial models for image processing (see e.g. \cite{arjovsky2017wgan}).  Wasserstein estimators also represent an important class of inference methods in the field of statistical optimal transport for distributional data analysis where the observations at hand can be modeled as a set of histograms (see e.g.\ \cite{bigotReview,Pana18,petersen} for  recent reviews).

Despite the appealing geometric properties of Wasserstein distances for comparing  probability distributions, the computational burden required to evaluate an optimal transport cost is an important limitation for its application in data analysis. The seminal paper \cite{cuturi2013sinkhorn} has opened a breach in the computational complexity of optimal transport by the addition of an entropic regularizing term  in the OT Kantorovich's formulation. In the last years, the benefit of this regularization  has been to allow the use of OT based methods in statistics and machine learning with a time complexity that scales  quadratically in the number of data using the Sinkhorn algorithm. This represents a significant improvement over the computational cost of un-regularized OT that scales cubically in the number of observations using linear programming. However, regularized OT has been  mainly used so far  as a fast numerical method to approximate un-regularized OT.

In this paper, we advocate the use of entropic regularization  in computational OT as a smoothing method for un-regularized  Wasserstein estimators.   These estimators are obtained by replacing the standard OT cost in a variational problem by its entropy regularized version. The use of such a regularization  has a  beneficial smoothing effect on the resulting estimators as shown in \cite{BCP18} for the specific problem of computing a smooth Wasserstein barycenter from a set of discrete probability measures. In this paper, we investigate the impact of this smoothing effect  of regularized Wasserstein estimators through the prism of the   tradeoff between approximation and estimation errors   in statistics which is reminiscent of the classical bias versus variance tradeoff).  Our main contribution is to discuss how entropic regularization  yields estimators that may reach,  at a lower computational cost,  statistical performances that are comparable to those of un-regularized  Wasserstein estimators  in statistical learning problems involving distributional data analysis. To this end, we present new theoretical results on the convergence of regularized  Wasserstein estimators. We also study their numerical performances using simulated and real data in the supervised learning problem of  proportions estimation in mixture models using optimal transport.
\subsection{Proportions estimation in mixture models using optimal transport}\label{sec:prop}

The motivation of this work comes from the active research field of automated analysis of flow cytometry measurements, see \cite{aghaeepour2013critical}. Flow cytometry is a high-throughput biotechnology used to characterize a large amount of  $m$ cells from a biological sample (with $m \geq 10^5$) that produces a data set $X_1,\dots,X_m$ where each observation $X_i \in \mathbb{R}^d$ corresponds to a vector of $d$ biomarkers of each single cell. Automated approaches in flow cytometry aim at clustering the data in order to estimate cellular population proportions in the biological sample. In \cite{freulon2020cytopt},  we have considered that such a data set can be represented with a discrete probability distribution $\frac{1}{m} \sum_{i=1}^{m} \delta_{X_i}$ with support in $\mathbb{R}^d$, and we have introduced a new supervised algorithm based on regularized OT to  estimate the different cell population proportions from a biological sample characterized with flow cytometry measurements. This approach optimally re-weights class proportions in a mixture model between a source data set (with known segmentation into cell sub-populations) to fit a target data set with unknown segmentation. 

Most automated methods in flow cytometry cluster the observations \cite{cheung2021current}.  However, the relevant clinical information is rather the class proportions, i.e. the cell population relative abundance. For instance, when monitoring the immune system of patient, clinicians are more interested by the proportion of CD4+ T-cells than to know whether each cell is a CD4+ T-cell.

To be more precise, let us denote by $Y_1,\ldots,Y_n$, the dataset  from the target sample, and by $X_1,\ldots,X_m$ the observations from the source biological sample. Thanks to the knowledge of a clustering  of the source dataset into $K$ classes $C_1,\ldots,C_{K}$, the empirical measure $\hat{\mu} = \frac{1}{m} \sum_{i=1}^{m} \delta_{X_i}$  can be decomposed as the following mixture of probability measures, 
\begin{equation}\label{Source_as_mixture}
	\hat{\mu} = \sum_{k=1}^K \frac{m_k}{m} \left( \sum_{i: X_i \in C_k} \frac{1}{m_k} \delta_{X_i} \right) = \sum_{k=1}^K \hat{\pi}_{k} \hat{\mu}_k, \mbox{ where } \hat{\pi}_{k} =  \frac{m_k}{m}, 
\end{equation}
and each component $\hat{\mu}_k = \sum_{i: X_i \in C_k} \frac{1}{m_k} \delta_{X_i}$ corresponds to a known sub-population of cells with $m_k = \#C_k$.  Then, the method proposed in   \cite{freulon2020cytopt} aims at modifying the weights $(\hat{\pi}_{k})_{1 \leq k \leq K}$ in such a way that the re-weighted source measure minimizes a regularized OT cost with respect to the target measure $ \frac{1}{n} \sum_{j=1}^{n} \delta_{Y_j}$. Then, the resulting weights yield an estimation of the proportions of sub-population of cells in the target sample. However, despite the efficiency of the method for the analysis of  flow cytometry  data, the work in  \cite{freulon2020cytopt} opens questions on the influence of the regularization, and we set to answer some of them in this work. 

Let us now formalize the problem of proportions estimation in mixture models using regularized OT. We denote by $\mu = \sum_{k=1}^K \pi_k \mu_k$ a probability measure that can be decomposed as a mixture of $K$ probability measures $\mu_1,\ldots,\mu_K$.  For $\theta \in \Sigma_K$, where $\Sigma_K = \{(\theta_1,\ldots,\theta_K)\in \mathbb{R}_{+}^{K} \; : \; \sum_{k=1}^{K} \theta_{k} = 1 \} $ is the $K$-dimensional simplex, we define $\mu_\theta$ as the re-weighted version of $\mu$ that is defined as
\begin{equation}\label{eq:definition_mu_theta}
	\mu_{\theta} =  \sum_{k=1}^{K} \theta_k \mu_k.
\end{equation}
Let $\nu$ be another probability measure. Proportions estimation in mixture models using OT  is defined as the problem of finding $\theta^{\ast} \in \Sigma_K$ that minimizes an  OT cost between $\mu_\theta$ and $\nu$. Denoting by $W_0(\mu, \nu)$ the un-regularized OT cost between $\mu$ and $\nu$ (we shall focus on the squared Wasserstein metric associated to the quadratic cost), the optimal vector of class proportions that we are targeting is:
$$
\theta^* \in \argmin_{\theta \in \Sigma_K} W_{0}(\mu_{\theta},\nu).
$$
In practice, one only has access to independent samples from $\mu$ and $\nu$  denoted by $X_1,\ldots,X_m$ (with a know clustering) and $Y_1,\ldots,Y_n$ respectively. Therefore, estimators of $\theta^*$ will be obtained from the empirical versions of $\mu_\theta$ and $\nu$ denoted by
$$
\hat{\mu}_\theta = \sum_{k=1}^K \theta_{k} \hat{\mu}_k \quad \mbox{and} \quad \hat{\nu} =  \frac{1}{n} \sum_{j=1}^{n} \delta_{Y_j}.
$$
The computational cost to numerically evaluate $W_0(\mu, \nu)$ can be prohibitive, which led the author of \cite{freulon2020cytopt} to consider its regularized version  denoted by $W_\lambda(\mu, \nu)$ where $\lambda > 0$ represents the amount of entropic penalty that is put on the transport plan in the primal formulation of OT. Here, this regularized version of the OT cost is computed using the Sinkhorn algorithm, an iterative procedure whose convergence properties are now well understood (see \cite{chizat2020sinkhorn} for a recent overview). However, after $\ell$ iterations of the Sinkhorn algorithm, it should be noted that one only has an approximation of the regularized OT cost that we will denote by $W_\lambda^{(\ell)}(\mu, \nu)$. In this work, we focuses on the study of the convergence rate of the following  estimator  towards the  optimal vector of class proportions  $\theta^*$: 
\begin{equation}\label{eq:def_sinkhorn_estimate}
	\hat{\theta}^{(\ell)}_\lambda = \argmin_{\theta \in \Sigma_K}W_\lambda^{(\ell)}(\hat{\mu}_\theta, \hat{\nu}).
\end{equation}
This takes into account both the effect of entropic regularization and the influence of the number of iterations of the Sinkhorn algorithm. Our theoretical results shed some light on how the parameters $\lambda$ and $\ell$ influence the performance of the estimator $\hat{\theta}^{(\ell)}_\lambda$. We demonstrate the practical efficiency of our method and the impact of the regularization parameter $\lambda$ on simulated and real data (flow cytometry measurements). 
\subsection{Related works based on regularized optimal transport}

Aside from the computational benefits of entropic regularization mentioned previously, recent developments have studied the statistical properties of a regularized  OT cost computed from empirical measures. Indeed, in most cases, $\mu$ and $\nu$ are not available, and one has only access to their empirical versions $\hat{\mu}$ and $\hat{\nu}$ respectively built from $X_1,\ldots,X_n$  sampled from $\mu$ and $Y_1,\ldots,Y_n$ sampled from $\nu$. In this setting, it is natural to investigate the convergence rate of the \textit{plug-in estimator} $W_0(\hat{\mu}, \hat{\nu})$ towards $W_0(\mu, \nu)$. This question is addressed in \cite{fournier2015rate} where the authors proved that the resulting estimation error decays to zero at the rate $n^{-2/d}$ when using the quadratic cost in high dimension $d$. Due to its attractive computational efficiency, it is obviously interesting to examine the statistical efficiency of the regularized Wasserstein \textit{plug-in estimator} naturally defined as $W_\lambda(\hat{\mu}, \hat{\nu})$. This issue as well as the approximation error induced by the regularization parameter is studied in \cite{genevay2019sample}. These questions are thoroughly pursued in \cite{chizat2020sinkhorn} as well as the effect of substituting  $W_\lambda(\mu, \nu)$ by its debiased counterpart $S_\lambda(\mu, \nu)$. 
Since the first iteration of our work was posted on ArXiv, some new results on the rate of convergence of $W_\lambda(\hat{\mu}, \hat{\nu})$ toward $W_\lambda(\mu,\nu)$ have been established. For instance, in \cite{stromme2024minimum,groppe2023lower}, some structural assumptions on the measures studied enable the authors to derive faster rate of convergence than what was previously known. Other works like \cite{bigot2019CLT_EOT}, and more recently \cite{gonzalez2023weak}, study the limiting distribution of entropic optimal transport estimators. The OT loss functions $W_0, W_\lambda$ and $S_\lambda$ also constitute efficient tools for statistical estimation. For instance, a framework of parametric estimation where regularized OT acts as a loss function in learning problems is considered in \cite{ballu2020stochastic}.  Regularized Wasserstein losses are also considered in \cite{genevay2017gansinkhorn,sanjabi2018wgan,liu2019wgan} for the design of generative models in image processing. In a more applied context, the use of regularized OT  is investigated in \cite{huizing2021optimal, freulon2020cytopt} to tackle  estimation problems in biostatistics.  The influence of the regularization parameter~$\lambda$ for the purpose of computing smooth Wasserstein barycenters is also analyzed in  \cite{BCP18, chizatDoublyRegularizedEntropic2023}.

\subsection{Organization of the paper}

In Section \ref{sec:OT_and_Regularized_OT} we recall  the mathematical aspects of regularized OT needed to derive our results, and we detail the problem of optimal class proportions estimation in mixture models using OT. In Section \ref{sec:wass_estimators},  we introduce the various parametric Wasserstein estimators used to estimate the optimal class proportions. We also give the main results of this paper on a theoretical comparison of the convergence rates of  regularized and un-regularized Wasserstein estimators. The influence of the number of iterations of the Sinkhorn algorithm on these convergence rates is also discussed. 
In Section \ref{sec:comparison}, we propose an alternative bound for the estimation error studied. We also compare our results with related works.
Section \ref{sec:numerical_experiments} is focused on numerical experiments that highlight the potential benefits of regularized Wasserstein estimators over un-regularized ones for  appropriate choices of the  entropic regularization parameter. Section \ref{sec:diss} contains a conclusion and some perspectives.  In the Appendix \ref{sec:convergence}, we detail the main arguments to obtain  the convergence rates of regularized and un-regularized Wasserstein estimators.

 \section{Background on optimal transport and the problem of  class proportions estimation}\label{sec:OT_and_Regularized_OT}

In this section, we introduce the notion of entropy regularized OT, and we present some of its mathematical properties needed to derive our results. Then, we describe the main application of this work on class proportions estimation in mixture models using OT. Finally, we discuss some identifiability issues in such models. 

\subsection{The OT problem and its regularized counterpart}

\paragraph{Notations.}
In the whole paper, we will work in the space $\mathbb{R}^d$ equipped with the quadratic cost $c(x,y) = \|x-y\|^2$, where $\|x\| = \sqrt{\sum x_i^2 }$ is the Euclidean norm.
Let  $\mathcal{X}$ and $\mathcal{Y}$ be two  subsets of $\mathbb{R}^d$ that {\it are assumed to be compact} and included in $B(0,R) = \{x \in \mathbb{R}^d~:~ \|x\|  \leq R\}$ throughout the paper.  We denote by $\mathcal{M}_+^1(\mathcal{X})$ and $\mathcal{M}_+^1(\mathcal{Y})$ the sets of probability measures on $\mathcal{X}$ and $\mathcal{Y}$ respectively. For $Y_1,\ldots, Y_n \sim \nu$, we denote by $\hat{\nu}$ the empirical counterpart of $\nu$ defined by $\hat{\nu}=\frac{1}{n}\sum_{i=1}^n \delta_{Y_i}$. 
The notation $\lesssim$ means inequality up to a multiplicative universal constant. For $\mu \in \mathcal{M}_{+}^1 (\mathcal{X})$ and $\nu \in \mathcal{M}_{+}^{1}(\mathcal{Y})$, we let $\Pi (\mu,\nu)$ be the set of probability measures on  $\mathcal{X} \times \mathcal{Y}$ with marginals $\mu$ and $\nu$. The problem of {\it entropic optimal transportation} between
$\mu \in \mathcal{M}_{+}^{1}(\mathcal{X})$ and $\nu \in \mathcal{M}_{+}^{1}(\mathcal{Y})$ is then defined as follows.

\begin{deff}[Primal formulation]   \label{def:primal}
	For any  $(\mu,\nu) \in \mathcal{M}_{+}^{1}(\mathcal{X}) \times \mathcal{M}_{+}^{1}(\mathcal{Y})$, the Kantorovich formulation of the regularized optimal transport between $\mu$ and $\nu$ is the following convex minimization problem 
	\begin{equation}
		\quad W_{\lambda}(\mu,\nu) = \min_{ \substack{\pi \in \Pi(\mu,\nu)} } \int_{\mathcal{X} \times \mathcal{Y}} \|x-y\|^2 d\pi(x,y) + \lambda \KL(\pi | \mu \otimes \nu), 
		\label{Primal}
	\end{equation}
	where $\|x-y\|$ is the Euclidean distance between $x$ and $y$, $\lambda \geq 0$ is 
	a  regularization parameter, and $ \KL$
	stands for the \textcolor{black}{Kullback-Leibler divergence, between $\pi$ and a positive measure $\xi$ on $\mathcal{X} \times \mathcal{Y}$, up to the  additive term $\int_{\mathcal{X} \times \mathcal{Y}}  d\xi(x,y)$, namely}
	$$
	\KL(\pi|\xi) = \int_{\mathcal{X} \times \mathcal{Y}}  \log \Bigl( \dfrac{d\pi}{d\xi}(x,y)\Bigr)  d\pi (x,y),
	$$
\end{deff}
in the case $\pi$ absolutely continuous w.r.t. $\xi$, otherwise $\KL(\pi|\xi) = +\infty$. For $\lambda = 0$, the quantity $W_{0}(\mu,\nu)$ is the {\it standard (un-regularized) OT cost},  and for $\lambda > 0$,
we  refer to $W_{\lambda}(\mu,\nu)$ as the {\it regularized OT cost} between 
$\mu$ and $\nu$. Note that the continuity of $c$ and the compactness of  $\mathcal{X}$ and $\mathcal{Y}$ imply that $W_{\lambda}(\mu,\nu)$ is finite for any value of  $\lambda \geq 0$.  Let us now introduce the dual and semi-dual formulations (see e.g.\ \cite{santambrogio2015ot,genevay2016ot})  of the minimization problem \eqref{Primal}.

\begin{thm}[Dual formulation] \label{thm:ot_dual}
	Strong duality holds for the primal problem \eqref{Primal} in the sense that
	\begin{align}
		\label{eq:ot_dual}
		W_{\lambda}(\mu,\nu)& = \sup_{\substack{\varphi \in  L^{\infty}(\mathcal{X}),\\ \psi \in L^{\infty}(\mathcal{X})}} \int \varphi(x) d \mu(x) + \int \psi(y) d\nu(y)\\
		& \quad \quad \quad - \int m_{\lambda}(\varphi(x) +\psi(y) - \|x-y\|^2) d\mu(x) d\nu(y) \nonumber
	\end{align}
	where $L^{\infty}(\mathcal{X})$ denotes the space of essentially bounded functions quotiented by a.e.\ equality, and
	$$
	m_{\lambda}(t) = \left\{\begin{array}{ccc}  +\infty 1_{\{t\geqslant 0\}} & \mbox{if} & \lambda = 0 \\  \lambda (e^{\frac{t}{\lambda}}-1) & \mbox{if} & \lambda > 0\end{array}\right.
	$$
\end{thm}

A solution $(\varphi,\psi)$ of the dual problem \eqref{eq:ot_dual} is called a pair of Kantorovich potentials. Besides,  since $\mathcal{X}, \mathcal{Y}$ are compact and  $c$ is continuous, it follows that the dual problem admits a solution $(\varphi, \psi) \in \mathscr{C}_b(\mathcal{X}) \times \mathcal{C}_b(\mathcal{Y})$. Moreover, when $\lambda > 0$, there exists solutions $\varphi, \psi$ to the dual problem~\eqref{eq:ot_dual} which are uniquely defined almost everywhere, up to an additive constant. The solutions of this regularized dual problem have the specific structure of $c$-transform functions.  For the quadratic cost $c(x,y) =\|x-y\|^2$, the regularized $c$-transform is defined as in~\cite{feydy2019sinkhorndiv}: for $\lambda > 0$, we set
\begin{equation}
	\label{eq:regularized_ctransform_explicit}
	\forall x \in \mathbb{R}^d, \quad
	\psi_{\nu}^{c,\lambda}(x) = - \lambda \log \ds\int e^{- \frac{\|x-y\|^2 - \psi(y)}{\lambda} } d\nu(y), 
\end{equation} 
and for $\lambda = 0$, the $c$-transform simply reads
\begin{equation}
	\label{eq:un-regularized_ctransform_explicit}
	\forall x \in \mathbb{R}^d, \quad
	\psi^{c}(x) = \min_{y \in \mathcal{Y}} \{\|x-y\|^2 - \psi(y) \} . 
\end{equation}
We also define the analogous operators for the $y$-variable (and for simplicity, we use the same notation for $c$-transforms of $x$-functions or $y$-functions). 
Notice that the operation used in~\eqref{eq:regularized_ctransform_explicit} can be understood as a smoothed minimum that depends on $\nu$.
Therefore, when $\lambda > 0$ we will stick to the notation $\psi_\nu^{c,\lambda}$ to keep in mind the possible dependence on $\nu$ of the regularized $c$-transform.
Notice also that, even if $\psi_{\nu}^{c, \lambda}$ will be integrated only on $\mathcal{X}$, the formulae allow to extrapolate the $c$-transforms on the whole space $\mathbb{R}^d$. In the sequel of this paper we extrapolate the $c$-transform on $B(0,R)$ to manipulate functions defined on a convex subset of $\mathbb{R}^d$ without imposing the convexity of $\mathcal{X}$ and $\mathcal{Y}$. Moreover, considering the $c$-transform amounts to optimizing one of the two potential, thus leading to an optimization problem with respect to one single potential, called the semi-dual problem \cite{genevay2016ot}:

\begin{cor}[Semi-dual formulation] \label{th:ot_semidual}
	The dual problem~\eqref{eq:ot_dual} is equivalent to the following semi-dual problem in the sense that
	\begin{equation}
		\label{eq:ot_semidual}
		W_{\lambda}(\mu,\nu) = \sup_{\psi \in L^{\infty}(\mathcal{Y})} \int \psi_{\nu}^{c,\lambda}(x) d \mu(x) + \int \psi(y) d\nu(y).
	\end{equation}
\end{cor}

A solution $\psi$ of the semi-dual problem is called a {\it Kantorovich potential}. In other words, $\psi$ is a Kantorovich potential if and only if $(\psi_{\nu}^{c,\lambda},\psi)$ is a pair of Kantorovich potentials. By symmetry, we can also formulate a semi-dual problem on the dual variable $\varphi$.
For discrete probability distributions, the iterative Sinkhorn algorithm, as defined below, (see e.g. \cite{cuturi2013sinkhorn}) allows to approximate the regularized OT cost $W_\lambda(\mu,\nu)$ as follows.

\begin{deff}\label{eq:sinkhorn_algorithm}
	For $\mu = \sum_{i=1}^n \mu_i \delta_{x_i}$and $\nu = \sum_{j=1}^m \nu_j \delta_{y_j}$ two discrete distributions on $\mathbb{R}^d$, the approximation of the regularized OT cost returned by the Sinkhorn approximation after $\ell$ iterations equals
	\begin{equation}\label{eq:sinkhorn_output_first}
		W_\lambda^{(\ell)}(\mu, \nu) = \sum_{i=1}^n\mu_i\varphi_i^{(\ell)} + \sum_{j=1}^m \nu_j \psi_j^{(\ell)}.
	\end{equation}
	The variables $\varphi^{(\ell)}$ and $\psi^{(\ell)}$ being the dual variables returned after $\ell$ iterations of the Sinkhorn algorithm. Starting from $\psi^{0} = 0_m \in \mathbb{R}^m$, the Sinkhorn $\ell^{th}$ iteration is defined by the update of the dual variables:
	\begin{equation}\label{eq:sinkhorn_updates}
		\begin{array}{ccl}
			\varphi^{(\ell)}_i & = & -\lambda \log \left( \sum_{j=1}^m \exp\left(-\frac{\|x_i-y_j\|^2 - \psi_j^{(\ell-1)}}{\lambda} \right)\nu_j\right)\\
			\psi^{(\ell)}_j & = & -\lambda \log \left( \sum_{i=1}^n \exp\left(-\frac{\|x_i-y_j\|^2 - \varphi_i^{(\ell)}}{\lambda} \right)\mu_i\right).
		\end{array}
	\end{equation}
\end{deff}

\begin{rem}[Convergence of Sinkhorn algorithm]\label{rem:conver_sink}
	Convergence guarantees of Sinkhorn algorithm are established for instance in \cite[Prop.2]{chizat2020sinkhorn}. That is, when the number of iterations $\ell$ goes to infinity, $W_\lambda^{(\ell)}(\mu, \nu)$ converges toward $W_\lambda(\mu, \nu)$. 
\end{rem}

A de-biased version of the regularized OT cost is also applied in \cite{genevay2017gansinkhorn} and further studied in \cite{feydy2019sinkhorndiv} and \cite{chizat2020sinkhorn}. This regularized OT cost is referred to as the Sinkhorn divergence, and defined as follows.

\begin{deff}[Sinkhorn Divergence \cite{feydy2019sinkhorndiv}]
	For $\lambda > 0,$ the Sinkhorn divergence between two probability measures $\mu$ and $\nu$ is defined by the formula
	\begin{equation}\label{eq:definition_sinkdiv}
		S_\lambda(\mu, \nu) := W_\lambda(\mu, \nu) - \frac{1}{2}\left( W_\lambda(\mu, \mu) + W_\lambda(\nu, \nu) \right).
	\end{equation}
\end{deff}
	In this article, we focus on the classical regularized transport cost $W_\lambda$. At the price of stronger assumptions, we also extend our results to the Sinkhorn divergence $S_\lambda$ in the Appendix \ref{sec:sinkdiv}. In the experimental Section~\ref{sec:numerical_experiments}, we will study the performance of the estimators based on the Sinkhorn divergence.

\subsection{An alternative dual problem} \label{sec:alternativedual}

We now introduce an alternative dual formulation of regularized OT that is specific to the quadratic cost.
This alternative dual problem is restricted to a class of Kantorovich potentials that are concave and Lipschitz functions, which proves useful to derive some of the convergence rates given in Section~\ref{sec:wass_estimators}.
The relation between those dual problems has already been explicited for un-regularized OT (for example in~\cite{chizat2020sinkhorn}), and we extend it to the regularized case.
Let $\lambda \geq 0$. By expanding the squared Euclidean cost, 
we have for any  $\pi \in \Pi(\mu,\nu)$, 
\begin{align}
	\int_{\mathcal{X} \times \mathcal{Y}} \|x-y\|^2 d\pi(x,y) + \lambda \textrm{KL} (\pi | \mu \otimes \nu) & = \int_\mathcal{X}\|x\|^2 d\mu(x) + \int_\mathcal{Y}\|y\|^2  d\nu(y) \\
	& -2 \int_{\mathcal{X} \times \mathcal{Y}}\langle x, y \rangle d\pi(x,y) + \lambda\textrm{KL} (\pi | \mu \otimes \nu). \nonumber
\end{align}
The above decomposition leads us to consider the new regularized transportation problem
\begin{equation}\label{eq:new_primal_scost}
	W_\lambda^s(\mu, \nu) = \min_{ \substack{\pi \in \Pi(\mu,\nu)} } \int_{\mathcal{X} \times \mathcal{Y}} s(x,y) d\pi(x,y) + \lambda \textrm{KL} (\pi | \mu \otimes \nu),
\end{equation}
with $s(x,y) = -2\langle x, y \rangle$. First, we remark that the standard regularized Wasserstein distance $W_\lambda(\mu, \nu)$ and the alternative regularized Wasserstein distance $W_\lambda^s(\mu, \nu)$ are link through the relation
\begin{equation}\label{eq:link_regularized_distances}
	W_\lambda(\mu, \nu) = \int_\mathcal{X}\|x\|^2 d\mu(x) + \int_\mathcal{Y}\|y\|^2  d\nu(y) + W_\lambda^s(\mu, \nu).
\end{equation}

A dual formulation associated to the problem \eqref{eq:new_primal_scost} is given by the next proposition.

\begin{prop}\label{prop:new_dual}  The dual problem associated to \eqref{eq:new_primal_scost} writes as
	\begin{align} \label{eq:dual:new_primal_scost}
		W_\lambda^s(\mu, \nu) & = \sup_{\substack{\varphi \in  L^{\infty}(\mathcal{X}) \\ \psi \in L^{\infty}(\mathcal{Y})}} \int \varphi(x) d \mu(x) + \int \psi(y) d\nu(y)\\
		& \qquad - \int m_{\lambda}(\varphi(x) +\psi(y) + 2 \langle x,y \rangle ) d\mu(x) d\nu(y), \nonumber
	\end{align}
	where $m_{\lambda}$ is defined in Theorem~\ref{thm:ot_dual}.
\end{prop}

\begin{proof}
	The key argument is to remark that  \eqref{eq:new_primal_scost} is a regularized optimal transportation problem with cost function $s(x,y) = -2\langle x,y \rangle$. Hence, as $\mathcal{X}$ and $\mathcal{Y}$ are assumed to be compact and $s$ continuous, it follows that strong duality holds  (see e.g. \cite{genevay2016ot,BercuBigot2021}) in the sense of equation \eqref{eq:dual:new_primal_scost}.
\end{proof}

Fort the cost function $s(x,y) = -2 \langle x, y \rangle$, we can also define a $s$-transform and a semi-dual problem as follows. For the cost $s(x,y) = -2  \langle x, y \rangle$ and for $\varphi \in L^{\infty}(\mathcal{X})$ the $s$-transform is defined as
\begin{equation}\label{eq:def_stransform}
	\forall y \in \mathbb{R}^d, \quad
	\varphi^{s,\lambda}_{\mu}(y)=
	\begin{cases}
		- \lambda \log \ds\left(\int \exp\left( \frac{\varphi(x) + 2 \langle x, y \rangle}{\lambda} \right) d\mu(x)\right), & \text{for} \ \lambda > 0 ,\\ 
		- \ds\max_{x \in \mathcal{X}} (\varphi(x) + 2 \langle x,y \rangle )  & \text{for} \ \lambda = 0 .
	\end{cases}
\end{equation}

By the above $s$-transform in the dual problem \eqref{eq:dual:new_primal_scost} we obtain the following semi-dual formulation 
\begin{equation}\label{eq:def_semi_dual_stransform}
	\sup_{\varphi \in L^{\infty}(\mathcal{X})}\int \varphi(x) d\mu(x) + \int \varphi^{s,\lambda}_{\mu}(y) d\nu(y).
\end{equation}
We conclude this section by studying some properties of this $s$-transform. While already established in \cite{chewientropicgeneralizationCaffarelli2022}, we give an elementary proof for completeness. 

\begin{prop}\label{prop:convexity_Lispchitz_stransform}
	For $\lambda \geq 0$, the $s$-transform $\varphi^{s,\lambda}_{\mu}$ is concave and $R$-Lipschitz on $B(0,R)$.
\end{prop}

\begin{proof}
	We start with the concavity of $\varphi^{s,\lambda}_{\mu}$.
	For $\lambda = 0$, it follows from the fact that a maximum of convex functions is convex.
	Now, for $\lambda > 0$, $y_1, y_2 \in \mathcal{Y}$ and $t \in (0,1)$, we have
	\begin{align*}
		\int_\mathcal{X} \exp &\left( \frac{ \varphi(x) + \langle x, t y_1 + (1-t)y_2 \rangle}{\lambda} \right) d\mu(x) \\
		& = \int_\mathcal{X} \exp\left(t\frac{\varphi(x) + 2 \langle x, y_1 \rangle}{\lambda} \right) \exp \left( (1-t)\frac{\varphi(x) + 2 \langle x, y_2 \rangle}{\lambda} \right) d\mu(x) \\
		& \leq \left( \int_\mathcal{X} \exp\left(\frac{\varphi(x) + 2 \langle x, y_1 \rangle}{\lambda} \right) d\mu(x)\right)^{t}  \left( \int_\mathcal{X} \exp\left(\frac{\varphi(x) + 2 \langle x, y_2 \rangle}{\lambda} \right) d\mu(x)\right)^{1-t},
	\end{align*}
	thanks to H\"older inequality with exponents $p = 1/t$ and $p'=1/(1-t)$.
	Applying $-\lambda \log$ on both sides gives directly
	\begin{equation}
		\varphi^{s,\lambda}_{\mu} \left( ty_1 + (1-t)y_2 \right) \geq  t \varphi^{s,\lambda}_{\mu}(y_1) + (1-t) \varphi^{s,\lambda}_{\mu}(y_2) .
	\end{equation}

	Now, we will see as in~\cite{feydy2019sinkhorndiv} that the regularized $s$-transform inherits the Lipschitz constant of the cost.
	For $y_1, y_2 \in \mathcal{Y}$ and $x \in \mathcal{X}$, we have 
	$|\langle x, y_1 - y_2 \rangle| \leq R \|y_1 - y_2\|$
	thanks to Cauchy-Schwarz inequality, and thus
	$$
	\varphi(x) + 2 \langle x, y_1 \rangle \leq R \|y_1 - y_2\| + \varphi(x) + 2 \langle x, y_2 \rangle .
	$$

	Taking $\lambda \log \int_{\mathcal{X}} \exp(\frac{\cdot}{\lambda}) d\mu(x)$ for $\lambda > 0$ (resp. $\max_{x \in \mathcal{X}}$ for $\lambda = 0$) on both sides gives
	$$\varphi^{s,\lambda}_{\mu}(y_2) \leq R\|y_1 - y_2\| +  \varphi^{s,\lambda}_{\mu}(y_1) \ .$$
	By symmetry, we get
	$|\varphi^{s,\lambda}_{\mu}(y_1) - \varphi^{s,\lambda}_{\mu}(y_2)| \leq R\|y_1 - y_2\|$.
\end{proof}

\subsection{Definition of the problem and quantity of interest} \label{subsec:exmixt}

Let $\mu = \sum_{k=1}^K \pi_k \mu_k$ be a probability measure that can be decomposed as a mixture of $K$ probability measures $\mu_1,\ldots,\mu_K$ in $ \mathcal{M}_+^1(\mathcal{X})$.  For $\theta \in \Sigma_K  = \{(\theta_1,\ldots,\theta_K)\in \mathbb{R}_{+}^{K} \; : \; \sum_{k=1}^{K} \theta_{k} = 1 \} $,  the re-weighted version of $\mu$  is defined as
\begin{equation}\label{eq:definition_mu_theta_bis}
	\mu_{\theta} =  \sum_{k=1}^{K} \theta_k \mu_k.
\end{equation}
Let $\nu$ be another probability measure in $ \mathcal{M}_+^1(\mathcal{Y})$ referred to as the target distribution.
The problem of class proportions estimation consists in estimating an optimal weighting vector
\begin{equation}\label{eq:theta_star_definition}
	\theta^* \in \argmin_{\theta \in \Sigma_K} W_{0}(\mu_{\theta},\nu)
\end{equation}
from empirical versions of the $\mu_1, \ldots, \mu_K$ and $\nu$.
In what follows, we discuss some properties of the optimisation problem \eqref{eq:theta_star_definition}.

First, this minimization problem is motivated by the implicit assumption that representing the target measure $\nu$ as a mixture of $K$ probability measures is relevant. To illustrate this point, we first state a result showing that one can recover the true class proportions in the ideal setting where the target distribution $\nu$ is also a mixture of $\mu_1,\ldots,\mu_K$.

\begin{lem}\label{lem:motivating_example}
	Suppose that $\mu_{\theta}$ and $\nu$ are mixtures of probability measures with the same components $\mu_1,\ldots,\mu_K$ but with different class proportions, respectively denoted by $\theta \in  \Sigma_K  $ and by $\tau \in  \Sigma_K  $. If the model $\left\{ \mu_\theta = \sum_{k=1}^K \theta_k \mu_k ~|~ \theta \in \Sigma_K\right\}$ is identifiable (in the sense that the mapping $\theta \mapsto  \mu_\theta$ is injective), then the solution of  optimization problem \eqref{eq:theta_star_definition} is unique and one has that $\theta^* = \tau$. 
\end{lem}

\begin{proof}
	The non-negativity property of $W_0$ ensures that for all $\theta \in \Sigma_K,~ W_0(\mu_\theta, \nu) \geq 0$. Next, for $\theta \in \Sigma_K$, 
	\begin{align*}
		W_0(\mu_\theta, \nu) = 0 & \Leftrightarrow \sum_{k=1}^K \theta_k \mu_k = \sum_{k=1}^K \tau_k \mu_k\\
		&\Leftrightarrow \theta = \tau,
	\end{align*} 
	where the last equivalence comes from the assumption that the model $\{ \mu_\theta ~|~ \theta \in \Sigma_K \}$ is identifiable. From this result, we deduce $\argmin_{\theta \in \Sigma_K}W_0(\mu_\theta, \nu) = \{\tau\}$.
\end{proof}

Notice that the injectivity of $\theta \mapsto \mu_{\theta}$ relates to the affine independence of $\{\mu_1, \ldots, \mu_K\}$. It is satisfied for example when the measures $\mu_1, \ldots, \mu_K$ have disjoint supports. If all the scenarios are not as friendly as the one considered in Lemma \ref{lem:motivating_example}, in numerous applications (for instance when the data can be clustered into sub-populations), it is relevant to approximate $\nu$ by a mixture model. The next result is about the smoothness of the minimization problem \eqref{eq:theta_star_definition}.

\begin{lem}\label{lem:continuity_un-regularized_loss_function}
	Suppose that $\mu_\theta$ is defined as in \eqref{eq:definition_mu_theta}. Then, the function \\$F~:~ \left\{\begin{array}{ccc}  \Sigma_K & \rightarrow & \mathbb{R}_+ 
		\\ \theta & \mapsto & W_0(\mu_{\theta}, \nu)
	\end{array}\right.$ is continuous on $\Sigma_K$.
\end{lem}
\begin{proof}
	
	Let $\theta \in \Sigma_K$ and $(\theta^{(n)})$ a sequence in $\Sigma_K$ that converges to $\theta$. Then, the probability sequence $(\mu_{\theta^{(n)}})$ converges weakly toward $\mu_\theta$. Indeed, for any bounded continuous function $\varphi$, one has that 
	$
	\int\varphi d\mu_{\theta^{(n)}} = \sum_{k=1}^K \theta_k^{(n)}\int \varphi d\mu_k.
	$
	As $\theta^{(n)} \rightarrow_{n \rightarrow \infty} \theta$, it follows that 
	$$\sum_{k=1}^K \theta_k^{(n)}\int \varphi d\mu_k \rightarrow \sum_{k=1}^K \theta_k\int \varphi d\mu_k. = \int \varphi d \mu_\theta.$$
	That is,  $(\mu_{\theta^{(n)}})$ weakly converges  towards $\mu_\theta$.
		As we work under the assumption that $\mathcal{X}$ and $\mathcal{Y}$ are compact, week convergence is equivalent to Wasserstein convergence \cite{santambrogio2015ot}[Thm. 5.10]. Hence, $W_0(\mu_{\theta^{(n)}}, \nu)$ converges to $W_0(\mu_\theta, \nu)$ when $n$ goes to infinity; which shows the continuity of $F$.
\end{proof}

Since the set $\Sigma_K$ is compact, the existence of a minimizer of the optimization problem \eqref{eq:theta_star_definition} follows from Lemma \ref{lem:continuity_un-regularized_loss_function}. We now give sufficient conditions that ensure the strict convexity of the objective function $\theta \mapsto W_0(\mu_\theta, \nu)$.

\begin{lem}\label{lem:identification}
	Assume that $\nu$ is absolutely continuous with respect to the Lebesgue measure. Then, if the model $\{ \mu_\theta~|~ \theta \in \Sigma_K \}$ is identifiable  (in the sense that the mapping $\theta \mapsto  \mu_\theta$ is injective), the function $F : \left\{\begin{array}{ccc}  \Sigma_K & \rightarrow & \mathbb{R}_+ 
		\\ \theta & \mapsto & W_0(\mu_{\theta}, \nu)
	\end{array}\right.$ is strictly convex.
\end{lem}

\begin{proof}
	Thanks to the assumption that  $\nu$ is absolutely continuous, Proposition 7.19 in \cite{santambrogio2015ot} ensures the strict convexity of the functional $\mu \mapsto W_0(\mu, \nu)$.
	Let   $\theta_0, \theta_1 \in \Sigma_K$ with $\theta_0 \neq \theta_1$ and $t \in (0,1)$. Then, we have that
	$F(t\theta_0 + (1-t)\theta_1) = W_0(\mu_{t\theta_0 + (1-t)\theta_1}, \nu),$
	and $\mu_{t\theta_0 + (1-t)\theta_1} =  t \mu_{\theta_0} + (1-t) \mu_{\theta_1}$.
	
	Since $\theta_0 \neq \theta_1$ and the model $\{\mu_\theta~|~ \theta \in \Sigma_K \}$ is supposed to be identifiable, we have that $\mu_{\theta_0} \neq \mu_{\theta_1}$. Therefore, the strict convexity of $\mu \mapsto W_0(\mu, \nu)$ yields
	$$W_0(t\mu_{\theta_0}+(1-t)\mu_{\theta_1}, \nu) < t W_0(\mu_{\theta_0}, \nu)+ (1-t)W_0(\mu_{\theta_1}, \nu).$$
	which proves the strict convexity of $F : \theta \mapsto W_0(\mu_\theta, \nu)$.
	
\end{proof}

\section{Convergence rates for the expected excess risk of parametric Wasserstein estimators} \label{sec:wass_estimators}

In this section,  we present the regularized and un-regularized parametric Wasserstein estimators that are considered in this paper, and we compare their convergence rates.

\subsection{Definition of the estimators} \label{sec:defestim}

We aim at estimating $\theta^*$ when the distributions $\mu$ and $\nu$ are only observed through samples. Hence, we assume given the following empirical measures (as defined in Section \ref{sec:prop})
$$
\hat{\mu} = \sum_{k=1}^K \hat{\pi}_{k} \hat{\mu}_k, \mbox{ where } \hat{\pi}_{k} =  \frac{m_k}{m}, \quad \text{and} \quad \hat{\nu} = \frac{1}{n} \sum_{j=1}^{n} \delta_{Y_j},
$$
where each component $\hat{\mu}_k$ corresponds to a known sub-population of cells of size $m_k$ in the source sample $X_1,\ldots,X_m$.

Moreover, we recall that  $\hat{\mu}_\theta =  \sum_{k=1}^{K} \theta_k \hat{\mu}_k.$ denotes the empirical version of the re-weighted measure $\mu_\theta$.  

We  can now define the various Wasserstein estimators whose convergence properties are discussed in Section \ref{sec:convrates}. Depending on the regularization parameter chosen, and using the empirical measures $\hat{\mu}_\theta$ and $ \hat{\nu}$, a family of estimators $(\hat{\theta}_\lambda)_{\lambda \geq 0}$ of the class proportions can be defined as follows:
\begin{equation}\label{eq:reg_estimate}
	\text{for}~ \lambda \geq 0, \quad	\hat{\theta}_{\lambda} \in \widehat{\Theta}_{\lambda} := \argmin_{\theta \in \Sigma_K} W_{\lambda}(\hat{\mu}_{\theta},\hat{\nu}),
\end{equation}

When considering entropy regularized OT, that is when $\lambda > 0$,  we also propose to study the estimators that are obtained with the Sinkhorn algorithm on the sample distributions after a limited number of $\ell$ iterations, that are
\begin{equation}\label{eq:sink_output_estimate}
	\hat{\theta}_{\lambda}^{(\ell)} \in \widehat{\Theta}_{\lambda}^{(\ell)}
	:= \argmin_{\theta \in \Sigma_K} W_{\lambda}^{(\ell)}(\hat{\mu}_{\theta},\hat{\nu}).
\end{equation}

As pointed previously in Remark~\ref{rem:conver_sink}, the convergence of $W_\lambda^{(\ell)}$ towards $W_\lambda$ allows to interpret the estimator $\hat{\theta}_\lambda$ as a limiting case of $\hat{\theta}_\lambda^{(\ell)}$ when $\ell$ goes to infinity. Beside studying the estimators $\hat{\theta}_\lambda$ and $\hat{\theta}_\lambda^{(\ell)}$; we also extend our result when substituting the regularized transport cost $W_\lambda$ in equation \eqref{eq:reg_estimate} by the Sinkhorn divergence $S_\lambda$ that is defined by formula \eqref{eq:definition_sinkdiv}. Due to space constraint, we have gathered theoretical results related to the Sinkhorn divergence to the Appendix \ref{sec:sinkdiv}. In this paper, to assess the performance of a given estimator $\hat{\theta}$ of $\theta^*$ based on $n$ samples, we shall consider the following expected excess risk defined as
\begin{equation}
	\label{eq:excess_risk}
	r_n(\mu_{\hat{\theta}},\nu) = \mathbb{E} \big[W_{0}(\mu_{\hat{\theta}},\nu) - W_{0}(\mu_{\theta^*}, \nu) \big].
\end{equation}

\begin{rem}
	In our context of parametric Wasserstein estimation, we can interpret the excess risk as the representation error of $\nu$ induced by the estimator. Indeed, $\mu_{\theta^*}$ defined in equation \eqref{eq:theta_star_definition} is the best representation of $\nu$ in the model $\{\mu_\theta ~|~ \theta \in \Sigma_K \}$ w.r.t. the Wasserstein distance. And, $W_0^{1/2}$ being a distance, under the assumption that the function $\theta \mapsto W_0(\mu_\theta, \nu)$ is bounded on $\Sigma_K$, we can write
	$$0 \leq W_{0}(\mu_{\hat{\theta}},\nu) - W_{0}(\mu_{\theta^*}, \nu) \lesssim W_{0}^{1/2}(\mu_{\hat{\theta}}, \mu_{\theta^*}).$$ This equation shows that the excess risk is closely related to Wasserstein distance between the best representation of $\nu$ in the model that is $\mu_{\theta^*}$ and its estimated version $\mu_{\hat{\theta}}$.
\end{rem}

\begin{rem}Instead of controlling the excess risk \eqref{eq:excess_risk} we would have preferred to work directly on the weights. That is, upper bounding the quantity $\|\hat{\theta}-\theta^*\|$. It would have been possible to derive such a result if the function $\theta \mapsto W_{0}(\mu_{\theta}, \nu)$ had been strongly convex. However, we can find elementary counter-examples where $\theta \mapsto W_{0}(\mu_{\theta}, \nu)$ is \textit{not} strongly convex.
Indeed, on the real line, let us consider $\mu_\theta = \theta \delta_{0} + (1-\theta)\delta_1$ and $\nu = (\delta_0 + \delta_1)/2$.
Then one can show that for every $\theta \in [0,1]$, $$W_0(\mu_\theta, \nu) = |1/2-\theta|,$$
which is not strongly convex. This result can be established thanks to the formula that links the quantile functions to the optimal transport cost on the real line (see e.g., \cite[Prop. 2.17]{santambrogio2015ot})
\end{rem}

\begin{rem}
We believe that our results could be extended to other parametric families of the form $\{\mu_\theta, \; \theta \in \Theta \}$, provided compact supports assumptions are satisfied, as well as continuity assumptions on the map $\theta \mapsto \mu_\theta$. For instance, further research could aim to extend existing results by ~\cite{Biau21} on the statistical analysis of un-regularized Wasserstein Generative Adversarial Networks~(WGAN) to the case of entropy regularized WGAN considered by~\cite{sanjabi2018wgan}.
\end{rem}

In Section \ref{sec:convrates}, we present upper bounds on the above expected risk for the proposed estimators. When the regularization parameter $\lambda$ is involved, we also propose a decreasing choice $(\lambda_n)_{n \geq 0}$ of its value to ensure that the resulting estimator has an expected excess risk that goes to zero when $n \to + \infty$. 


\subsection{Convergence rates for the expected excess risk} \label{sec:convrates}
This section contains the main results of this paper. We study the rate of convergence of the family of estimators $(\hat{\theta}_\lambda^{(\ell)})_{\lambda \geq 0}$ depending on the parameters $\lambda \geq 0$ and $\ell \in \mathbb{N}^*$. In the following results, the notation $\lesssim$ means inequality up to a multiplicative universal constant.\\
	
	As classically done in nonparametric statistics, we decompose the excess risk of an estimator into an estimation error and an approximation error that need to be balanced to derive an optimal choice of the regularization parameter $\lambda \to 0$ as the number of observations tends to infinity. For example, the excess risk of the estimator $\hat{\theta}_{\lambda}$ defined in equation \eqref{eq:reg_estimate} is upper bounded as follows:
	\begin{align}\label{eq:decompo_estim_approx}
		W_0(\mu_{\hat{\theta}_\lambda}, \nu) - W_0(\mu_{\theta^*}, \nu)   \leq ~& 2\sup_{\theta \in \Sigma_K}|W_\lambda(\hat{\mu}_{\theta}, \hat{\nu}) - W_\lambda(\mu_{\theta}, \nu)|\\  
		& + 2\sup_{\theta \in \Sigma_K} | W_{0}(\mu_{\theta}, \nu) - W_{\lambda}(\mu_{\theta}, \nu) |. \nonumber
	\end{align}
	As the introduction of entropic penalty term in the optimal transport problem was motivated by computational improvement \cite{cuturi2013sinkhorn}, it is also useful to take into account the algorithmic error. Therefore, we substitute in equation \eqref{eq:decompo_estim_approx} the estimator $\hat{\theta}_\lambda$ by its version computed with Sinkhorn algorithm: $$\hat{\theta}_\lambda^{(\ell)} = \argmin_{\theta \in \Sigma_K} W_\lambda^{(\ell)}(\hat{\mu}_\theta, \hat{\nu}).$$ In such a case, we provide an upper bound in the next lemma. 
	
	\begin{lem}\label{lem:decompo_estim_approx_algo} The excess risk of the estimator $\hat{\theta}_\lambda^{(\ell)}$ is upper bounded as follows:
		\begin{align}\label{eq:decompo_estim_approx_algo}
			0 \leq 	W_0(\mu_{\hat{\theta}_\lambda^{(\ell)}}, \nu) - W_0(\mu_{\theta^*}, \nu)   \leq 2 &\underbrace{ \sup_{\theta \in \Sigma_K}|W_\lambda(\hat{\mu}_{\theta}, \hat{\nu}) - W_\lambda(\mu_{\theta}, \nu)|}_{\textrm{Estimation error}}\\  
			& + 2\underbrace{ \sup_{\theta \in \Sigma_K} | W_{0}(\mu_{\theta}, \nu) - W_{\lambda}(\mu_{\theta}, \nu) | }_{\text{Approximation error}} \nonumber\\	
			& + \underbrace{2 \sup_{\theta \in \Sigma_K}|W_\lambda^{(\ell)}(\hat{\mu}_{\theta}, \hat{\nu}) - W_\lambda(\hat{\mu}_{\theta}, \hat{\nu})|.}_{\text{Algorithm error}}  \nonumber
		\end{align}
	\end{lem}
	The computations leading to Lemma \ref{lem:decompo_estim_approx_algo} are gathered in Section \ref{subsec:estimation_approximation_reg_wass} of the Appendix. The main theorem of this article is based on a new bound for the control of the estimation error, given in the following proposition.
	\begin{prop}\label{prop:reg_free_estimation}
		Set $\lambda \geq 0,$ and suppose that the probability measures $\mu_1,\ldots, \mu_K$, and $\nu$ have compact supports included in $B(0,R)$. If for all components $\mu_k$ as well as for $\nu$ at least $n$ observations are available, Then the following inequality holds true:
		\begin{equation}
			\mathbb{E}\left[\sup_{\theta \in \Sigma_K}|W_\lambda(\hat{\mu}_{\theta}, \hat{\nu}) - W_\lambda(\mu_{\theta}, \nu)| \right]   \lesssim K \mathcal{E}(n,d),
		\end{equation}
		where the upper bound $\mathcal{E}(n,d)$ is defined by 
		\begin{equation}\label{eq:upper_bound_estimation}
			\mathcal{E}(n,d) := 	\left\{ \begin{array}{lll}
				R^2 n^{-1/2} & \text{if} & d<4,\\
				R^2 n^{-1/2}\log(n)& \text{if} & d=4,\\
				R^2 n^{-2/d} & \text{if} & d>4.
			\end{array}
			\right.
		\end{equation}
	\end{prop}
	
	The proof of Proposition \ref{prop:reg_free_estimation} is deferred to Section \ref{subsec:control_estimation} of the Appendix. This proof relies on \cite[Lemma 4]{chizat2020sinkhorn} where the maximum of an empirical process is upper bounded. We point out that the upper bound in Proposition~\ref{prop:reg_free_estimation} is independent of the regularization parameter $\lambda$. 
	
	\begin{thm}\label{thm:bound_decomposed_error}
		Set $\lambda > 0$ and suppose that all probability measures $\mu_1, \ldots, \mu_K$ and $\nu$ have compact supports. Assume that for all the components $\mu_k$ as well as for $\nu$, at least $n$ observations are available. Then, the expected excess risk of the estimator $\hat{\theta}_\lambda^{(\ell)}$ introduced in equation \eqref{eq:sink_output_estimate} is upper bounded as follows:
		$$ \mathbb{E}\left[W_0(\mu_{\hat{\theta}_{\lambda}^{(\ell)}}, \nu) -W_0(\mu_{\theta^*}, \nu) \right]   \lesssim \mathcal{E}(n,d) + \lambda \log\left( \frac{1}{\sqrt{d}\lambda}\right) + \frac{1}{\lambda \ell},$$
		where the quantity $\mathcal{E}(n,d)$ is defined by formula \eqref{eq:upper_bound_estimation}, and the implicit multiplicative constant depends only on $K$, $d$ and $R$.
	\end{thm}
	
	The detailed proof of Theorem \ref{thm:bound_decomposed_error} can be found in Section \ref{subsec:proof_main} of the Appendix. We mention that it is based on the upper bound \eqref{eq:decompo_estim_approx_algo} where each term of the right-hand side is controlled by the appropriate bound. The expected estimation error is upper bounded thanks to Proposition \ref{prop:reg_free_estimation}. For the remaining terms, we collect results established in the literature. More precisely, we exploit the works of Genevay et al. \cite{genevay2019sample} and of Chizat et al. \cite{chizat2020sinkhorn} to control the approximation and the algorithm errors.

	\begin{cor}\label{cor:bound_selected_parameters}
		Suppose that every probability measure has compact support and that $d>4$. 
		If for all the components $\mu_k$ as well as for $\nu$, at least $n$ observations are available, then the following non-asymptotic rates of convergence hold for the estimator $\hat{\theta}_{\lambda_n}^{(\ell_n)}$ computed with the Sinkhorn algorithm:
		$$ \mathbb{E}\left[W_0(\mu_{\hat{\theta}_{\lambda_n}^{(\ell_n)}}, \nu) -W_0(\mu_{\theta^*}, \nu) \right]   \lesssim   n^{-\frac{2}{d}}\log(n), \quad  \text{with} \quad \left\{ \begin{array}{l}
			\lambda_n = n^{-2/d}, \\
			\ell_n \geq 32R^4n^{4/d}.
		\end{array}\right.$$
	\end{cor}
	
	\begin{proof}
		Setting the regularization parameter to $\lambda = n^{-2/d}$, and the number of iterations to $\ell =  32R^4n^{4/d}$ in inequality \eqref{eq:sink_output_estimate} yields the announced rate of convergence. 
	\end{proof}
	
	\begin{rem}[Extension to the Sinkhorn divergence $S_\lambda$]
		The estimators analyzed in Theorem \ref{thm:bound_decomposed_error} and Corollary \ref{cor:bound_selected_parameters} are defined as solutions of variational problems based on the regularized transport cost $W_\lambda$. Under ad hoc assumptions, our results can be extended to the Sinkhorn divergence $S_\lambda$ whose definition is reminded in equation \eqref{eq:definition_sinkdiv}. In such a case, we define the collection of estimators $(\hat{\theta}_\lambda^{S})_{\lambda > 0}$ as follows:
		\begin{equation}\label{eq:sinkdiv_estimate}
			\hat{\theta}_\lambda^S  \in \argmin_{\theta \in \Sigma_K}S_\lambda(\hat{\mu}_\theta, \hat{\nu}).	
		\end{equation}
		Provided stronger assumptions are made, the approximation error of the Sinkhorn divergence is smaller than the approximation error $|W_\lambda - W_0|$. However, the estimators $\hat{\theta}_\lambda^S$ and $\hat{\theta}_\lambda$ have estimation errors of the same magnitude $n^{-2/d}$. Therefore, when tuning the parameter $\lambda$ depending on the number of observations and the dimension, we reach the rate $n^{-2/d}$. This the same rate, up to logarithm factor, as for the estimator $\hat{\theta}_\lambda$ that we study throughout this article. We can also take into account the algorithm error for the estimator $\hat{\theta}_\lambda^{S}$ depending on the number of iterations $\ell$. All the results related to the estimator \eqref{eq:sinkdiv_estimate} introduced in this remark can be found in Section \ref{sec:sinkdiv} of the Appendix.
	\end{rem}
	
	When studying the sample complexity of the regularized OT cost, that is $\mathbb{E}\left[ | W_\lambda(\mu, \nu) - W_\lambda(\hat{\mu}, \hat{\nu}) | \right]$  as  done in \cite{genevay2019sample, mena2019statistical}, or when estimating the standard optimal transport cost $W_0(\mu, \nu)$ as in \cite{chizat2020sinkhorn}, bounds related to the control of the estimation error have been proved. These results give a control of $\mathbb{E}\left[ | W_\lambda(\mu, \nu) - W_\lambda(\hat{\mu}, \hat{\nu}) |\right]$ that is of order $C_\lambda/\sqrt{n}$ with $C_\lambda$ a constant that depends on the regularizing parameter. In the following section we give a similar result adapted to our context of weights estimation, and discuss why we favored the upper bound given in Proposition \ref{prop:reg_free_estimation}.

\section{Alternative bound on the estimation error, and relation to state of the art}\label{sec:comparison}

Proposition \ref{prop:reg_free_estimation} gives a control of the estimation error that is independent of the regularization parameter $\lambda$. We now give a bound much closer to what is known in the literature, where a small regularization parameter severely impacts the rate of convergence.

\begin{prop}\label{prop:estimation_reg_impact}
	Let $\lambda > 0,$ and suppose that all probability measures have compact supports included in $B(0,R)$. If for all components $\mu_k$ as well as for $\nu$ at least $n$ observations are available, the estimation error can be upper bounded as follows:
	\begin{equation}\label{eq:estimation_reg_impact}
		\mathbb{E}\left[\sup_{\theta \in \Sigma_K}|W_\lambda(\hat{\mu}_{\theta}, \hat{\nu}) - W_\lambda(\mu_{\theta}, \nu)| \right]   \lesssim \frac{K M_\lambda}{\sqrt{n}},
	\end{equation} 
	With $M_\lambda := M_d\max\left(R^2,\frac{R^{\lfloor d/2 \rfloor +1}}{\lambda^{\lfloor d/2 \rfloor}}\right).$
\end{prop}

The proof of this last proposition \ref{prop:estimation_reg_impact} can be found in Section \ref{subsec:empirical_process_reg_wass} of the Appendix.\\

This last upper bound \eqref{eq:estimation_reg_impact} seems appealing because independent of the dimension~$d$ and going to zero at the same rate of $n^{-1/2}$. However, the constant $M_\lambda$ depends on the dimension $d$ \emph{and} the regularization parameter $\lambda$. Thus, when one tries to exploit this bound \eqref{eq:estimation_reg_impact}, instead of $\mathcal{E}(n,d)$ like it is done in Theorem \ref{thm:bound_decomposed_error}, one reaches the following upper bound on the expected risk of the estimator $\hat{\theta}_\lambda$:
\begin{equation}\label{eq:estim_approx_regimpact}
	\mathbb{E}\left[W_0(\mu_{\hat{\theta}_{\lambda_n}}, \nu) -W_0(\mu_{\theta^*}, \nu) \right]  \lesssim \frac{1}{\lambda^{\lfloor d/2 \rfloor} \sqrt{n}}   + \lambda \log\left( \frac{1}{\sqrt{d}\lambda}\right).
\end{equation}

In the last inequality we have not taken into account the algorithm error. 
Balancing the two terms of the right-hand side of~\eqref{eq:estim_approx_regimpact} leads to a regularization parameter $\lambda_n = n^{1/(2\lfloor d/2 \rfloor +2)}$. Finally, under the assumptions of Theorem \ref{thm:bound_decomposed_error}, using the estimation error \eqref{eq:estimation_reg_impact} gives a slower rate of convergence than in Corollary \ref{cor:bound_selected_parameters}. Indeed the expected excess risk of the estimator $\hat{\theta}_{\lambda}$ is upper bounded by 
$$ \mathbb{E}\left[W_0(\mu_{\hat{\theta}_{\lambda_n}}, \nu) -W_0(\mu_{\theta^*}, \nu) \right]   \lesssim   n^{-1/(2\lfloor d/2 \rfloor +2 )}\log(n), \quad  \text{with} \quad 
\lambda_n =  n^{1/(2\lfloor d/2 \rfloor +2)}.
$$

\begin{rem}[Estimation of $W_0(\mu, \nu)$] We can adapt the results established in Corollary \ref{cor:bound_selected_parameters} to estimate the optimal transport cost with a regularized transport cost. Indeed we can use a \textit{regularized plug-in} estimators $W_\lambda(\hat{\mu}, \hat{\nu})$. This question is for instanced investigated in \cite{chizat2020sinkhorn} where the estimation of $W_0(\mu, \nu)$ is based on the Sinkhorn divergence defined in~\eqref{eq:definition_sinkdiv}.
	If $n$ samples are available from each measure $\mu$ and $\nu$, thanks to the estimator $S_\lambda(\hat{\mu}, \hat{\nu}),$ they reach the rate of convergence  (see \cite[Prop.~4]{chizat2020sinkhorn})
	\begin{equation}\label{eq:convergence_plug_sink}
		\mathbb{E}\left[| S_\lambda(\hat{\mu}, \hat{\nu}) - W_0(\mu, \nu) |\right]  \lesssim n^{-2/(2\lfloor d/2 \rfloor +4)},
	\end{equation}
	for some well chosen regularization parameter $\lambda_n$ that depends on $n$. However, based on the results we established $\hat{\theta}_\lambda$ we can derive faster rates of convergence for $W_{\lambda_n}(\hat{\mu}, \hat{\nu})$ toward $W_0(\mu, \nu)$.
\end{rem}

\begin{rem}
When setting $K=1$ in Proposition \ref{prop:estimation_reg_impact}, we recover a result close to what was previously established by other authors. To the best of our knowledge, the first result of this flavor was proven in \cite[Thm.~3]{genevay2019sample}, but with a factor $e^{R^2/\lambda}$ in the multiplicative constant. Soon after, the result \cite[Cor.~3]{mena2019statistical} improved the sample complexity control by removing the exponential factor, and substituting a compact support assumption by a sub-Gaussian one. In Proposition \ref{prop:estimation_reg_impact}, exploiting the compact supports assumption enables us to provide a multiplicative constant that scale with $R^{\lfloor d/2 \rfloor +1}\lambda^{-\lfloor d/2 \rfloor}$. All these controls over the sample complexity dramatically deteriorate in high dimension when $\lambda$ is close to zero. To mitigate this high dimensional phenomenon, the recent article \cite[Thm.~2]{stromme2024minimum} provides bounds where the impact of the ambient dimension $d$ is substituted by an intrinsic dimensional quantity of the measures compared.
\end{rem}

\begin{prop}\label{prop:estimation_prop}
	Suppose that $\mu$ and $\nu$ have their supports included in $B(0,R)$ and that the dimension $d > 4$. If $n$ samples are available for each probability measure, then the regularized plug-in estimator reaches the rate of convergence
	\begin{equation}\label{eq:convergence_plug_regwass}
		\mathbb{E}\left[|W_{\lambda_n}(\hat{\mu}, \hat{\nu}) - W_0(\mu, \nu) |\right] \lesssim n^{-2/d} \log(n) \quad \text{with} \quad \lambda_n = n^{-2/d}.
	\end{equation}	
\end{prop}
\begin{proof}
	We have
	\begin{equation*}
		\mathbb{E}\left[|W_{\lambda}(\hat{\mu}, \hat{\nu}) - W_0(\mu, \nu)|\right] \leq 
		\mathbb{E}\left[|W_{\lambda}(\hat{\mu}, \hat{\nu}) - W_\lambda(\mu, \nu)|\right]  + 
		|W_{\lambda}(\mu, \nu) - W_0(\mu, \nu)|.
	\end{equation*}
	The first term on the right-hand side is upper bounded by $\mathcal{E}(n,d)$ thanks to Proposition \ref{prop:reg_free_estimation} applied in the case $K=1$ and $\mu_1 = \mu$. For the second term, the result established in \cite[Thm. 1]{genevay2019sample} gives a control of order $\lambda\log(1/\lambda)$ when $\lambda$ goes to zero. Hence, assuming that $d>4$, and choosing $\lambda = n^{-2/d}$, we recover the rate of convergence claimed in equation \eqref{eq:convergence_plug_regwass}.
\end{proof}

Hence, in the case $d>4$, the expected error of the estimator $ W_{\lambda_n}(\hat{\mu}, \hat{\nu})$ goes to zero faster than when considering $S_{\lambda_n}(\hat{\mu}, \hat{\nu})$. While establishing inequality \eqref{eq:convergence_plug_regwass} only requires the measures $\mu$ and $\nu$ to have compact support, the previously known inequality \eqref{eq:convergence_plug_sink} requires stronger assumptions on the measures $\mu$ and $\nu$.

\begin{rem}
	To compare our results with the state-of-the-art, we set $K=1$ in Proposition \ref{prop:estimation_prop} to recover the sample complexity problem; that is controlling $\mathbb{E}\left[|W_{\lambda}(\hat{\mu}, \hat{\nu}) - W_\lambda(\mu, \nu)|\right]$. 
	After the first version of our manuscript was posted on ArXiv, the preprint \cite{groppe2023lower} established more general results than what we derive when taking $K=1$ in Proposition \ref{prop:estimation_reg_impact}. Indeed, the result \cite[Thm.~3]{groppe2023lower} provides a sample complexity control that holds for more general costs than the squared Euclidean distance we considered in the present article. Moreover, the authors of \cite{groppe2023lower} provide some adaptive results where dimension-dependent bounds are substituted by intrinsic dimensional properties of the measures compared. To stress the difference between our work and \cite{groppe2023lower}, we also mention that in this article the authors study the Entropic Gromov-Wasserstein cost. Conversely, our main results, that are Theorem \ref{thm:bound_decomposed_error} and its Corollary \ref{cor:bound_selected_parameters}, study an estimation framework where  entropic optimal transport costs are exploited as losses function. Our estimation scenario is thus different from \cite{groppe2023lower} where the quantities targeted are entropic optimal transport costs.
\end{rem}

\begin{rem}[Near minimax-rate for the estimation of $W_0(\mu, \nu)$]
	It has been shown in \cite[Thm. 21]{manole2021sharp} that the minimax rate of convergence of $W_0(\mu, \nu)$ is lower bounded by $(n\log(n))^{-2/d}$ when $n$ observations from each measure are available. Up to a logarithmic factor, as shown in \cite[Thm.2]{chizat2020sinkhorn}, this rate is reached by the plug-in estimator $W_0(\hat{\mu}, \hat{\nu})$. An application of our work is to show that, up to another logarithmic factor, the regularized plug-in estimator $W_{\lambda_n}(\hat{\mu}, \hat{\nu})$ also reaches this rate of convergence. However, in some cases, the computation of $W_{\lambda_n}(\hat{\mu}, \hat{\nu})$ might be faster than $W_0(\hat{\mu}, \hat{\nu})$.
\end{rem}

\begin{rem}[Computational cost of $W_{\lambda_n}(\hat{\mu}, \hat{\nu})$]
	One iteration of Sinkhorn algorithm requires $\mathcal{O}(n^2)$ arithmetic operations \citep[Page 5]{chizat2020sinkhorn}. And we compute an approximation of $W_{\lambda_n}(\hat{\mu}, \hat{\nu})$ with $W_{\lambda_n}^{(\ell_n)}(\hat{\mu}, \hat{\nu})$ where $\ell_n = 32R^4 n^{4/d}$. Hence, the global cost of computing of $W_{\lambda_n}^{(\ell_n)}(\hat{\mu}, \hat{\nu})$ is of $\mathcal{O}(n^{2+4/d})$ arithmetic operations. On the other hand computing $W_0(\hat{\mu}, \hat{\nu})$ with a linear programming algorithm requires $\mathcal{O}(n^3)$ \cite{pele2009fast, cuturi2013sinkhorn}. Hence, as soon as $d > 4$, approximating $W_0(\mu,\nu)$ based on $n$ samples is faster with Sinkhorn algorithm than with a linear program. From our understanding, the computational advantages of entropic optimal transport are in high dimension. This is due to the fact that in high dimension the estimation error is large enough to allow for a choice of large $\lambda$, and thus a fast convergence of Sinkhorn algorithm.  
\end{rem}

\section{Numerical experiments}\label{sec:numerical_experiments}

In this section, using simulated and real data from flow cytometry, we analyze the numerical performances of the estimators introduced in Section \ref{sec:wass_estimators}. These numerical experiments have been designed to demonstrate how the parameters $\lambda$ and $\ell$ impact the performance of regularized Wasserstein estimators. Moreover, these experiments show that an appropriate choice of the parameters allows regularized estimators to reach the performance of estimators based on the standard OT cost $W_0$. For the results reported here, the parameter $\lambda$ ranges in a finite grid $\Lambda \subset \mathbb{R}_+^*$ from $0.01$ to $1$. 
	Sinkhorn algorithm is either limited to $\ell = 5$ on simulated data, or to $\ell = 10$ on flow cytometry data. To simulate the setting where Sinkhorn in unlimited, our stopping criterion is based the difference between two consecutive outputs. More specifically, it stops if $|W_\lambda^{(\ell)} - W_\lambda^{(\ell-1)}| < 10^{-9}.$\\
	
	For the estimators based on the transport cost $W_\lambda$ with $\lambda > 0$, we follow the protocol described thereafter. Given a data set ${\bf X}$ classified into $K$ classes, and an unclassified data set ${\bf Y}$ where we want to estimate the class proportions, we compute the empirical distributions $\hat{\mu}_1, \ldots, \hat{\mu}_K$ and $\hat{\nu}$. Then, we compute the estimator $\hat{\theta}_\lambda$ of the class proportions by solving the optimization problem
	\begin{equation}\label{eq:sampling_wasserstein_estimators}
		\hat{\theta}_\lambda =\argmin_{\theta \in \Sigma_K}W_\lambda(\hat{\mu}_\theta, \hat{\nu}).
	\end{equation}
	To solve this problem, we apply a gradient descent algorithm to the function ${\theta \mapsto W_\lambda(\hat{\mu}_\theta, \hat{\nu})}$. In order to move from a constrained problem to an unconstrained one, we re-parameterize the simplex $\Sigma_K$ with a soft-max function $\chi : \mathbb{R}^K \rightarrow \Sigma_K $, where the $l$th  component of $\chi(z)$ is defined by 
	$$ 
	\chi(z)_l = \frac{\exp(z_l)}{\sum_{k=1}^K \exp(z_k)}.
	$$
	Then, we introduce the linear operator $\Gamma : \Sigma_K \rightarrow \Sigma_n$ that maps the weights associated to each component $\hat{\mu}_k$ to the weights associated to each observations. That is 
	$$
	\forall (i, k) \in \{1, \ldots, n\} \times \{1, \ldots ,K\},\quad \Gamma_{i,k} := \begin{cases}
		1/n_k &  \text{if} ~ X_i \sim \mu_k, \\
		0 & \text{otherwise.}
	\end{cases}
	$$
	Thus, our objective function reads $F_\lambda = W_{\lambda,\bf{X}}(\cdot,\hat{\nu}) \circ \Gamma \circ \chi$, where for $a \in \Sigma_n$, $W_{\lambda,\bf{X}} (a, \hat{\nu})$ denotes the transport cost between the measure with weights $a$ and support ${\bf X}= (X_1, \ldots, X_n)$, and the measure $\hat{\nu}$. From \cite[Prop. 9.1]{peyre2019computational}, we know that the gradient of $W_{\lambda,\bf{X}} (\cdot, \hat{\nu})$ at point $a$ is given by the unique dual potential $\varphi$ associated to the measure $\sum_{i=1}^n a_i \delta_{X_i}$ such that $\sum_{i=1}^n \varphi_i = 0$. From the chain rule of differentiation, we derive that the gradient of the objective function is given by 
	\begin{equation}\label{eq:gradient_step}
		\nabla_z F_\lambda(z) = J_\chi(z)^{T}\Gamma^T \varphi_z, 
	\end{equation}
	where $J_\chi$ is the Jacobian matrix of $\chi$ and $\varphi_z$ is the optimal potential with respect to $W_\lambda(\hat{\mu}_{\chi(z)}, \hat{\nu})$. Our approximation of $\varphi_z$ computed is with the Sinkhorn algorithm.
	
	\begin{algorithm}
			\SetAlgoLined
			
			$z \gets 1_K$\\
			
			\For{$N \leftarrow 1$ \KwTo $N_{\rm{out}}$}{
				\tcc{Sinkhorn Algorithm to compute the dual potentials of $W_{\lambda}(\hat{\mu}_{\chi(z)}, \hat{\nu})$}
				$\varphi \gets 0_n$ \\
				
				\For{$l \leftarrow 1$ \KwTo $\ell$}{
					\tcc{One Sinkhorn algorithm iteration}
					$\psi \gets \varphi^{c, \lambda}$ \\
					$\varphi \gets \psi^{c,\lambda}$
					
				}
				
				\tcc{Approximation of the gradient of $z \mapsto  W_{\lambda}(\hat{\mu}_{\theta_z},\hat{\nu})$ where $\theta_z = \chi(z)$}
				$\omega(z) \gets (\Gamma J_{\chi}(z))^{T}\varphi$
				
				$z \gets z - \eta \omega(z)$
				
			}
			
			\tcc{Approximation of the estimator $\hat{\theta}_\lambda$}
			\KwRet{$\hat{\theta}_\lambda = \chi(z)$}
			
			\caption{Approximation of $\hat{\theta}_\lambda^{(\ell)}  = \argmin_{\theta \in \Sigma_K}W_\lambda^{(\ell)}(\hat{\mu}_\theta, \hat{\nu})$. }
		\label{algorithm}
	\end{algorithm}
	
	When relying on the other transport costs studied, that are $W_0$,  $S_\lambda$, $W_\lambda^{(\ell)}$ or $S_\lambda^{(\ell)}$, we apply the same protocol as for $W_\lambda$; apart from the gradient formula \eqref{eq:gradient_step}. Indeed, denoting by $\mathcal{L}$ an optimal transport cost among $W_0, S_\lambda, W_\lambda^{(\ell)}$ and $S_\lambda^{(\ell)}$, the problem we are trying to solve (after parameterization with the soft-max function $\chi$) is
	\begin{equation}
		\min_{z \in \mathbb{R}^K} \mathcal{L}(\hat{\mu}_{\chi(z)}, \hat{\nu}).
	\end{equation}
	We can rewrite the objective function $F: z \rightarrow F(z)= \mathcal{L}_{\bf{X}}(\cdot,\hat{\nu}) \circ \Gamma \circ \chi(z)$. Here,
	$\mathcal{L}_{\bf{X}} (a, \hat{\nu})$ denotes the transport cost criterion between the measure with weights $a \in \Sigma_n$ and support ${\bf X}= (X_1, \ldots, X_n)$, and the measure $\hat{\nu}$. Then, differentiating this function $F$ gives the gradient
	$$\nabla_z F(z) = J_\chi(z)^T \Gamma^T \nabla \mathcal{L}_{{\bf X}}(\Gamma \chi(z), \hat{\nu}).$$ 
	
	Finally, depending on the loss $\mathcal{L}$, we substitute $\nabla \mathcal{L}_{{\bf X}}(\Gamma \chi(z), \hat{\nu})$ by its value. For the un-regularized case, we have $\nabla W_{0,{\bf X}}(\Gamma \chi(z), \hat{\nu}) = \varphi_z$ with $\varphi_z$ the first potential associate to $W_{0}(\hat{\mu}_{\chi(z)}, \hat{\nu})$. We rely on a linear programming algorithm to approximate this dual vector $\varphi_z$, which in this case is a sub-gradient \citep[Prop.~9.1]{peyre2019computational}. For $W_\lambda^{(\ell)}$, as in Algorithm 1, we rely on Sinkhorn algorithm to compute $\varphi^{(\ell)}$, the dual potential after $\ell$ iterations. For $S_\lambda$, the gradient is given by the formula $\nabla S_{\lambda,{\bf X}}(\Gamma \chi(z), \hat{\nu}) = \varphi^{\mu, \nu} - (\varphi^{\mu, \mu} + \psi^{\mu, \mu})/2$ \citep[eq.~2.12]{bigot2019CLT_EOT}, where $\varphi^{\mu, \nu}$ is the first potential associated to $W_\lambda(\hat{\mu}_{\chi(z)}, \hat{\nu})$, and $\varphi^{\mu, \mu}, \psi^{\mu, \mu}$ are the two potentials associated to $W_\lambda(\hat{\mu}_{\chi(z)}, \hat{\mu}_{\chi(z)})$. For $S_\lambda^{(\ell)}$, its gradient is given by the same formula as $S_\lambda$ while substituting the potentials by their approximations after $\ell$ steps of the Sinkhorn algorithm.\\
	
	\begin{rem}
		The algorithm described in the present article is fairly similar to the numerical scheme exploited in \cite{freulon2020cytopt}. However, in the present work, the approximation of the dual potential required to compute the gradient of $W_\lambda$ is based on Sinkhorn algorithm. While in the previous work \cite{freulon2020cytopt}, the authors applied the stochastic optimization algorithm studied in \cite{genevay2016ot, BercuBigot2021}. Relying on Sinkhorn algorithm enables us to incorporate the algorithmic error in our theoretical study.
	\end{rem}
	
	For each setting, that is choosing a loss among $W_\lambda$, $S_\lambda$, $W_\lambda^{(\ell)}$ or $S_\lambda^{(\ell)}$; and setting the parameters $\lambda$ and $\ell$, we sample (or sub-sample when experimenting on real data) $N=50$ couples of datasets $({\bf X}^{[1]}, {\bf Y}^{[1]}), \ldots , ({\bf X}^{[N]}, {\bf Y}^{[N]})$. Then, for each couple $({\bf X}^{[r]}, {\bf Y}^{[r]})$, we compute an estimator $\hat{\theta}^{[r]}$ of the class proportions in ${\bf Y}^{[r]}$. We thus obtain $N$ realizations  $\hat{\theta}^{[1]},\ldots, \hat{\theta}^{[N]}$ of a given estimator of the class proportions.  Then, we choose to evaluate performance of the estimator, by computing the quadratic errors $\|\hat{\theta}^{[1]}-\theta^*\|^2,\ldots, \|\hat{\theta}^{[N]}- \theta^*\|^2$. We display these error with box plots as in Figure \ref{fig:Results_Simulation_NoLimit_Variation}, where circles are the errors $\|\hat{\theta}^{[r]}-\theta^*\|^2$ beyond 1.5 times the inter-quartile range. When experimenting on synthetic data, $\theta^*$ is known as $\nu=\sum_{k=1}^K \theta_k^* \mu_k$ as ensured by Lemma \ref{lem:motivating_example}. In experiments on cytometry data, $\theta^*$ is unknown because all probability measures underlying the observations are unknown. In this case, we substitute $\theta^*$ by the true proportions in the unclassified data set ${\bf Y}$, to which we actually have access.
	
	We also approximate the expected quadratic risk $\mathbb{E}[\|\hat{\theta}-\theta^*\|^2]$ by Monte-Carlo repetitions as classically done in statistical experiments:
	\begin{equation}\label{eq:estimate_expectation_error}
		\mathbb{E}[\|\hat{\theta}-\theta^*\|^2] \approx \frac{1}{N}\sum_{r=1}^N\|\hat{\theta}^{[r]}-\theta^*\|^2.
	\end{equation}
	We plot this approximated average error on Figure \ref{fig:Result_Simulation_NoLimit_Average}, when considering for instance the losses $W_\lambda$ and $S_\lambda$. This protocol is repeated for each value of $\lambda$ in the grid $\Lambda$ and each loss function.

\begin{rem}
	In these numerical experiments, we have chosen to focus on the expected error $\mathbb{E}[\|\hat{\theta}-\theta^*\|^2]$ rather than the expected excess risk  $r_n(\mu_{\hat{\theta}},\nu)$ as in flow cytometry the relevant quantity is an accurate estimation of class proportions in the target dataset.
	Also, notice that the risk  $r_n(\mu_{\hat{\theta}},\nu)$ cannot be computed exactly because it involves the quantity $W_0(\mu_{\hat{\theta}}, \nu)$ for which we have no closed-form formula.
\end{rem}

\subsection{Simulated data}\label{subsec:simulated_data}

We first simulated two Gaussian mixtures of dimension $d=6$ with the same $K=5$ components but with different class proportions. Thus, a source data set ${\bf X}$ corresponds to random vectors $X_1,\ldots,X_n$ sampled with respect to $\mu$ and a target data set ${\bf Y}$ corresponds to  random vectors $Y_1,\ldots,Y_n$ sampled with respect to the distribution $\nu$, where $\mu$ and $\nu$ are defined below:
\begin{equation}\label{eq:simu_source_target_distribution}
	\mu = \sum_{k=1}^5 \pi_k \mathcal{N}(\rho_k,\sigma^2 I_d), \qquad  
	\nu = \sum_{k=1}^5 \theta_k^* \mathcal{N}(\rho_k,\sigma^2 I_d).
\end{equation} 
Because the vector of proportions $\pi$ and $\theta^*$ are not assumed to be equal, we exploit the known classes at the source in order to estimate the class proportions $\theta^*$ at the target, based on empirical versions of $\mu_1, \ldots, \mu_K$ and $\nu$.

We have same number of samples $m_k=n$ from each source components $\mu_k$ than samples from the target distribution $\nu$.  This experimentation setting matches the presentation of our theoretical results given in Section \ref{sec:convrates}. To ease the simulation study, we constrain the number of observations to $m_k=50$ observations for each class of the source data set. In the target data set, we also constrain the number of observations per class with $n_1 = 20, n_2=5, n_3=8, n_4=7, n_5=10$, so $n=50$ in total. We display in Figure \ref{fig:Simulated_data_source_target} two-dimensional projections of one dataset from the source measure and one dataset from the  target measure with their respective clustering. Note that the clustering of the target dataset is then assumed to be unknown.

\begin{figure}[htbp]
	\centering
	\includegraphics[width=\textwidth]{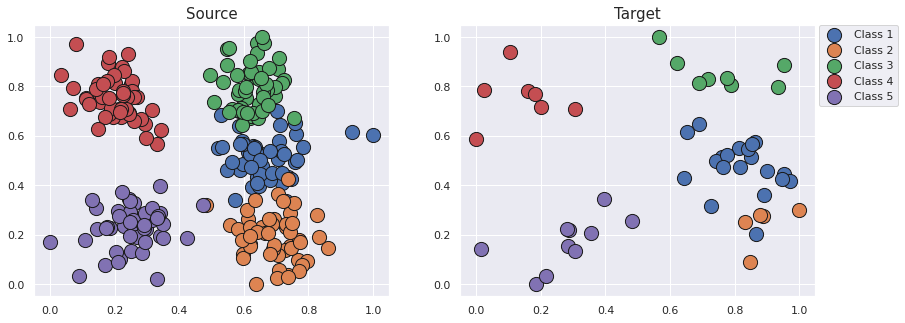}
	\caption{2D projections of a simulated source data set and a target data set with their clustering.}
	\label{fig:Simulated_data_source_target} 
\end{figure}

\subsubsection{Unlimited number of Sinkhorn iterations}\label{subsubsec:simu_nolimit}

Through a first series of experiments, we compare the performances of the estimators computed with the losses $W_0$, $W_\lambda$ and $S_\lambda$. In Figure \ref{fig:Results_Simulation_NoLimit_Variation}, using a boxplot we display the behavior of the error $\|\hat{\theta}_\lambda - \theta^*\|^2$ for each value of the regularization parameter $\lambda \in \Lambda$. In Figure \ref{fig:Result_Simulation_NoLimit_Average}, we also display the  estimation of $\mathbb{E}[\|\hat{\theta}_\lambda-\theta\|^2]$ using the Monte-Carlo estimator \eqref{eq:estimate_expectation_error}. For small values of $\lambda$, the regularized losses $W_\lambda$ and $S_\lambda$ yield competitive estimators compared to the one obtained with~$W_0$. Notice that the regularization parameter advised from Corollary \ref{cor:bound_selected_parameters} is $\lambda_n := n^{-2/d}$. In this first series of experiments $n=50$ and $d=6$, that gives $\lambda_n \approx 0.27$. This parameter $\lambda$ is slightly larger, than suggested by our empirical results from Figure \ref{fig:Results_Simulation_NoLimit_Variation}. This gap between theory and practice might be explained by the fact that we did not take into account the multiplicative constant in the approximation error. According to \cite{genevay2019sample}, this constant is of order $2d\lambda\log(1/\lambda)$. Thus, taking this constant into account would give a regularization parameter $\tilde{\lambda}_n = (2d)^{-1}\lambda \log(1/\lambda)$, which is closer to the parameters that perform best in these experiments.

\begin{figure}[htbp]
	\includegraphics[width=1\textwidth]{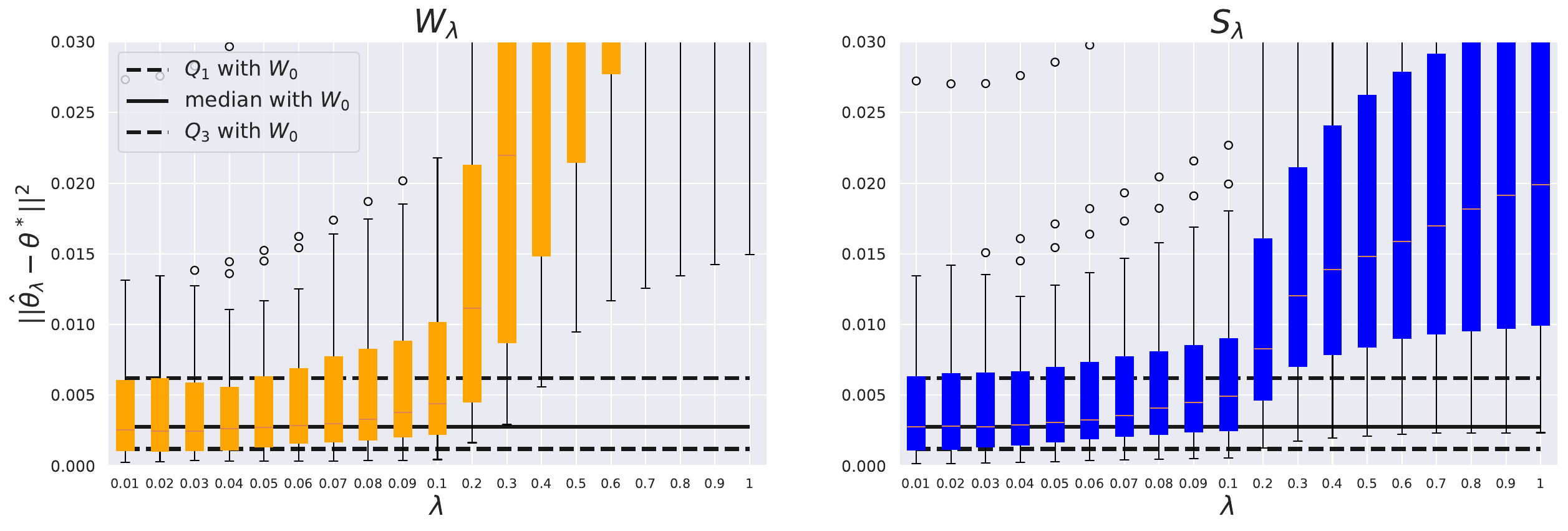}
	\caption{Estimation results on simulated data without limitation on the number of iterations of the Sinkhorn algorithm. We display the error $\|\hat{\theta}_\lambda - \theta^*\|^2$ using either the loss $W_\lambda$ (left) or $S_\lambda$ (right). The black line is the median error of the un-regularized estimator $\hat{\theta}_0$ using the standard optimal transport cost $W_0$, while the dotted lines are the first and third quartiles of the errors of estimation $\|\hat{\theta}_0- \theta^*\|^2$.  Circles are the errors $\|\hat{\theta}_\lambda-\theta^*\|^2$ beyond 1.5 times the inter-quartile range.}
	\label{fig:Results_Simulation_NoLimit_Variation} 
\end{figure}

\begin{figure}[htbp]
	\centering
	\includegraphics[width=0.7\textwidth]{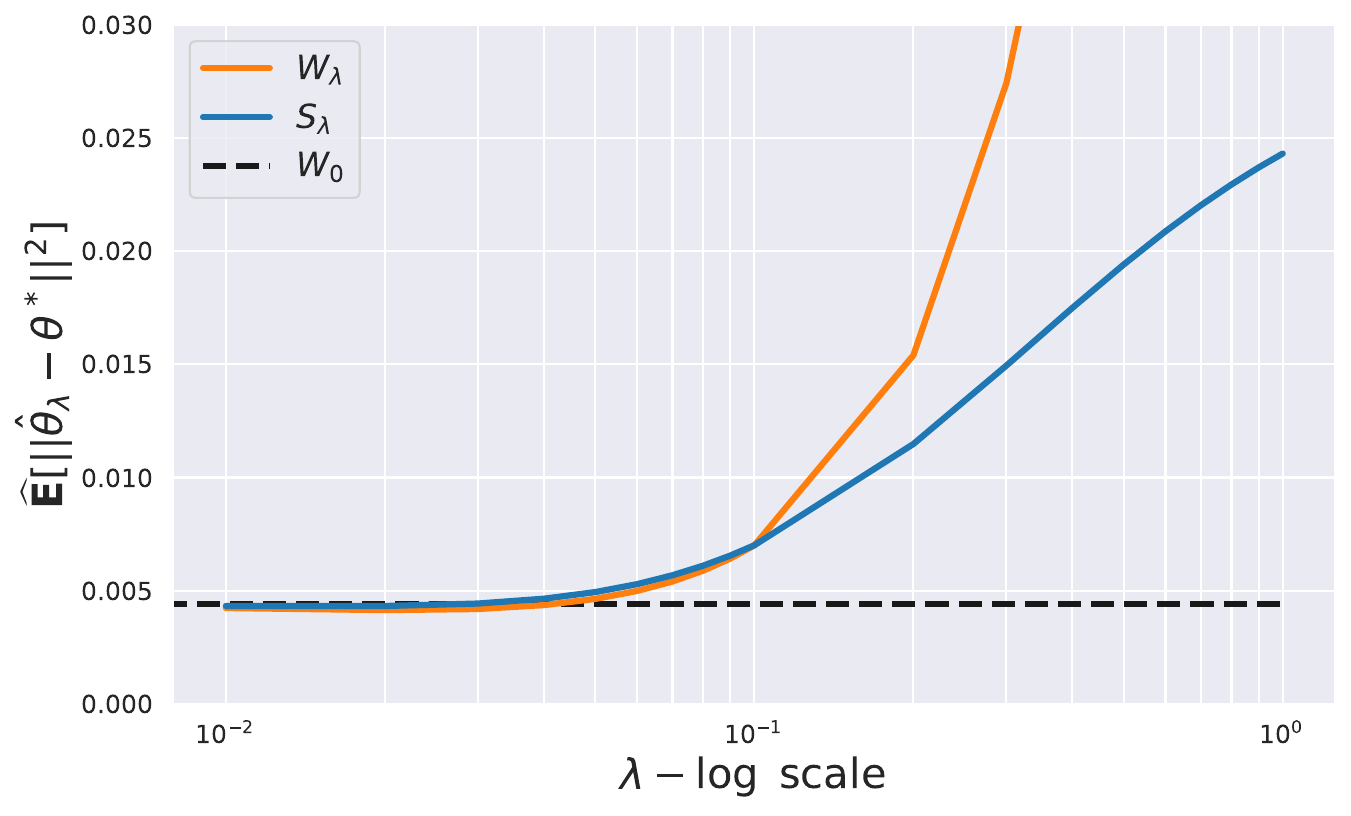}
	\caption{Average error on simulated data of the estimators $\hat{\theta}_\lambda$ (orange) and $\hat{\theta}_\lambda^S$ (blue) as a function of the regularization parameter $\lambda$. Sinkhorn algorithm runs until convergence is reached. That is, when the difference between two iterates $|W^{(\ell)}-W^{(\ell-1)}|$ is less than $10^{-9}$. The black dotted line is the average error of the un-regularized estimator $\hat{\theta}_0$.}
	\label{fig:Result_Simulation_NoLimit_Average} 
\end{figure}

We also point out that the computational complexity of the Sinkhorn algorithm is highly dependent on the regularization parameter $\lambda$ as discussed in \citep{dvurechensky2018computational}, \citep{altschuler2017near}. To illustrate this fact, we display in Figure \ref{fig:Computational_Time_SinkhornAlgorithm} the time (in seconds) required to compute $N=50$ samples of $\hat{\theta}_\lambda$ depending on the value of $\lambda$. As ${\nabla_\theta S_\lambda(\mu_\theta, \nu) =  \nabla_\theta W_\lambda(\mu_\theta, \nu) -\frac{1}{2} \nabla_\theta W_\lambda(\mu_\theta, \mu_\theta)}$, computing the gradient of $S_\lambda(\mu_\theta, \nu)$, requires to solve the dual problem associated to $W_\lambda(\mu_\theta, \mu_\theta)$ in addition to the dual problem associated to $W_\lambda(\mu_\theta, \nu)$. But as noticed in \cite{feydy2019sinkhorndiv}, Sinkhorn algorithm converges much faster for the symmetric term $W_\lambda(a, a)$ than in the general case when computing  $W_\lambda(a,b)$. We have observed in our experiment that the number of iterations before reaching convergence when computing $W_\lambda(a,a)$ does not seem to be a monotonic function with respect to the regularization parameter $\lambda$. This partially accounts for the slightly longer time of computation for $\lambda=0.02$ in comparison to $\lambda=0.01$ on the right side of Figure \ref{fig:Computational_Time_SinkhornAlgorithm}, that is when using $S_\lambda$ as loss function.
\begin{figure}[htbp]
	\includegraphics[width=1\textwidth]{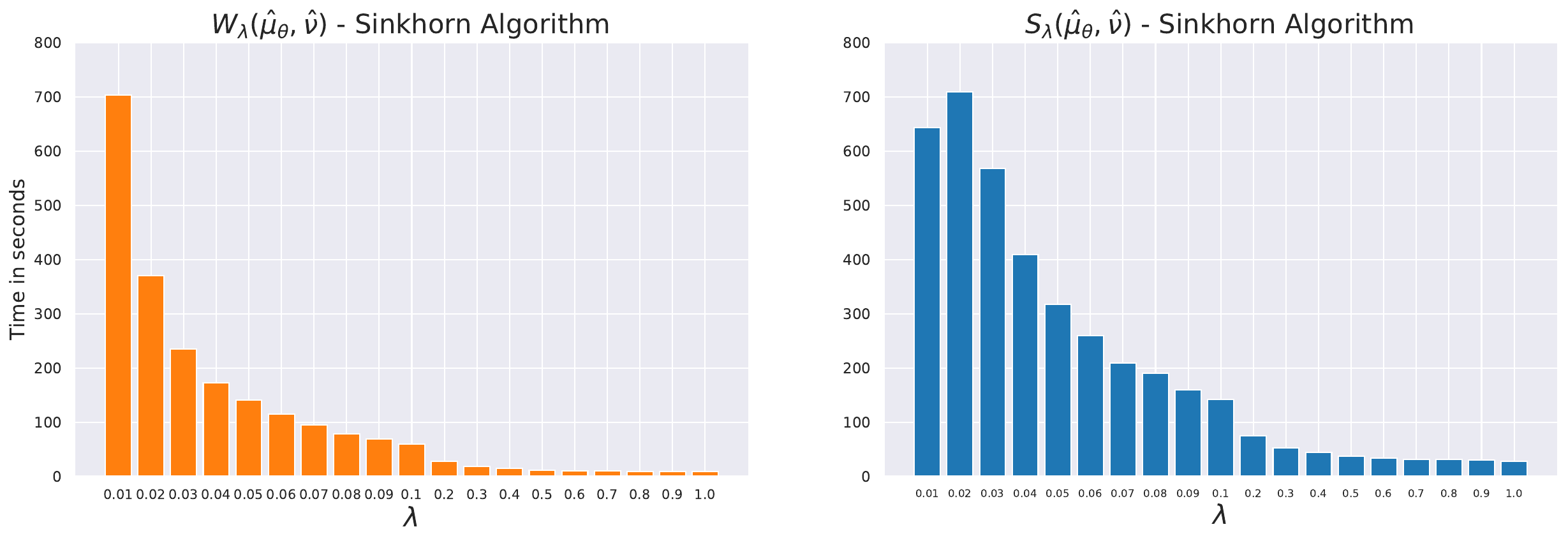}
	\caption{Time required to compute $N=50$ estimators $\hat{\theta}_\lambda$ (left) and $\hat{\theta}_\lambda^S$ (right) when the number of iterations is unlimited.}
	\label{fig:Computational_Time_SinkhornAlgorithm} 
\end{figure}

\subsubsection{Limited number of Sinkhorn iterations}

Figure  \ref{fig:Results_Simulation_NoLimit_Variation} and Figure \ref{fig:Computational_Time_SinkhornAlgorithm} presents results questionning the trade-off between the computational cost of regularized OT and the quality of statistical estimation.
We  have repeated the experiments of Section \ref{subsubsec:simu_nolimit} by now constraining the number iterations of the Sinkhorn algorithm to be equal to $\ell=5$ for any value $\lambda$. In other words, we compute the estimators $\hat{\theta}_\lambda^{(\ell)}$ and $\hat{\theta}_\lambda^{S(\ell)}$  with $\ell=5$, thus fxing the computational budget. Figure~\ref{fig:Simulation_variation_finite_Sinkhorn} and Figure~\ref{fig:Simulation_Mean_Finite_Sinkhorn} both present the performances of those estimators: by limiting the number of Sinkhorn iterations, the accuracy of the estimation deteriorates for small values of $\lambda$. This degradation comes from $\ell=5$ being too small a number of iterations  for the Sinkhorn algorithm to converge for small values of $\lambda$. Yet Figure \ref{fig:Simulation_variation_finite_Sinkhorn} points to some values of $\Lambda$  as offering a nice trade-off between the computational cost of small $\lambda$ and the approximation error induced by larger $\lambda$.  For such values, the performances of the regularized estimators $\hat{\theta}_\lambda^{(\ell)}$ and $\hat{\theta}_\lambda^{S(\ell)}$ are seen to be comparable to those of the un-regularized estimator $\hat{\theta}_0$. However, we must grant a minor divergence between our theoretical findings of Corollary \ref{cor:bound_selected_parameters}. Indeed, our theoretical results suggest that $\ell$ should be set of order $\ell_n = n^{4/d}$, which gives $\ell_n \approx 14$ in this context. We suspect two reasons for this gap. Firstly, some constants in the estimation errors $\hat{\theta}$ and $\hat{\theta}^S$ are unknown. In such a case, allowing a larger algorithm error by choosing $\ell$ smaller would not reduce the performance of estimator considered. A second source of error in these experiments is that we are not exactly under the assumptions of Corollary \ref{cor:bound_selected_parameters}. Indeed, this corollary requires measures to have compact support. While in our experiments, the probability measures are Gaussian variables, which do not have compact supports. 

\begin{figure}[ht]
	\includegraphics[width=1\textwidth]{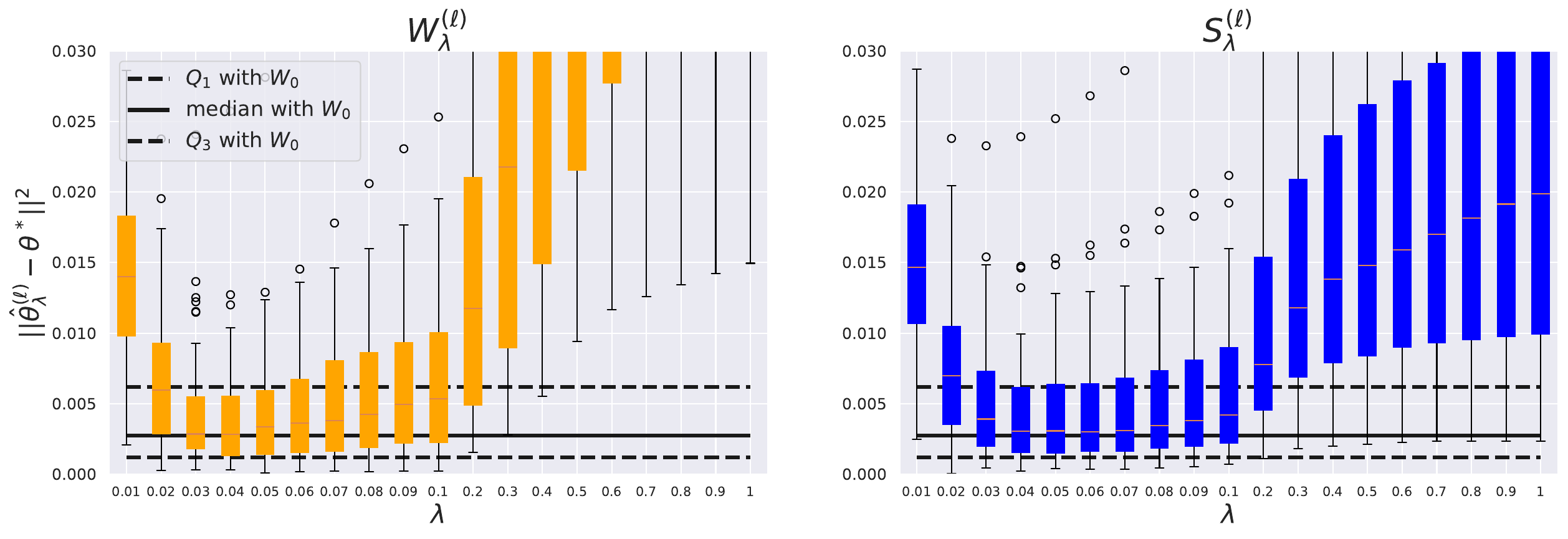}
	\caption{Estimation results on simulated data when the number of iterations of the Sinkhorn algorithm is limited to $\ell = 5$.  We display the error $\|\hat{\theta}_\lambda^{(\ell)} - \theta^*\|^2$ using either the loss $W_\lambda^{(\ell)}$ (left) or $S_\lambda^{(\ell)}$ (right). The black line is the median error of the un-regularized estimator $\hat{\theta}_0$ using the standard optimal transport cost $W_0$, while the dotted lines are the first and third quartiles of  $\|\hat{\theta}_0- \theta^*\|^2$.}
	\label{fig:Simulation_variation_finite_Sinkhorn} 
\end{figure}

\begin{figure}[htbp]
	\centering
	\includegraphics[width=0.7\textwidth]{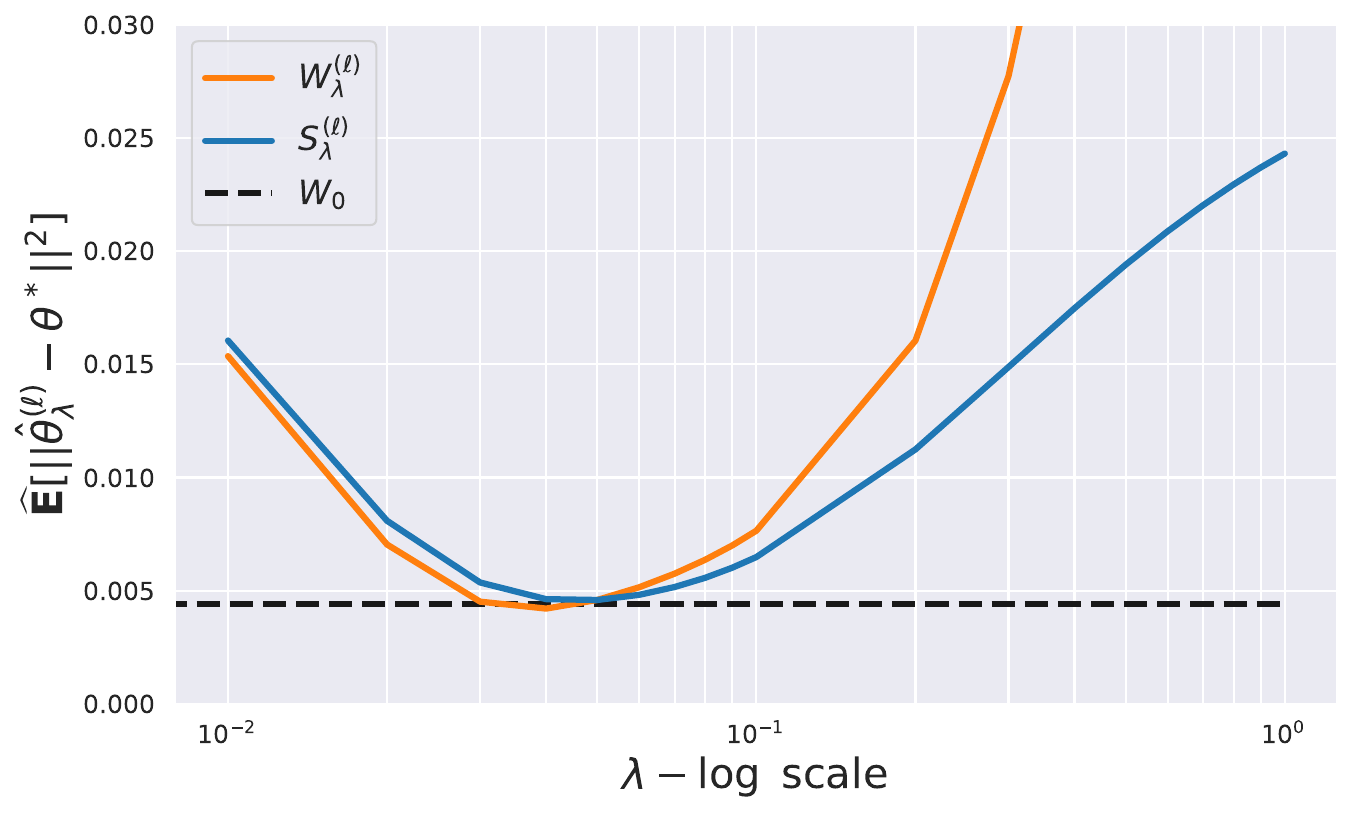}
	\caption{Average error on simulated data of the estimators $\hat{\theta}_\lambda^{(\ell)}$ (orange) and $\hat{\theta}_\lambda^{S(\ell)}$ (blue) as a function of the regularization parameter $\lambda$ with a limitation on the number of iterations. For all values of $\lambda$, Sinkhorn algorithm is limited to $\ell = 5$ iterations. The black dotted line is the average error of the un-regularized estimator $\hat{\theta}_0$.}
	\label{fig:Simulation_Mean_Finite_Sinkhorn} 
\end{figure}

\subsection{Flow cytometry data}\label{subsec:flow_cytometry_data}

We now apply our method of class proportions estimation on flow cytometry data. We demonstrate that the regularization parameter $\lambda$ has also a significant impact on the estimation of class proportions on real data. As an illustrative example, we apply our technique to  flow cytometry data sets from the T-cell panel of the Human Immunology Project Consortium (HIPC) \--- publicly available on ImmuneSpace \citep{Data_HIPC}. We arbitrarily chose two data sets that comes from cytometry measurements performed in the ``Stanford'' laboratory center. One data set, that acts as the source measure, is built from observations measured from a biological sample of a certain patient. Another  second data set,  acting as the target measure, is built from the observations obtained from a biological sample that comes from another patient. After performing  cytometry measurements the observations were manually gated into 10 cell populations: CD4 Effector (CD4~E), CD4 Naive (CD4~N), CD4 Central memory (CD4~CM), CD4 Effector memory (CD4~EM), CD4 Activated (CD4~A), CD8 Effector (CD8~E), CD8 Naive (CD8~N), CD8 Central memory (CD8~CM), CD8 Effector memory (CD8~EM) and CD8 Activated (CD8~A). Hence, for these data sets, a manual clustering is at our disposal to evaluate the performances of our method. In this context $\theta^*$ is defined as the class proportions defined thanks to the manual gating. For each cell, seven biological markers have been measured, and it thus leads to observations $X_i$ and $Y_j$ that belong to $\mathbb{R}^d$ with $d=7$. A two-dimensional projection of these datasets is displayed in Figure \ref{fig:HIPC_full_data} with the resulting manual clustering.

\begin{figure}[htbp]
	\centering
	\includegraphics[width=\textwidth]{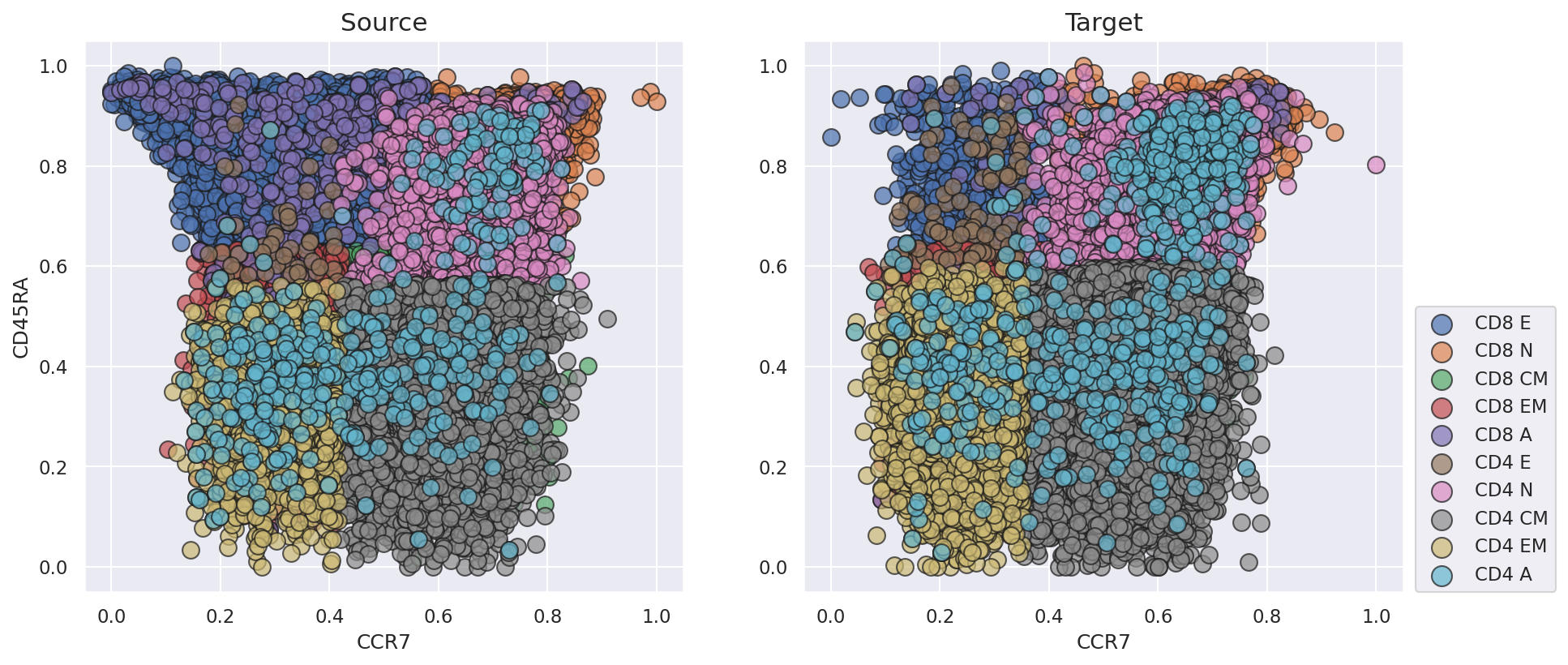}
	\caption{Two-dimensional projection of the flow cytometry datasets used in these numerical experiements with a clustering of the cells into 10 sub-populations. Note that the true dimension of the data is $d=7$, thus limiting the readability of such 2D projections.}
	\label{fig:HIPC_full_data} 
\end{figure}

\begin{figure}[htbp]
	\centering
	\includegraphics[width=\textwidth]{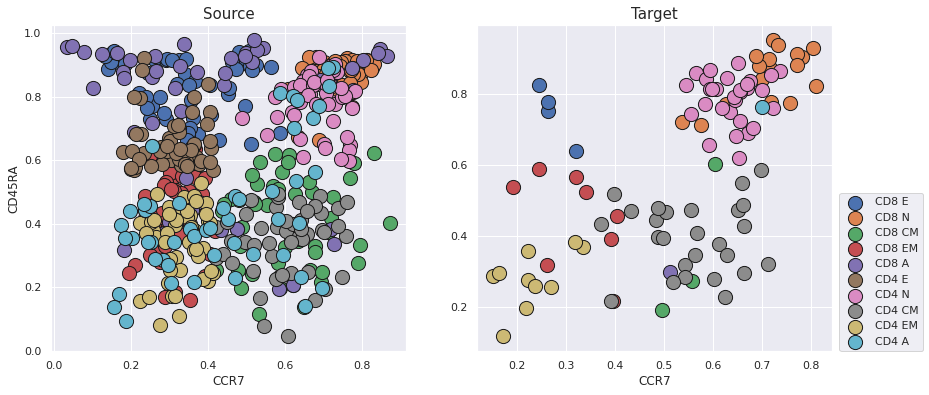}
	\caption{Sub-sample of the source and target flow cytometry data sets. In the source sub-sample, $m_k = m = 50$ elements of each class have been sampled. In the target sub-sample, $n = 100$ observations have been randomly chosen.}
	\label{fig:HIPC_Data_Subsample_with_clustering} 
\end{figure}

\subsubsection{Unlimited Sinkhorn iterations}\label{subsubsec:experiments_nolimit_HIPC}

We reproduce the protocol that we have considered in the case of simulated data. To build an empirical distribution of the source distribution when analyzing flow cytometry data,  we sub-sample $50$ observations from each class of the source data set in order to construct the empirical  measures $\hat{\mu}_1,\ldots, \hat{\mu}_K$, and to define  $\hat{\mu}_\theta = \sum_{k=1}^K \theta_k \hat{\mu}_k$  for $\theta \in \Sigma_K$. Figure \ref{fig:HIPC_Data_Subsample_with_clustering} shows two sub-samples from the source and target distributions displayed in Figure \ref{fig:HIPC_full_data}.

We recall that the clustering of the target dataset is not used in the estimation procedure.

The numerical performances of the estimators computed   with the three loss functions $W_0$, $W_\lambda$ and $S_\lambda$ are displayed on Figure \ref{fig:Results_Variations_HIPC_NoLimitation} and Figure \ref{fig:Results_Mean_HIPC_NoLimitation}. In the context of flow cytometry data, the underlying distributions $\mu$ and $\nu$ are obviously unknown, and the quantity $\min W_0(\mu_\theta, \nu)$ is thus not accessible. Therefore, we define the optimal vector $\theta^*$  of class proportions to be the one in the fully observed (not sub-sampled) target dataset that is displayed in Figure \ref{fig:HIPC_full_data}. Those results on real data are consistent with the results of simulated data. Indeed, one can observe that for small values of $\lambda \in \Lambda$ the accuracy of the estimation obtained with the loss functions $W_\lambda$ and $S_\lambda$  is very similar to the one obtained using  $W_0$.

\begin{figure}[htbp]
	\centering
	\includegraphics[width=1\textwidth]{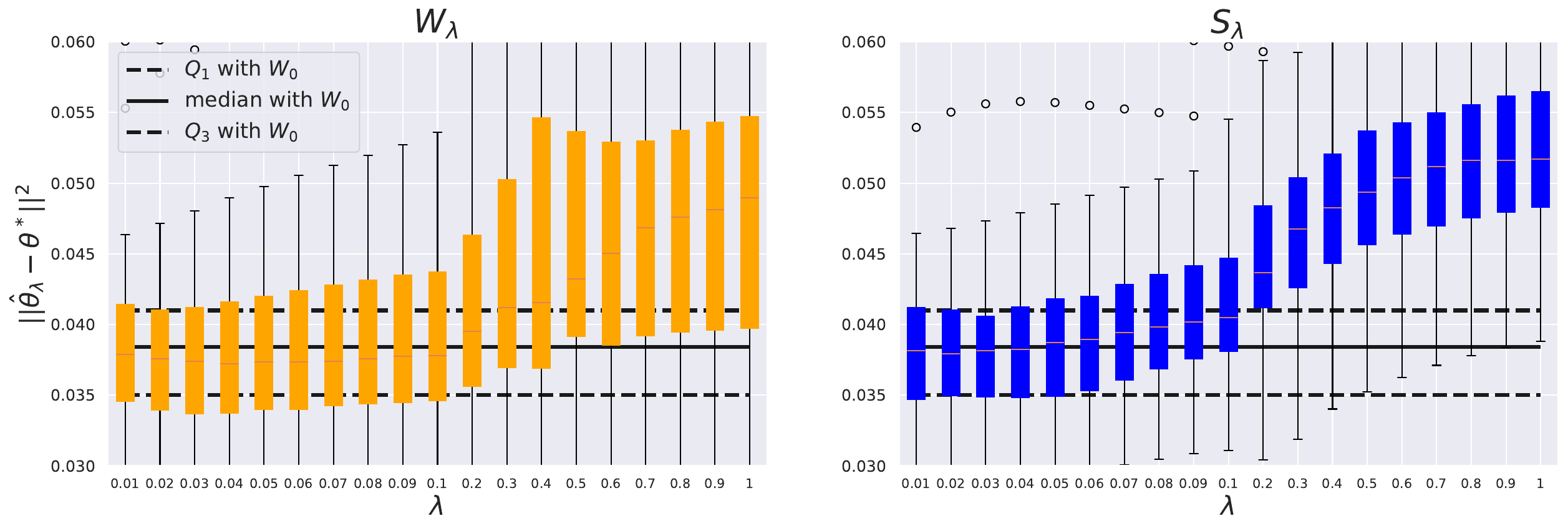}
	\caption{Results on HIPC data without imposing limitations on the number of Sinkhorn iterations. We display the error $\|\hat{\theta}_\lambda - \theta^*\|^2$ using either the loss $W_\lambda$ (left) or $S_\lambda$ (right). The black line is the median error of the un-regularized estimator $\hat{\theta}_0$ using the loss $W_0$, while the dotted lines are the first and third quartiles of  $\|\hat{\theta}_0- \theta^*\|^2$.}
	\label{fig:Results_Variations_HIPC_NoLimitation} 
\end{figure}

\begin{figure}[htbp]
	\centering
	\includegraphics[width=0.7\textwidth]{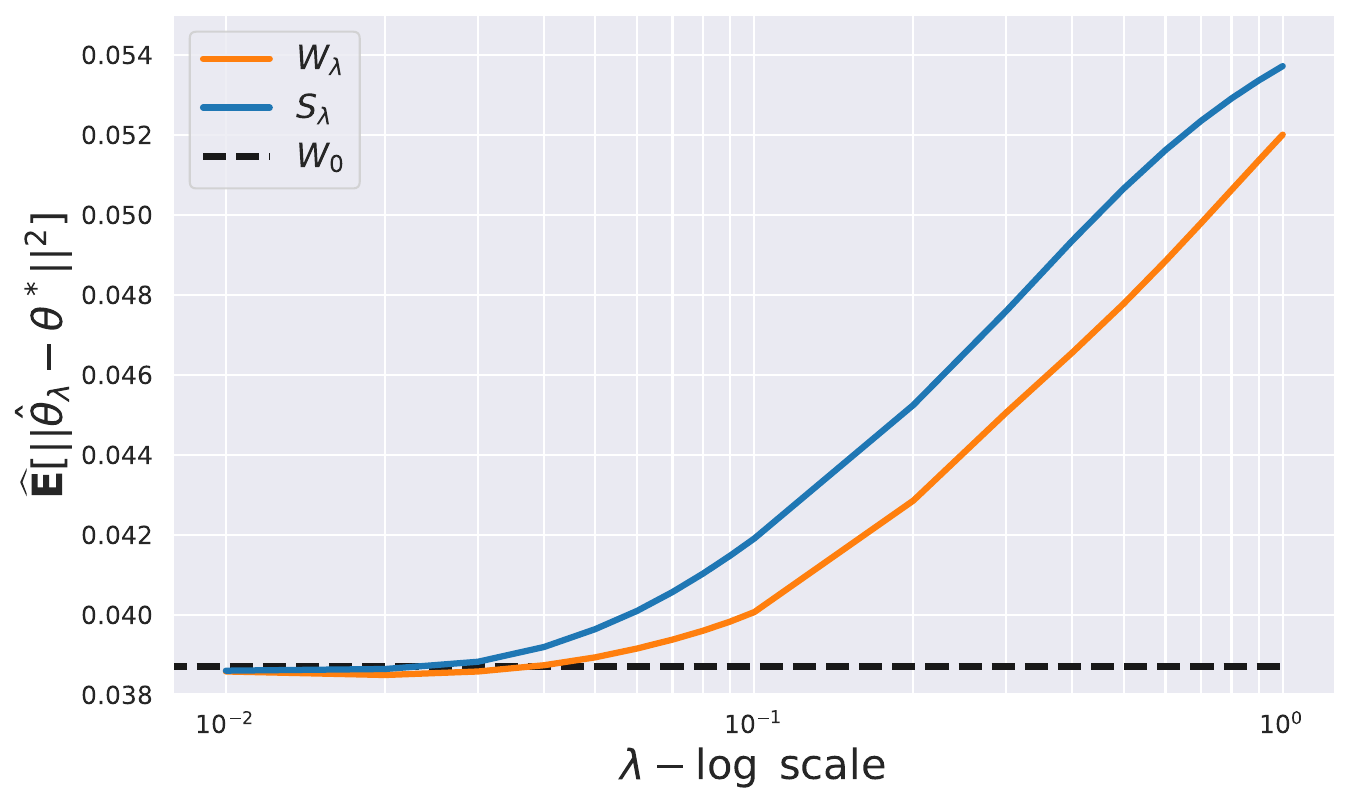}
	\caption{Average error on HIPC data of the estimators $\hat{\theta}_\lambda$ (orange) and $\hat{\theta}_\lambda^S$ (blue) as a function of the regularization parameter $\lambda$. There is no limitation on the number of iterations, Sinkhorn algorithm runs until convergence is reached. The black dotted line is the average error of the un-regularized estimator $\hat{\theta}_0$.}
	\label{fig:Results_Mean_HIPC_NoLimitation} 
\end{figure}

\subsubsection{Limited Sinkhorn iterations}

In order to reduce the computational cost of our estimation method, we limit the number of Sinkhorn iterations to $\ell=10$. Once again, the results displayed in Figure \ref{fig:Results_Variations_HIPC_WithLimitation} and Figure \ref{fig:Results_Mean_HIPC_WithLimitation} show that it is possible to propose a competitive alternative to $W_0$ at a lower computational cost.

\label{subsubsec:experiments_limited_HIPC}
\begin{figure}[htbp]
	\centering
	\includegraphics[width=1\textwidth]{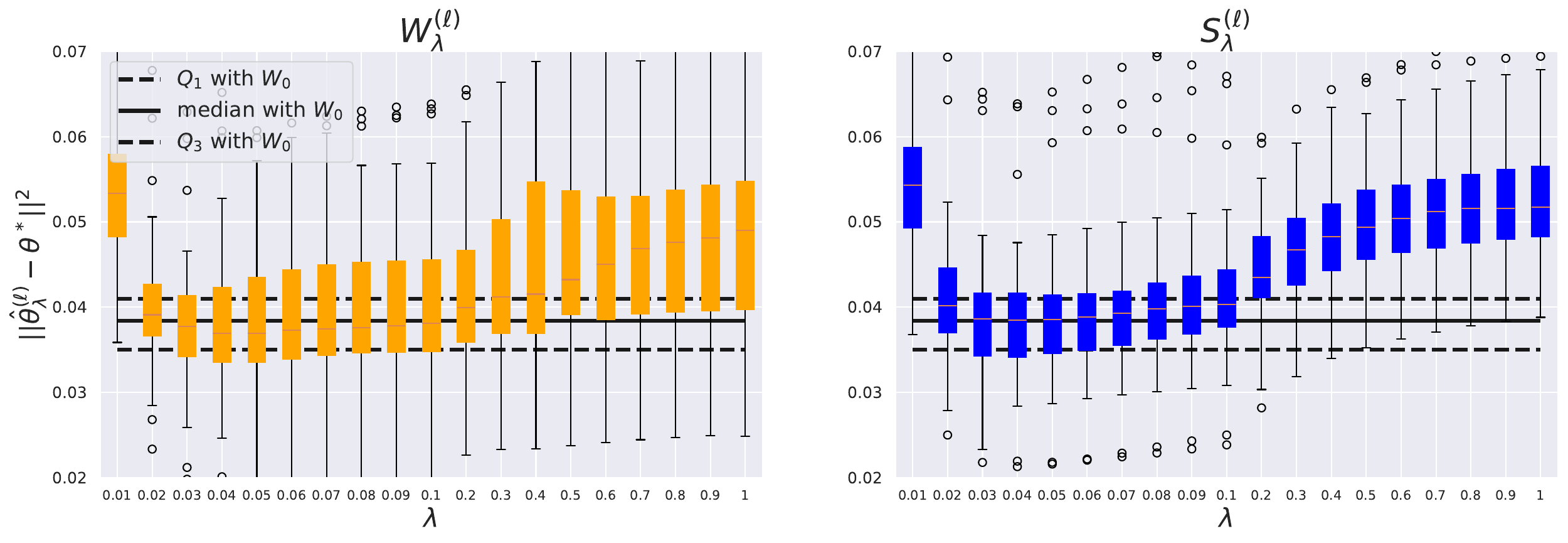}
	\caption{Results on HIPC data when the number of Sinkhorn iterations is limited to $\ell = 10$.  We display boxplots of the error $\|\hat{\theta}_\lambda - \theta^*\|^2$ using either the loss $W_\lambda$ (left) or $S_\lambda$ (right). The black line is the median error of the un-regularized estimator $\hat{\theta}_0$ using the loss $W_0$, while the dotted lines are the first and third quartiles of  $\|\hat{\theta}_0- \theta^*\|^2$.}
	\label{fig:Results_Variations_HIPC_WithLimitation} 
\end{figure}

\begin{figure}[htbp]
	\centering
	\includegraphics[width=0.7\textwidth]{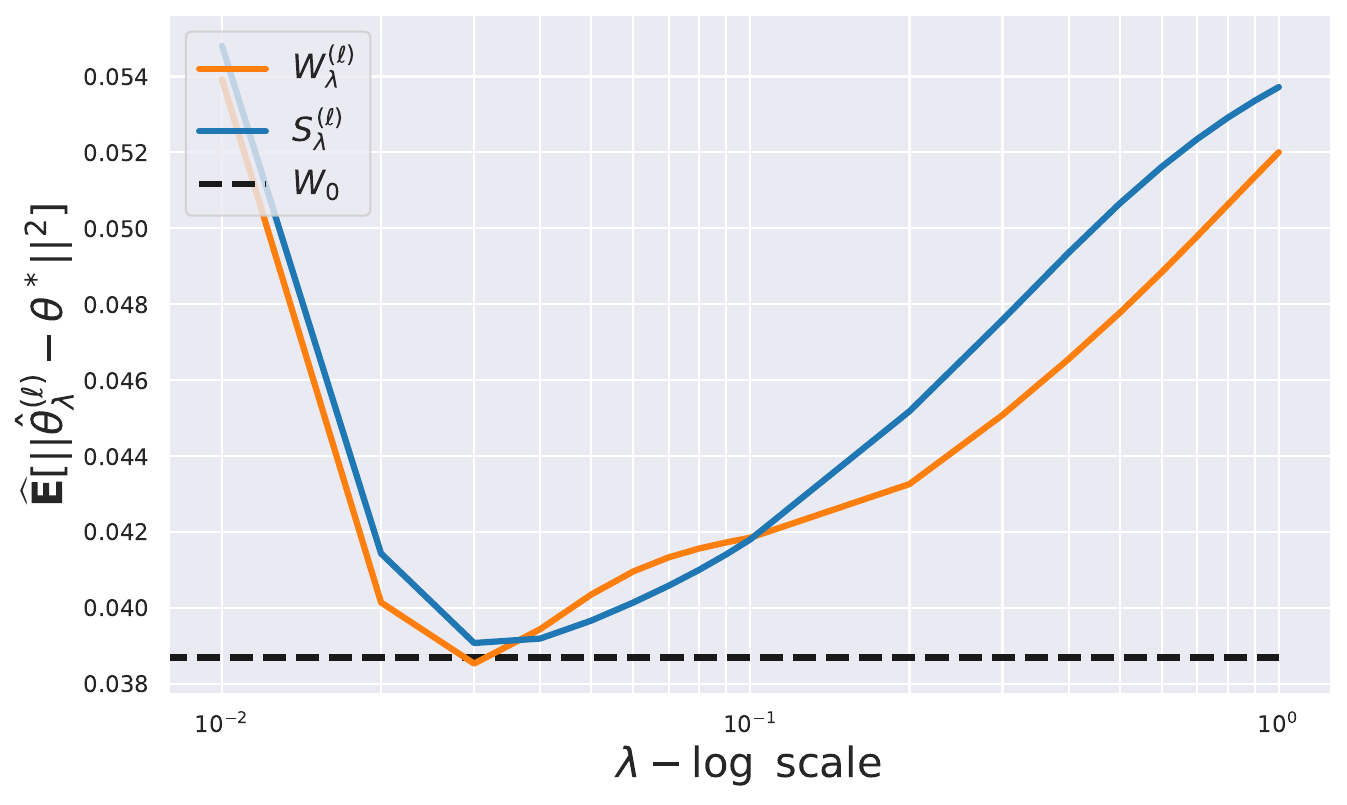}
	\caption{Average error on HIPC data of the estimators $\hat{\theta}_\lambda^{(\ell)}$ (orange) and $\hat{\theta}_\lambda^{S(\ell)}$ (blue) as a function of the regularization parameter $\lambda$ with a limitation on the number of iterations. For all values of $\lambda$, Sinkhorn algorithm is limited to $\ell = 10$ iterations. The black dotted line is the average error of the un-regularized estimator $\hat{\theta}_0$}
	\label{fig:Results_Mean_HIPC_WithLimitation} 
\end{figure}


\section{Conclusion and discussion}\label{sec:diss}

In this work, we have presented a thorough study of Wasserstein estimators based on regularized OT with an emphasis on the influence of the regularization parameter $\lambda$. This study was carried out through the example of a mixture model and weights estimation.  We derived upper bounds on the risk of Wasserstein estimators in terms of an estimation error and an approximation error. We assessed the influence of the chosen OT-based loss (among $W_\lambda, S_\lambda$ and $W_0$) on the decay of the estimation and approximation terms. We have also proposed an optimal decay  of the regularization parameter $\lambda = \lambda_n$ 
based on these upper bounds to achieve decreasing rate of $n^{-2/d}$ for the expected excess risk. Secondly, motivated by the sensitive question of the computational cost of regularized OT, we have studied the algorithmic error induced by limiting the number of iterations in the Sinkhorn algorithm.  This study resulted in a principled strategy to set the number of Sinkhorn iterations $\ell = \ell_n$ in order to maintain the algorithm error below the statistical error. We have also demonstrated with numerical experiments that an appropriate choice of $\lambda$ and a limited number of Sinkhorn iterations $\ell$ allow to equal the performances of the un-regularized estimator at a reduced computational cost.\\

We now present a few perspectives for future research. For an estimator $\hat{\theta}_n$ of $\theta^*$, we have derived a control on the excess risk, that is $W_0(\mu_{\hat{\theta}_n}, \nu) - W_0(\mu_{\theta^*}, \nu)$. However, a direct control of the weights estimator, i.e. of the quantity $\|\hat{\theta}_n - \theta^*\|$, would be even more valuable. For instance, a control of $\|\hat{\theta}_n - \theta^*\|$  may allow to develop statistical tests on the estimator $\hat{\theta}_n$. An other possible extension of this work is suggested by our numerical experiments. Figure \ref{fig:Results_Simulation_NoLimit_Variation} and Figure \ref{fig:Simulation_variation_finite_Sinkhorn} indicate that limiting the number of iterations for Sinkhorn algorithm could improve statistical performance. These better results with limited iterations are not accounted by the present work. Hence, further investigation on this observation is an other direction for research.

\bibliographystyle{abbrv}
\bibliography{regularized_wasserstein_estimators}

\newpage

\appendix

\section{Proofs of the main results} \label{sec:convergence}

The goal of this section is to derive the rate of convergence of regularized Wasserstein estimators. That is, when considering $W_\lambda$ as a loss function with $\lambda \geq 0$. To be more specific, we investigate in this section the estimators
\begin{equation}\label{eq:theta_hat_lambda}
	\hat{\theta}_{\lambda} := \argmin_{\theta \in \Sigma_K} W_{\lambda}(\hat{\mu}_{\theta},\hat{\nu})  \quad \text{for}~ \lambda \geq 0,
\end{equation}
or
\begin{equation}\label{eq:theta_hat_lambda_ell}
	\hat{\theta}_\lambda^{(\ell)} \in \argmin_{\theta \in \Sigma_K}W_\lambda^{(\ell)}(\hat{\mu}_\theta, \hat{\nu}) \quad \text{for}~ \lambda > 0,
\end{equation}
when taking into account the algorithm error.

\subsection{Decomposition of the excess risk}\label{subsec:estimation_approximation_reg_wass}

We first detail how the excess risk of $\hat{\theta}_{\lambda}^{(\ell)}$ defined in equation \eqref{eq:theta_hat_lambda_ell} can be upper bounded by the sum of three terms. They represent a tradeoff between an approximation error, estimation error and an algorithm error. 

\begin{lem}\label{lem:approx_estim_algo}
	Set $\lambda \geq 0$. the excess risk of  the estimator  $\hat{\theta}_\lambda^{(\ell)}$ defined by \eqref{eq:theta_hat_lambda_ell} is bounded as follows:
	\begin{align}\label{eq:supp_decomposition}
		0 \leq 	W_0(\mu_{\hat{\theta}_\lambda^{(\ell)}}, \nu) -W_0(\mu_{\theta^*}, \nu) \leq & ~ 2\sup_{\theta \in \Sigma_K}|W_0(\mu_{\theta}, \nu) - W_\lambda(\mu_{\theta}, \nu)|\\
		& + 2 \sup_{\theta \in \Sigma_K}|W_\lambda(\mu_{\theta}, \nu) - W_\lambda(\hat{\mu}_{\theta}, \hat{\nu})| \nonumber \\ 
		& + 2\sup_{\theta \in \Sigma_K}|W_\lambda(\hat{\mu}_{\theta}, \hat{\nu}) - W_\lambda^{(\ell)}(\hat{\mu}_{\theta}, \hat{\nu})| \nonumber.
	\end{align}
\end{lem}

\begin{proof}
	We begin with the decomposition
	\begin{align}\label{eq:regwass_decomposition_fullsampling}
		W_0(\mu_{\hat{\theta}_\lambda^{(\ell)}}, \nu) -W_0(\mu_{\theta^*}, \nu)~& = ~W_0(\mu_{\hat{\theta}_\lambda^{(\ell)}}, \nu) -  W_\lambda(\mu_{\hat{\theta}_\lambda^{(\ell)}}, \nu) +  W_\lambda(\mu_{\hat{\theta}_\lambda^{(\ell)}}, \nu) - W_\lambda(\hat{\mu}_{\hat{\theta}_\lambda^{(\ell)}}, \hat{\nu}) \nonumber \\ & + W_\lambda(\hat{\mu}_{\hat{\theta}_\lambda^{(\ell)}}, \hat{\nu}) - W_\lambda^{(\ell)}(\hat{\mu}_{\hat{\theta}_\lambda^{(\ell)}}, \hat{\nu}) + W_\lambda^{(\ell)}(\hat{\mu}_{\hat{\theta}_\lambda^{(\ell)}}, \hat{\nu}) - W_\lambda^{(\ell)}(\hat{\mu}_{\theta^{*}}, \hat{\nu}) \nonumber \\ 
		& +  W_\lambda^{(\ell)}(\hat{\mu}_{\theta^*}, \hat{\nu}) -  W_\lambda(\hat{\mu}_{\theta^*}, \hat{\nu}) + W_\lambda(\hat{\mu}_{\theta^*}, \hat{\nu}) - W_\lambda(\mu_{\theta^*}, \nu)   \nonumber \\
		& +W_\lambda(\mu_{\theta^*}, \nu) -  W_0(\mu_{\theta^*}, \nu).
	\end{align}
	Let us focus on the right hand side of this last equation \eqref{eq:regwass_decomposition_fullsampling}.
	The first and the last differences are controlled by the approximation error $\sup_{\theta \in \Sigma_K}|W_\lambda(\mu_{\theta}, \nu) - W_0(\mu_{\theta}, \nu)|$. The second and sixth differences can be upper bounded by the estimation error ${\sup_{\theta \in \Sigma_K}|W_\lambda(\mu_{\theta}, \nu) - W_\lambda(\hat{\mu}_{\theta}, \hat{\nu})|}$. The third and fifth differences are upper bounded by the algorithm error ${\sup_{\theta \in \Sigma_K}|W_\lambda^{(\ell)}(\hat{\mu}_{\theta}, \hat{\nu}) - W_\lambda(\hat{\mu}_{\theta}, \hat{\nu})|}$.\\
	
	It only remains to control $ W_\lambda^{(\ell)}(\hat{\mu}_{\hat{\theta}_\lambda^{(\ell)}}, \hat{\nu}) - W_\lambda^{(\ell)}(\hat{\mu}_{\theta^{*}}, \hat{\nu})$. However, $\hat{\theta}_\lambda^{(\ell)}$ minimizes the function $\theta \mapsto W_\lambda^{(\ell)}(\hat{\mu}_{\theta}, \hat{\nu})$. Hence $W_\lambda^{(\ell)}(\hat{\mu}_{\hat{\theta}_\lambda^{(\ell)}}, \hat{\nu}) - W_\lambda^{(\ell)}(\hat{\mu}_{\theta^{*}}, \hat{\nu}) \leq 0$.\\
	
	Going back to equation \eqref{eq:regwass_decomposition_fullsampling} and substituting every difference of the right hand side by its appropriate bound we derive
	\begin{align*}
		W_0(\mu_{\hat{\theta}_\lambda^{(\ell)}}, \nu) -W_0(\mu_{\theta^*}, \nu) \leq & ~ 2\sup_{\theta \in \Sigma_K}|W_0(\mu_{\theta}, \nu) - W_\lambda(\mu_{\theta}, \nu)|\\
		& + 2 \sup_{\theta \in \Sigma_K}|W_\lambda(\mu_{\theta}, \nu) - W_\lambda(\hat{\mu}_{\theta}, \hat{\nu})| \\ 
		& + 2\sup_{\theta \in \Sigma_K}|W_\lambda(\hat{\mu}_{\theta}, \hat{\nu}) - W_\lambda^{(\ell)}(\hat{\mu}_{\theta}, \hat{\nu})|,
	\end{align*}
	
	which is the result claimed in Lemma \ref{lem:approx_estim_algo}.
	
\end{proof}

\subsection{Control of the estimation error}\label{subsec:control_estimation}
To control the estimation error,  we split it into two terms: 
\begin{align}\label{eq:Wlambda_decomposition_estimation}
	\sup_{\theta \in \Sigma_K}|W_\lambda(\mu_{\theta}, \nu) - W_\lambda(\hat{\mu}_{\theta}, \hat{\nu})|
	& \leq \sup_{\theta \in \Sigma_K}|W_\lambda(\mu_{\theta}, \hat{\nu}) - W_\lambda(\mu_{\theta}, \nu)| \\ \nonumber
	& \quad \quad + \sup_{\theta \in \Sigma_K}|W_\lambda(\hat{\mu}_{\theta}, \hat{\nu}) - W_\lambda(\mu_{\theta}, \hat{\nu})| .
\end{align}

Hence, controlling the expected estimation error boils down to controlling the (closely related) quantities
\begin{equation*}
	\mathbb{E} \left[ \sup_{\theta \in \Sigma_K} \Big| W_{\lambda}(\mu_{\theta},\nu) - W_{\lambda}(\mu_{\theta},\hat{\nu}) \Big| \right] ,
	\quad \text{and} \quad
	\mathbb{E} \left[ \sup_{\theta \in \Sigma_K} \Big| W_{\lambda}(\hat{\mu}_{\theta},\hat{\nu}) - W_{\lambda}(\mu_{\theta},\hat{\nu}) \Big| \right]. \label{eq:control_emp_processes}
\end{equation*}

The upper bounds we will derive are based on the supremum of an empirical process that has been previously studied by Chizat et al. in \cite{chizat2020sinkhorn}. More precisely, we rely on the following Lemma.

\begin{lem}\label{lem:bound_empirical_process_chizat}
	\cite[Lemma 4 and proof of Theorem 2] {chizat2020sinkhorn} 
	Assume that all probability measures have compact supports included in $B(0,R)$, and that $n$ samples for $\nu$ are available. Then,
	\begin{equation}\label{eq:bound_chizat_process_hard}
		\mathbb{E}\left[\sup_{\varphi  \in  \mathcal{F}_R} \left| \int \varphi d(\nu - \hat{\nu}) \right|\right]  \lesssim 
		\left\{ \begin{array}{lll}
			R^2 n^{-1/2} & \text{if} & d<4,\\
			R^2 n^{-1/2}\log(n)& \text{if} & d=4,\\
			R^2 n^{-2/d} & \text{if} & d>4,
		\end{array}
		\right.
	\end{equation}
	where $\lesssim$ hides a constant that depends only on $d$, and $\mathcal{F}_R$ denotes the class of concave and $R$-Lipschitz functions on $B(0,R)$. In the same paper, the authors established that 
	\begin{equation}\label{eq:bound_chizat_process_easier}
		\mathbb{E}\left[ \left| \int_{\mathcal{Y}} \|y\|^2 d(\nu-\hat{\nu})(y) \right|\right] \leq 4 R^2 n^{-1/2}.
	\end{equation}
	
\end{lem}

As we make a repeating use of the upper bound in equation \eqref{eq:bound_chizat_process_hard} we denote it by $\mathcal{E}(d,n)$. From now on 
\begin{equation}\label{eq:constant_chizat}
	\mathcal{E}(d,n) := \left\{ \begin{array}{lll}
		R^2 n^{-1/2} & \text{if} & d<4,\\
		R^2 n^{-1/2}\log(n)& \text{if} & d=4,\\
		R^2 n^{-2/d} & \text{if} & d>4.
	\end{array}
	\right.
\end{equation}

The next proposition gives an upper bound of the estimation error that is independent of the regularization parameter $\lambda$.

\begin{prop}\label{prop:estimation_regfree}
	Let $\lambda \geq 0$.
	Suppose that every probability measure considered has compact support included in $B(0,R)$.  
	\begin{enumerate}[(i)]
		\item If $n$ samples from $\nu$ are available, then it holds that
		
		\begin{equation}\label{eq:bound_estimation_target_sampling_noreg}
			\mathbb{E}\left[\sup_{\theta \in \Sigma_K}|W_{\lambda}(\mu_{\theta}, \nu) - W_{\lambda}(\mu_\theta, \hat{\nu})| \right]
			\lesssim 
			\mathcal{E}(n,d).
		\end{equation}
		
		\item If for each distribution $\mu_k$, $m_k$ samples are available, then
		\begin{equation}\label{eq:bound_estimation_full_sampling_noreg}
			\mathbb{E}\left[\sup_{\theta \in \Sigma_K}|W_{\lambda}(\mu_{\theta}, \hat{\nu}) - W_{\lambda}(\hat{\mu}_\theta, \hat{\nu})|\right]  \lesssim 
			K\mathcal{E}(\underbar{m},d),
		\end{equation}
		where $\underbar{m} = \min(m_1, \ldots, m_K)$.
		
	\end{enumerate}
\end{prop}

\begin{proof}
	The key point is to exploit the alternative dual formulation of regularized OT that has been introduced in Section 2.2 of the article. Using~relation (2.11) of the article, we remark that for any $\theta \in \Sigma_K$,
	\begin{align}\label{eq:inproof_link_regularized_distances}
		W_\lambda(\mu_{\theta}, \nu) - W_\lambda(\mu_{\theta}, \hat{\nu})& = \int_{\mathcal{Y}} \|y\|^2 d\nu(y) - \int_{\mathcal{Y}} \|y\|^2 d\hat{\nu}(y) + W_\lambda^s(\mu_{\theta}, \nu) - W_\lambda^s(\mu_{\theta}, \hat{\nu}) \nonumber \\
		& = \int_{\mathcal{Y}} \|y\|^2 d(\nu - \hat{\nu})(y)  + W_\lambda^s(\mu_{\theta}, \nu) - W_\lambda^s(\mu_{\theta}, \hat{\nu}).
	\end{align}
	
	Now, let us denote by $\varphi$ and $\hat{\varphi}$ two optimal dual potentials respectively associated to $W_\lambda^s(\mu_{\theta}, \nu)$ and $W_\lambda^s(\mu_{\theta}, \hat{\nu})$ when exploiting the semi-dual formulation 2.14 of the article. We can thus write
	$$
	\begin{aligned}
		W_\lambda^s(\mu_{\theta}, \nu) - W_\lambda^s(\mu_{\theta}, \hat{\nu})& =  \int \varphi(x) d\mu_{\theta}(x) + \int \varphi^s(y) d\nu(y) \\
		& - \left(\int \hat{\varphi}(x) d\mu_{\theta}(x) + \int \hat{\varphi}^s(y) d\hat{\nu}(y) \right)\\
		& = \int \varphi^s(y) d\nu(y) - \int \varphi^s(y) d\hat{\nu}(y) \\
		& + \int \varphi(x) d\mu_{\theta}(x) + \int \varphi^s(y) d\hat{\nu}(y) \\
		& - \left(\int \hat{\varphi}(x) d\mu_{\theta}(x) + \int \hat{\varphi}^s(y) d\hat{\nu}(y)\right)\\
		& \leq \int \varphi^s(y) d(\nu-\hat{\nu})(y),
	\end{aligned}
	$$
	where the last inequality derives from the optimality of $\hat{\varphi}$ for the semi-dual formulation of $W_\lambda^s(\mu_{\theta}, \hat{\nu})$. A similar reasoning yields 
	$$ 
	W_\lambda^s(\mu_{\theta}, \hat{\nu}) - W_\lambda^s(\mu_{\theta}, \nu) \leq \int \hat{\varphi}^s(y) d(\hat{\nu}-\nu)(y).$$
	As $\varphi^s$ and $\hat{\varphi}^s$ are both $s$-transform, Proposition 2.2 in the article ensures that both $\varphi^s$ and $\hat{\varphi}$ are concave and $R$-Lipschitz on $B(0,R)$. We deduce the upper bound
	\begin{equation}\label{eq:inproof_control_estimation_RLip}
		|W_\lambda^s(\mu_{\theta}, \nu) - W_\lambda^s(\mu_{\theta}, \hat{\nu})| \leq \sup_{\varphi \in \mathcal{F}_R} \left|\int_\mathcal{Y} \varphi d(\nu-\hat{\nu})\right|,
	\end{equation}
	where $\mathcal{F}_R$ denotes the class of concave and $R$-Lipschitz functions on $B(0,R)$.
	Taking the expectation of inequality \eqref{eq:inproof_control_estimation_RLip}, point (i) of Proposition~\ref{prop:estimation_regfree} follows from Lemma~\ref{lem:bound_empirical_process_chizat}.\\
	
	Point (ii) of Proposition~\ref{prop:estimation_regfree} can be obtained with a similar reasoning. That is, we exploit the relation between $W_\lambda$ and $W_\lambda^s$. Moreover, Proposition 3.1 ensures that the optimal potentials associated to $W_\lambda^s$ can be chosen $R$-Lipschitz and concave. Performing the same computations as in point (i), and decomposing $\mu_\theta$ into a convex combination of $\mu_1, \ldots, \mu_K$, we derive   
	
	\begin{align*}
		|W_{\lambda}(\mu_\theta, \hat{\nu}) - W_{\lambda}(\hat{\mu}_\theta, \hat{\nu})| & \leq \sup_{\varphi \in \mathcal{F}_R}\left| \int_\mathcal{X} \varphi d(\mu_\theta - \hat{\mu}_\theta) \right| + \left| \int_{\mathcal{X}} \|x\|^2 d(\mu_\theta- \hat{\mu}_\theta) \right| \\
		& \leq \sum_{k=1}^K \theta_k \left( \sup_{\varphi \in \mathcal{F}_R}\left| \int_{\mathcal{X}} \varphi d(\mu_k - \hat{\mu}_k) \right| + \left| \int_{\mathcal{X}} \|x\|^2 d(\mu_k- \hat{\mu}_k) \right| \right)\\
		& \leq \sum_{k=1}^K \left( \sup_{\varphi \in \mathcal{F}_R}\left| \int_{\mathcal{X}} \varphi d(\mu_k - \hat{\mu}_k) \right| + \left| \int_{\mathcal{X}} \|x\|^2 d(\mu_k- \hat{\mu}_k) \right| \right),		
	\end{align*}
using the crude upper bound $\theta_k \leq 1$. This last upper bound being independent of $\theta$, we can write 
\begin{align*}
	\sup_{\theta \in \Sigma_K}|W_{\lambda}(\mu_\theta, \hat{\nu}) - W_{\lambda}(\hat{\mu}_\theta, \hat{\nu})|  \leq \sum_{k=1}^K \left( \sup_{\varphi \in \mathcal{F}_R}\left| \int_{\mathcal{X}} \varphi d(\mu_k - \hat{\mu}_k) \right| + \left| \int_{\mathcal{X}} \|x\|^2 d(\mu_k- \hat{\mu}_k) \right| \right).
\end{align*}
Taking the expectation of last inequality yields
\begin{align*}
	\mathbb{E}\left[\sup_{\theta \in \Sigma_K}|W_{\lambda}(\mu_\theta, \hat{\nu}) - W_{\lambda}(\hat{\mu}_\theta, \hat{\nu})| \right] \leq \sum_{k=1}^K \mathbb{E}\left[\left( \sup_{\varphi \in \mathcal{F}_R}\left| \int_{\mathcal{X}} \varphi d(\mu_k - \hat{\mu}_k) \right| + \left| \int_{\mathcal{X}} \|x\|^2 d(\mu_k- \hat{\mu}_k) \right| \right) \right].
\end{align*}
Next, applying Lemma \ref{lem:bound_empirical_process_chizat} to the probability distribution $\mu_k$, we obtain
$$
\mathbb{E}\left[ \sup_{\varphi \in \mathcal{F}_R}\left| \int_{\mathcal{X}} \varphi d(\mu_k - \hat{\mu}_k) \right| \right] +\mathbb{E}\left[ \left| \int_{\mathcal{X}} \|x\|^2 d(\mu_k-\hat{\mu}_k) \right|\right] \lesssim 
\mathcal{E}(d, m_k),
$$
where $\mathcal{E}(d,m_k)$ is defined in equation \eqref{eq:constant_chizat}. It follows that 
$$
\mathbb{E}\left[\sup_{\theta \in \Sigma_K}|W_{\lambda}(\mu_\theta, \hat{\nu}) - W_{\lambda}(\hat{\mu}_\theta, \hat{\nu})| \right] \lesssim \sum_{k=1}^K \mathcal{E}(d, m_k).
$$
Finally, introducing the notation $\underbar{m} = \min(m_1, \ldots, m_K)$, we derive 
$$
\mathbb{E}\left[\sup_{\theta \in \Sigma_K}|W_{\lambda}(\mu_{\theta}, \hat{\nu}) - W_{\lambda}(\hat{\mu}_\theta, \hat{\nu})|\right]  \lesssim 
K \mathcal{E}(d,\underbar{m}),
$$
which gives the last inequality of Proposition~\ref{prop:estimation_regfree}.
\end{proof}

We now gather the pieces to prove Proposition \ref{prop:reg_free_estimation}. Taking the expectation of equation \eqref{eq:control_emp_processes} and using Proposition \ref{prop:estimation_regfree} gives 
$$ 	\mathbb{E} \left[ \sup_{\theta \in \Sigma_K} \Big| W_{\lambda}(\mu_{\theta},\nu) - W_{\lambda}(\mu_{\theta},\hat{\nu}) \Big| \right] \lesssim \mathcal{E}(n,d) + K \mathcal{E}(\underbar{m},d).$$
Under the assumption of Proposition \ref{prop:reg_free_estimation} from the main article that we have access to $n$ samples from each probability distribution $\mu_1, \ldots, \mu_K$, as well as for $\nu$; last inequality reads
$$\mathbb{E} \left[ \sup_{\theta \in \Sigma_K} \Big| W_{\lambda}(\mu_{\theta},\nu) - W_{\lambda}(\mu_{\theta},\hat{\nu}) \Big| \right] \lesssim K \mathcal{E}(n,d), $$
which is the result announced in Proposition \ref{prop:reg_free_estimation} of the main article.

\subsection{Collecting existing results to prove the main Theorem}\label{subsec:proof_main}

\paragraph{Approximation error}

Thanks to \cite[Theorem 1]{genevay2019sample} adapted to the squared Euclidean cost $c(x,y) = \|x-y\|^2$ (which is $R$-Lipschitz on $B(0,R)$ w.r.t. both its variables), we can control the impact of entropic regularization on the approximation of the value of the un-regularized OT cost.

\begin{prop} \cite[Theorem 1]{genevay2019sample}\label{prop:control_Wlambda}
	Assume that $\mathcal{X}, \mathcal{Y}$ are compact subsets of  $B(0,R)$. Then, it holds that
	\begin{equation}
		0 \leq W_{\lambda}(\mu,\nu)  - W_{0}(\mu,\nu) \leq 2 d \lambda \log \left( \frac{8 \exp(2) R^2}{\sqrt{d} \lambda} \right)  , \label{eq:approxW0}
	\end{equation}
	and consequently
	\begin{equation} \label{eq:Blambda}
		\sup_{\theta \in \Sigma_K} | W_{0}(\mu_{\theta}, \nu) - W_{\lambda}(\mu_{\theta}, \nu) | \leqslant B(\lambda) \quad \text{where} \quad B(\lambda) = 2 d \lambda \log \left( \frac{8 \exp(2) R^2}{\sqrt{d} \lambda} \right)
	\end{equation}
	Notice that $B(\lambda)$ goes to zero when $\lambda \to 0$ at the speed
	$$
	B(\lambda) \sim_{\lambda \to 0} 2 d \lambda \log \left( 1/\lambda \right).
	$$
\end{prop}

\paragraph{Algorithm error}

For $\mu = \sum_{i=1}^n a_i \delta_{x_i}$, and $\nu = \sum_{j=1}^m b_j \delta_{y_j}$ two discrete distributions, we denote by
\begin{equation}\label{eq:sinkhorn_output}
	W_\lambda^{(\ell)}(\mu, \nu) = \sum_{i=1}^n a_i\varphi_i^{(\ell)} + \sum_{j=1}^m b_j \psi_j^{(\ell)},
\end{equation}
the approximation of the regularized OT cost $W_\lambda(\mu,\nu)$ that is returned by the Sinkhorn approximation after $\ell$ iterations. The variables $\varphi^{(\ell)}$ and $\psi^{(\ell)}$ denote the dual variables after $\ell$ iterations of the Sinkhorn algorithm. We thus consider the estimator used in our numerical experiments that is defined as

\begin{equation}\label{eq:sinkhorn_estimate}
	\hat{\theta}^{(\ell)}_\lambda = \argmin_{\theta \in \Sigma_K}W_\lambda^{(\ell)}(\hat{\mu}_\theta, \hat{\nu}).
\end{equation}

The computational complexity of Sinkhorn algorithm has been studied in \cite{chizat2020sinkhorn} and we remind the error after $\ell$ iterations of the Sinkhorn algorithm with respect to the regularized OT cost.

\begin{prop}\cite[Proposition 2]{chizat2020sinkhorn}\label{prop:sinkhorn_algorithm_error}. Assume that $\lambda > 0$. For $\mu = \sum_{i=1}^n a_i\delta_{x_i}$ and $\nu = \sum_{j=1}^m b_j \delta_{y_j}$ two discrete distributions and a ground cost set to $c(x,y) = \|x-y\|^2$ on $\mathbb{R}^d$. The approximation of the regularized OT cost after $\ell$ iterations of the Sinkhorn algorithm satisfies:
	\begin{equation}\label{eq:sinkhorn_algorithm_error}
		|W_\lambda^{(\ell)}(\mu, \nu) - W_\lambda(\mu, \nu)| \leq \frac{\|c\|^2_\infty}{\lambda \ell}
	\end{equation}
	where $\|c\|_{\infty} = \max_{(i,j)}\|x_i - y_j\|^2$.
\end{prop}

\begin{rem}\label{bound_eulcidean_groundcost}
	If the discrete distributions $\mu$ and $\nu$ both have their supports subsets of $B(0,R)$, it implies that $\max_{(i,j)}\|x_i - y_j\|^2\leq 4R^2$. And then, the quantity $\|c\|_{\infty}^2$ can be upper bounded by $\|c\|_{\infty}^2 \leq 16R^4$.
\end{rem}

\paragraph{Proof of the main result}

We now conclude the proof of our main result; that is Theorem 3.1 of the main paper. From Lemma \ref{lem:approx_estim_algo} we have the inequality
\begin{align*}
	W_0(\mu_{\hat{\theta}_\lambda^{(\ell)}}, \nu) -W_0(\mu_{\theta^*}, \nu) \leq & ~ 2\sup_{\theta \in \Sigma_K}|W_0(\mu_{\theta}, \nu) - W_\lambda(\mu_{\theta}, \nu)|\\
	& + 2 \sup_{\theta \in \Sigma_K}|W_\lambda(\mu_{\theta}, \nu) - W_\lambda(\hat{\mu}_{\theta}, \hat{\nu})| \\ 
	& + 2\sup_{\theta \in \Sigma_K}|W_\lambda(\hat{\mu}_{\theta}, \hat{\nu}) - W_\lambda^{(\ell)}(\hat{\mu}_{\theta}, \hat{\nu})|.
\end{align*}
Taking the expectation on both sides of this last inequality yields

\begin{align*}
	\mathbb{E}\left[W_0(\mu_{\hat{\theta}_\lambda^{(\ell)}}, \nu) -W_0(\mu_{\theta^*}, \nu)\right] \leq & ~ 2\sup_{\theta \in \Sigma_K}|W_0(\mu_{\theta}, \nu) - W_\lambda(\mu_{\theta}, \nu)|\\
	& + 2 \mathbb{E}\left[ \sup_{\theta \in \Sigma_K}|W_\lambda(\mu_{\theta}, \nu) - W_\lambda(\hat{\mu}_{\theta}, \hat{\nu})| \right] \\ 
	& + 2\sup_{\theta \in \Sigma_K}|W_\lambda(\hat{\mu}_{\theta}, \hat{\nu}) - W_\lambda^{(\ell)}(\hat{\mu}_{\theta}, \hat{\nu})|.
\end{align*}

Exploiting Proposition \ref{prop:control_Wlambda}, we have that $$\sup_{\theta \in \Sigma_K}|W_0(\mu_{\theta}, \nu) - W_\lambda(\mu_{\theta}, \nu)| \leq 2 d \lambda \log \left( \frac{8 \exp(2) R^2}{\sqrt{d} \lambda} \right).$$
Under the assumption that $n$ samples are available from each probability measure, Proposition \ref{prop:estimation_regfree} ensures the control 
$$
\mathbb{E}\left[ \sup_{\theta \in \Sigma_K}|W_\lambda(\mu_{\theta}, \nu) - W_\lambda(\hat{\mu}_{\theta}, \hat{\nu})| \right] \lesssim K\mathcal{E}(n,d).
$$
Finally, the algorithm error is upper bounded thanks to Proposition \ref{prop:sinkhorn_algorithm_error} as follows:
\begin{equation}
	\sup_{\theta \in \Sigma_K}|W_\lambda(\hat{\mu}_{\theta}, \hat{\nu}) - W_\lambda^{(\ell)}(\hat{\mu}_{\theta}, \hat{\nu})| \leq \frac{16 R^4}{\lambda \ell}.
\end{equation}
Gathering the pieces together, we derive \\
$$
\mathbb{E}\left[W_0(\mu_{\hat{\theta}_\lambda^{(\ell)}}, \nu) -W_0(\mu_{\theta^*}, \nu)\right]  \lesssim \lambda \log\left(\frac{1}{\lambda}\right) + \mathcal{E}(n,d) + \frac{1}{\lambda \ell},
$$
as claimed in Theorem \ref{thm:bound_decomposed_error} of the main paper.

\section{Alternative bound for the estimation error of $W_\lambda$}\label{subsec:empirical_process_reg_wass}

This section aims at controlling the estimation error
\begin{equation}\label{eq:empirical_process_theta}
	\sup_{\theta \in \Sigma_K} \Big| W_{\lambda}(\mu_{\theta},\nu) - W_{\lambda}(\mu_{\theta},\hat{\nu}) \Big| \ 
\end{equation}
with an upper bound of of the magnitude $M_\lambda n^{-1/2}$, where $M_\lambda$ is a constant that depends on $\lambda$. The arguments that we use are very much inspired by the works  \cite{genevay2019sample,chizat2020sinkhorn}. We will also see how the constants $M_\lambda$ involved in this upper bound depend on the regularizing parameter~$\lambda$ with a power that depends on the dimension $d$. We now introduce the space of functions that we exploit in our analysis. For a bounded subset $\mathcal{Z}$ of $\mathbb{R}^d$, we shall denote by $\mathscr{C}^\mathscr{K}(\mathcal{Z})$ the set of $\mathscr{C}^\mathscr{K}$ functions on~$\mathcal{Z}$ equipped with the norm
$  \|f\|_{\mathscr{K}} = \max_{|\kappa|\leqslant \mathscr{K}} \|\partial^\kappa f\|_{\infty} $
and 
\begin{equation}\label{eq:set_highly_regular_maps}
	\mathscr{C}_M^\mathscr{K}(\mathcal{Z}) = \{ f \in \mathscr{C}^ \mathscr{K}(\mathcal{Z}) \ | \ \|f\|_\mathscr{K} \leqslant M \},
\end{equation}
where the notation $ |\kappa|\leqslant \mathscr{K}$ denotes any multi-index $\kappa$ of differentiation of length $|\kappa|$ at most $ \mathscr{K}$.

\begin{prop}\label{prop:bound_empiricalprocess_Wlambda_targetsampling}
	Suppose that every probability measure considered has compact support included in $B(0,R)$. Then, we can bound the estimation error  \eqref{eq:empirical_process_theta} as follows:
	\begin{equation}\label{eq:bound_empiricalprocess_Wlambda_targetsampling}
		\mathbb{E}\left[  \sup_{\theta \in \Sigma_K} \Big| W_{\lambda}(\mu_{\theta},\nu) - W_{\lambda}(\mu_{\theta},\hat{\nu}) \Big| \right] \lesssim \frac{M_{\lambda}}{\sqrt{n}},
	\end{equation}
	with
	\begin{equation}\label{eq:new_constant_Mlambda}
		M_\lambda = M_{d} \max\left(R^2,\frac{R^{\lfloor d/2 \rfloor + 1}}{\lambda^{\lfloor d/2 \rfloor}} \right), 
	\end{equation}
	and $M_{d}$ is a constant that depends only on $d$. 
\end{prop}

This proof is built upon the following lemmas. A first step is to study the regularity of the optimal potentials of the dual formulation (2.2) of the paper. To this end, we adapt the analysis of~\cite{genevay2019sample} to the setting where the measure $\mu$ belongs to the parametric model $\{\mu_\theta~|~\theta \in \Sigma_K \}$. This result is adapted from \cite{genevay2019sample} and relies on the fact that an optimal potential can be chosen as the $c$-transform of $\varphi \in L^{\infty}(\mathcal{X})$. The choice of the squared Euclidean cost allows for a very clear description of the regularity of an optimal potential. We recall below the definition of the $c$-transform of  $\varphi \in L^{\infty}(\mathcal{X})$ that we denote by $\varphi_\mu^{c,\lambda}$ as $\lambda >0$. With the aim of manipulating functions defined on a convex and bounded subset of $\mathbb{R}^d$, the $c$-transforms are defined on $B(0,R)$. Hence, even if integrated only against $\mu$ or $\nu$ that have support $\mathcal{X}$ and $\mathcal{Y}$ respectively, the $c$-transform are defined on $B(0,R)$.
The expression of $\varphi_\mu^{c,\lambda} \in L^{\infty}(B(0,R))$ is given by

\begin{equation}
	\label{eq:regularized_ctransform_reminder}
	\forall y \in B(0,R),~\varphi_{\mu}^{c,\lambda}(y) = - \lambda \log \ds\int e^{- \frac{\|x-y\|^2 - \varphi(x)}{\lambda} } d\mu(x).
\end{equation} 

\begin{lem}\label{lem:bound_ctransform_regularized_euclidean}
	Suppose that every probability measure has compact support included in $B(0,R)$. Then, for all $\theta \in \Sigma_K$, there exists a couple of dual potentials $(\varphi, \psi)$ with respect to $W_\lambda(\mu_\theta, \nu)$, that satisfies $\psi(0)=0$, and $\psi= \varphi_{\mu_\theta}^{c,\lambda}$. Moreover, $\psi$ belongs to $\mathscr{C}^{\infty}(B(0,R))$, and for each $\mathscr{K}>0,$ there exists a constant $M_{\mathscr{K}} >0$ that depends only on $\mathscr{K}$ such that 
	\begin{equation}\label{eq:bound_regularized_ctransform_euclidean}
		\quad \|\psi\|_{\mathscr{K}} \leqslant M_{\mathscr{K}}\max\left(R^2,\frac{R^\mathscr{K}}{\lambda^{\mathscr{K}-1}}\right).
	\end{equation}
	The sup norm is taken over $B(0,R)$.
\end{lem}

This lemma will be proved in Section~\ref{proof:bound_ctransform_regularized_euclidean}.
Two observations on this Lemma \ref{lem:bound_ctransform_regularized_euclidean} will reveal useful in the sequel.

\begin{rem}
	We stress that $\psi = \varphi_{\mu_\theta}^{c,\lambda}$ is a regularized $c$-transform with respect to $\mu_\theta$ (with $\theta \in \Sigma_K$), and that the constant $M_{\mathscr{K}}>0$ that appears in Lemma \ref{lem:bound_ctransform_regularized_euclidean} does not depend on $\theta$.
\end{rem}

\begin{rem}\label{rem:link_dual_semidual}
	If we denote by $\varphi$ the dual potential of Lemma \ref{lem:bound_ctransform_regularized_euclidean} such that $\psi= \varphi_{\mu_\theta}^{c,\lambda}$ where $\psi$ meets the requirements of the Lemma \ref{lem:bound_ctransform_regularized_euclidean}, this variable $\varphi$ can be chosen as an optimal potential for the semi-dual formulation given in equation (2.5) of the paper.
\end{rem}

In the sequel of this section, we make a repeating use of the constant of Lemma \ref{lem:bound_ctransform_regularized_euclidean}. Thus, we introduce the notation $M_{\lambda, \mathscr{K}}$ to refer to the upper bound of equation \eqref{eq:bound_regularized_ctransform_euclidean} which is defined as 
\begin{equation}\label{eq:def_notation_bound_ctransform}
	M_{\lambda, \mathscr{K}}:= M_{\mathscr{K}}\max\left(R^2,\frac{R^\mathscr{K}}{\lambda^{\mathscr{K}-1}}\right),
\end{equation}
and 
\begin{equation}\label{eq:def_notation_constant}
	M_{\mathscr{K}} > 0 ~\text{is a constant that depends only on $\mathscr{K}$}.
\end{equation}

The next proposition links the estimation error $\sup_{\theta \in \Sigma_K} | W_{\lambda}(\mu_\theta,\nu) - W_{\lambda}(\mu,\hat{\nu}) |$ to the regularity of the $c$-transforms established in Lemma \ref{lem:bound_ctransform_regularized_euclidean}. 

\begin{prop}\label{prop:bound_estimation_empirical_proc_Wlambda}
	Suppose that all probability measures have compact supports included in $B(0,R)$. Then, for $\mathscr{K} >0$, we have the following upper bound
	\begin{equation}\label{eq:from_wass_to_empirical}
		\sup_{\theta \in \Sigma_K} | W_{\lambda}(\mu_\theta,\nu) - W_{\lambda}(\mu_\theta,\hat{\nu}) |
		\leqslant \sup_{\psi \in \mathscr{C}_{M_{\lambda, \mathscr{K}}}^\mathscr{K}(B(0,R))} \int \psi d(\nu - \hat{\nu}),
	\end{equation}
	with $M_{\lambda, \mathscr{K}} > 0$ a constant defined in equation \eqref{eq:def_notation_bound_ctransform}.
\end{prop}

\begin{proof}
	Let $\theta \in \Sigma_K$, $\mathscr{K} > 0$ and introduce $\varphi, \psi$ (resp. $\hat{\varphi}, \hat{\psi}$) two optimal potentials for the dual formulation of $W_{\lambda}(\mu_\theta, \nu)$ (resp. $W_{\lambda}(\mu_\theta, \hat{\nu})$) chosen as in Lemma \ref{lem:bound_ctransform_regularized_euclidean}. In particular $\psi, \hat{\psi} \in \mathscr{C}^{\mathscr{K}}_{M_{\lambda, \mathscr{K}}}(B(0,R))$, and the potentials $\varphi$ and $\hat{\varphi}$ are respectively optimal potentials for the semi-dual formulation of $W_\lambda(\mu_\theta, \nu)$ and $W_\lambda(\mu_\theta, \hat{\nu})$ as precised in Remark \ref{rem:link_dual_semidual}. We can thus write
	\begin{align*}
		W_{\lambda}(\mu_\theta,\nu) - W_{\lambda}(\mu_\theta,\hat{\nu})
		&= \int \varphi d \mu_\theta + \int \psi d\nu - \int \hat{\varphi} d\mu_\theta - \int \hat{\psi} d\hat{\nu}  \\
		&= \int \psi d\nu -  \int \psi d\hat{\nu}  \\ 
		&+ \underbrace{\left( \int \varphi d \mu_\theta + \int \psi d\hat{\nu} - \int \hat{\varphi} d\mu_\theta - \int \hat{\psi} d\hat{\nu} \right)}_{\leq 0}.
	\end{align*}
	By optimality of $\hat{\varphi}$ for the semi dual formulation of  $W_{\lambda}(\mu_\theta, \hat{\nu})$, the last term in the above parenthesis is non-positive.
	Using a symmetric optimality argument for $W_{\lambda}(\mu_\theta, \nu)$ and its optimal potential $\varphi$ for the semi-dual formulation, we get
	\begin{equation} \label{eq:bilateral_control}
		\int \hat{\psi} d(\nu - \hat{\nu})
		\leqslant W_{\lambda}(\mu_\theta,\nu) - W_{\lambda}(\mu_\theta,\hat{\nu})
		\leqslant \int \psi d(\nu - \hat{\nu}).
	\end{equation}
	As $\psi$ and $\hat{\psi}$ belong to $\mathscr{C}_{M_{\lambda, \mathscr{K}}}^\mathscr{K}(B(0,R))$, we can write 
	\begin{align} 
		|  W_{\lambda}(\mu_\theta,\nu) - W_{\lambda}(\mu_\theta,\hat{\nu}) |
		&\leqslant \sup_{\psi  \in \mathscr{C}_{M_{\lambda, \mathscr{K}}}^\mathscr{K}(B(0,R))} \left| \int \psi d(\nu - \hat{\nu}) \right| \nonumber \\ 
		& = \sup_{\psi  \in \mathscr{C}_{M_{\lambda, \mathscr{K}}}^\mathscr{K}(B(0,R))} \int \psi d(\nu - \hat{\nu}).
	\end{align}
	The set $\mathscr{C}_{M_{\lambda, \mathscr{K}}}^\mathscr{K}(B(0,R))$ being independent of $\theta$, we get
	\begin{equation}
		\sup_{\theta \in \Sigma_K}| W_{\lambda}(\mu_\theta,\nu) - W_{\lambda}(\mu_\theta,\hat{\nu}) |
		\leqslant \sup_{\psi  \in \mathscr{C}_{M_{\lambda, \mathscr{K}}}^\mathscr{K}(B(0,R))}  \int \psi d(\nu - \hat{\nu}).
	\end{equation}
\end{proof}

Therefore, the search for a control over $\sup_{\theta \in \Sigma_K} | W_{\lambda}(\mu_\theta,\nu) - W_{\lambda}(\mu,\hat{\nu}) |$ leads us to the study of the following empirical process
\begin{equation}\label{eq:empirical_process_definition}
	\sup_{\psi \in \mathscr{C}_{M_{\lambda, \mathscr{K}}}^\mathscr{K}(B(0,R))} Y_{\psi} \quad \text{with} \quad  Y_{\psi} = \int \psi d(\hat{\nu}-\nu) = \frac{1}{n} \sum_{i=1}^n \psi(Y_i) -   \int \psi d\nu.
\end{equation}

In order to bound the empirical process \eqref{eq:empirical_process_definition}, we will need several ingredients. First, we show that this empirical process has a sub-Gaussian behavior (see the definition in \cite{vanhandel2016proba}).

\begin{lem}\label{lem:subgaussian_behavior}
	Under the assumption that all probability measures have compact supports included in $B(0,R)$, if $Y = \{Y_1,\ldots,Y_n\}$ where $Y_1,\ldots,Y_n$ are independent random samples from $\nu$, the empirical process $(Y_{\psi})_{\psi \in \mathscr{C}^\mathscr{K}_{M_{\lambda, \mathscr{K}}}(B(0,R))}$ defined in equation \eqref{eq:empirical_process_definition} has zero mean and is subgaussian w.r.t. $2n^{-\frac{1}{2}} \|\cdot\|_{\infty}$. In other terms, for all $\varphi, \psi \in \mathscr{C}_{M_{\lambda, \mathscr{K}}}^\mathscr{K}(B(0,R))$ we have
	\begin{equation}
		\forall s \in \mathbb{R}, \quad
		\mathbb{E}[ e^{s(Y_{\varphi} - Y_{\psi})}] \leqslant e^{\frac{2s^2}{n}\|\varphi-\psi\|_{\infty}^2} = e^{\frac{s^2}{2}\\(2n^{-\frac{1}{2}}\|\varphi-\psi\|_{\infty})^2}.
	\end{equation}
\end{lem}
 
	This lemma in an application of Azuma-Hoeffding inequality \cite{vanhandel2016proba}[Corollary 3.9]. We can now use Dudley's entropy integral inequality, which we recall below.

\begin{thm}\cite{vanhandel2016proba}[Corollary 5.25]
	Let $(Y_{\varphi})_{\varphi \in \Phi}$ be a zero mean stochastic process which is sub-Gaussian with respect to the distance induced by a norm $\|\cdot\|$ on the indexing set $\Phi$.
	Then
	$$\mathbb{E}\left[\sup_{\varphi \in \Phi} Y_\varphi \right] \leq 12 \int_0^\infty\sqrt{\log(N(\varepsilon, \Phi, \|\cdot\|))}d\varepsilon,$$
	where $N(\varepsilon, \Phi, \|\cdot\|)$ is the covering number of $\Phi$ by balls of radius $\varepsilon$ with respect to the norm $\|\cdot\|$.
\end{thm}\label{thm:dudley_entropy}

A classical bound  on the covering number for smooth functions  (see e.g.\  \cite[Theorem 2.7.1]{vandervaart1996empirical}) will prove highly valuable.

\begin{thm}
	If $\mathcal{Z}$ is a bounded convex subset of $\mathbb{R}^d$ with nonempty interior, then there exists a constant $L(\mathscr{K},d)$ such that
	\begin{equation}
		\forall \varepsilon > 0, \quad \log N(\varepsilon, \mathscr{C}_1^\mathscr{K}(\mathcal{Z}) , \|\cdot\|_{\infty} )
		\leqslant L(\mathscr{K},d) |\mathcal{Z} + B(0,1)| \frac{1}{\varepsilon^{d/\mathscr{K}}} .
	\end{equation}
	where $N(\varepsilon, \mathscr{C}_1^\mathscr{K}(\mathcal{Z}) , \|\cdot\|_{\infty} )$ denotes the covering number of $ \mathscr{C}_1^\mathscr{K}(\mathcal{Z}) $ (by balls of radius $\varepsilon$) with respect to the $\ell_{\infty}$ norm, and where $|\mathcal{Z} + B(0,1)|$ is the Lebesgue measure of $\mathcal{Z} + B(0,1)$
\end{thm}

We now have all the ingredients to bound the expectation of the empirical process \eqref{eq:empirical_process_definition}.

\begin{prop}\label{prop:bound_empiricalprocess_regular_functions}
	Suppose that all probability measures have compact supports included in $B(0,R)$. Then, we have the following upper bound
	\begin{equation}\label{eq:bound_empiricalprocess_regular_functions}
		\mathbb{E}\left[ \sup_{\psi \in \mathscr{C}_{M_\lambda}^{d'}(B(0,R))} \int \psi d(\hat{\nu}-\nu) \right] \lesssim \frac{M_{\lambda}}{\sqrt{n}} \quad \text{with} \quad   M_\lambda = M_{d} \max\left(R^2,\frac{R^{\lfloor d/2 \rfloor + 1}}{\lambda^{\lfloor d/2 \rfloor}} \right),
	\end{equation}
	and $M_{d} > 0$ is a constant that depends only on $d$. 
\end{prop}

\begin{proof}
	Set $\mathscr{K}>0$. We denote the empirical process under study
	by $$ \left(Y_{\psi}\right)_{\psi \in \mathscr{C}_{M_{\lambda, \mathscr{K}}}^\mathscr{K}(B(0,R))} \quad \text{with} \quad  Y_{\psi} = \int \psi d(\hat{\nu}-\nu) = \frac{1}{n} \sum_{i=1}^n \psi(Y_i) -   \int \psi d\nu,$$
	and $M_{\lambda, \mathscr{K}} > 0$ the constant defined in equation \eqref{eq:def_notation_bound_ctransform}.
	Dudley's inequality with the entropy integral (see Theorem \ref{thm:dudley_entropy}) gives
	\begin{align}
		\mathbb{E}\left[ \sup_{\psi \in \mathscr{C}_{M_{\lambda, \mathscr{K}}}^\mathscr{K}(B(0,R))} Y_{\psi} \right] &
		\leqslant 12 \int_0^{\infty} \sqrt{ \log N(\varepsilon, \mathscr{C}_{M_{\lambda, \mathscr{K}}}^\mathscr{K}(B(0,R)), 2n^{-\frac{1}{2}} \|\cdot\|_{\infty})} d \varepsilon \\
		& \leqslant 12 \int_0^{\infty} \sqrt{ \log N\left( \frac{1}{2}\sqrt{n} M_{\lambda, \mathscr{K}}^{-1} \varepsilon, \mathscr{C}_{1}^\mathscr{K}(B(0,R)), \|\cdot\|_{\infty}\right)} d \varepsilon\\
		& \leqslant \frac{24 M_{\lambda, \mathscr{K}}}{\sqrt{n}} \int_0^{\infty} \sqrt{ \log N(\varepsilon, \mathscr{C}_1^\mathscr{K}(B(0,R)), \|\cdot\|_{\infty})} d \varepsilon \\
		&\lesssim \frac{M_{\lambda, \mathscr{K}}}{\sqrt{n}} \int_0^{1/2} \varepsilon^{-\frac{d}{2 \mathscr{K}}} d \varepsilon .
	\end{align}
	This integral is finite as soon as $\mathscr{K} > d/2$. As $M_{\lambda, \mathscr{K}} =M_{\mathscr{K}}\max\left(R^2,\frac{R^\mathscr{K}}{\lambda^{\mathscr{K}-1}}\right)$ and $\lambda$ will be chosen little in the sequel, we set $\mathscr{K}= \lfloor d/2 \rfloor + 1$ in order to have the quantity $M_{\lambda, \mathscr{K}}$ as small as possible. From now on, we denote by $d':= \lfloor d/2 \rfloor + 1$, and with this choice of $\mathscr{K}=d'$, we can substitute the constant $M_{\lambda, \mathscr{K}}$ of Lemma \ref{lem:bound_ctransform_regularized_euclidean} with a new constant $M_\lambda:= M_{\lambda, d'}$ that reads
	$$
	M_\lambda = M_{d} \max\left(R^2,\frac{R^{\lfloor d/2 \rfloor + 1}}{\lambda^{\lfloor d/2 \rfloor}} \right),
	$$
	where $M_{d}$ is a constant that depends only on $d$. Finally, we have the following bound for the empirical process under study
	
	\begin{equation}\label{eq:final_bound_empirical}
		\mathbb{E}\left[ \sup_{\psi \in \mathscr{C}_{M_\lambda}^{d'}(B(0,R))} Y_{\psi}\right] \lesssim \frac{M_{\lambda}}{\sqrt{n}}.
	\end{equation}
	
\end{proof}

\paragraph{Proof of Proposition \ref{prop:bound_empiricalprocess_Wlambda_targetsampling}}\label{proof:bound_empiricalprocess_Wlambda_targetsampling}
Gathering the results established since Lemma \ref{lem:bound_ctransform_regularized_euclidean}, we are in a favorable position to prove Proposition \ref{prop:bound_empiricalprocess_Wlambda_targetsampling}.
\begin{proof}
	
	Indeed, by combining inequality \eqref{eq:bound_empiricalprocess_regular_functions} from Proposition \ref{prop:bound_empiricalprocess_regular_functions} with upper bound \eqref{eq:from_wass_to_empirical}, we derive
	\begin{equation} \label{eq:rateexpectation}
		\mathbb{E}\left[  \sup_{\theta \in \Sigma_K} \Big| W_{\lambda}(\mu_{\theta},\nu) - W_{\lambda}(\mu_{\theta},\hat{\nu}) \Big| \right] \lesssim \frac{M_{\lambda}}{\sqrt{n}} .
	\end{equation}
	And the above  inequality is the result claimed in Proposition \ref{prop:bound_empiricalprocess_Wlambda_targetsampling}. 
\end{proof}

The second empirical process in \eqref{eq:control_emp_processes} can be upper bounded in a similar manner, as shown in the next proposition.

\begin{prop}\label{prop:bound_second_empirical_process} 
	Suppose that all probability measures have compact support included in $B(0,R)$. Make the additional assumption that for all $k \in \{1,\ldots,K\},~ m_k$ samples from $\mu_k$ are available, and denote by $\underbar{m} = \min(m_1, \ldots, m_K)$. Then, the following inequality holds
	\begin{equation}\label{eq:bound_second_empirical_process}
		\mathbb{E}\left[\sup_{\theta \in \Sigma_K}|W_\lambda(\hat{\mu}_{\theta}, \hat{\nu}) - W_\lambda(\mu_{\theta}, \hat{\nu})|\right] \lesssim \frac{K M_\lambda}{\sqrt{\underbar{m}}},
	\end{equation}
	where $M_\lambda$ is defined in equation \eqref{eq:new_constant_Mlambda}.
\end{prop}

\begin{proof}
	We begin by setting $\theta \in \Sigma_K$. Let us denote by $\varphi$ (resp.$\hat{\varphi}$) an optimal potential chosen as in Lemma \ref{lem:bound_ctransform_regularized_euclidean} when considering the semi dual formulation of the regularized optimal transport problem between $\mu_\theta$ and $\hat{\nu}$ (resp. $\hat{\mu}_\theta$ and $\hat{\nu}$). Thus,
	$$
	\begin{aligned}
		W_\lambda(\hat{\mu}_\theta, \hat{\nu}) - W_\lambda(\mu_\theta, \hat{\nu}) & = \int \hat{\varphi}_{\hat{\nu}}^{c,\lambda} d\hat{\mu}_\theta + \int \hat{\varphi} d\hat{\nu} - \left( \int \varphi_{\hat{\nu}}^{c,\lambda} d\mu_\theta + \int \varphi d\hat{\nu}\right) \\
		& =  \int \hat{\varphi}_{\hat{\nu}}^{c,\lambda} d\hat{\mu}_\theta -  \int \hat{\varphi}_{\hat{\nu}}^{c,\lambda} d\mu_\theta \\
		& + \underbrace{\left( \int \hat{\varphi}_{\hat{\nu}}^{c,\lambda} d\mu_\theta + \int \hat{\varphi} d\hat{\nu} - \left( \int \varphi_{\hat{\nu}}^{c,\lambda} d\mu_\theta + \int \varphi d\hat{\nu}\right)\right)}_{\leq 0}
	\end{aligned}
	$$
	
	The optimality of the variable  $\varphi$  with respect to the measures $\mu_\theta$ and $\hat{\nu}$ ensures the last term of the previous equation to be non positive. Hence,
	
	$$W_\lambda(\hat{\mu}_\theta, \hat{\nu}) - W_\lambda(\mu_\theta, \hat{\nu}) \leq \int \hat{\varphi}_{\hat{\nu}}^{c,\lambda} d(\hat{\mu}_\theta -\mu_\theta).$$
	
	With a slight modification of the last argument we get
	$$W_\lambda(\mu_\theta, \hat{\nu}) - W_\lambda(\hat{\mu}_\theta, \hat{\nu}) \leq \int \varphi_{\hat{\nu}}^{c,\lambda} d(\mu_\theta -\hat{\mu}_\theta).$$
	
	Combining these last inequalities, we have
	\begin{align*}
		\left|W_\lambda(\hat{\mu}_\theta, \hat{\nu}) - W_\lambda(\mu_\theta, \hat{\nu})\right|& \leq  \sup_{\psi \in L^{\infty}(\mathcal{Y})}\left|\sum_{k=1}^K \theta_k \int \psi_{\hat{\nu}}^{c,\lambda} d(\mu_k - \hat{\mu}_k)\right| \\ & \leq  \sum_{k=1}^K \theta_k \sup_{\psi \in L^{\infty}(\mathcal{Y})} \left|\int \psi_{\hat{\nu}}^{c,\lambda} d(\mu_k - \hat{\mu}_k)\right|\\
		& \leq  \sum_{k=1}^K \sup_{\psi \in L^{\infty}(\mathcal{Y})} \left|\int \psi_{\hat{\nu}}^{c,\lambda} d(\mu_k - \hat{\mu}_k)\right|.
	\end{align*}	
This last upper bound being independent of $\theta$, we derive
\begin{equation}\label{eq:inproof_altern_bound_K}
	\sup_{\theta \in \Sigma_K} \left|W_\lambda(\hat{\mu}_\theta, \hat{\nu}) - W_\lambda(\mu_\theta, \hat{\nu})\right| \leq   \sum_{k=1}^K \sup_{\psi \in L^{\infty}(\mathcal{Y})} \left|\int \psi_{\hat{\nu}}^{c,\lambda} d(\mu_k - \hat{\mu}_k)\right|.
\end{equation}

Next, the application of Lemma \ref{lem:bound_ctransform_regularized_euclidean} when computing a $c$-transform w.r.t. $\hat{\nu}$ gives that $\psi_{\hat{\nu}}^{c, \lambda} \in \mathscr{C}_{M_\lambda}^{d'}(B(0,R))$ with $d'$ and $M_\lambda$ both defined in equation \eqref{eq:new_constant_Mlambda}. Hence,
for $k \in \{1,...,K\}$, using the same ingredients as in Proposition \ref{prop:bound_estimation_empirical_proc_Wlambda} , we reach the study of an empirical process indexed by the class of functions $\mathscr{C}_{M_\lambda}^{d'}(B(0,R))$. Finally, the straight application of Proposition \ref{prop:bound_empiricalprocess_regular_functions} yields
$$\mathbb{E}\left[\sup_{f \in \mathscr{C}_{M_\lambda}^{d'}((B(0,R))}\left|\int f d(\mu_k - \hat{\mu}_k)\right|\right]  \lesssim \frac{M_\lambda}{\sqrt{m_k}},$$
where $m_k$ is the number of observations sampled from distribution $\mu_k$. Then, exploiting inequality \eqref{eq:inproof_altern_bound_K} we derive
\begin{align*}
	\mathbb{E}\left[\sup_{\theta \in \Sigma_K}|W_\lambda(\hat{\mu}_{\theta}, \hat{\nu}) - W_\lambda(\mu_{\theta}, \hat{\nu})|\right]& \lesssim \sum_{k=1}^K \frac{M_\lambda}{\sqrt{m_k}} \\
	& \leq \frac{K M_\lambda}{\sqrt{\underbar{m}}},
\end{align*}
where $\underbar{m} = \min{ \{m_k : 1 \leq k \leq K\}}$. 
\end{proof}

We now gather the results from the previous section to obtain an upper bound on the expected excess risk of our regularized Wasserstein estimators.

\begin{prop}\label{prop:rate_wlambda_full_sampling}
	Suppose that all probability measures have compact supports included in $B(0,R)$. If, for all $k \in \{1,\ldots,K\}$, $m_k$ samples from $\mu_k$ are available and $n$ samples from $\nu$ are available, denoting $\underbar{m} = \min(m_1, \ldots, m_K)$, we have
	\begin{equation}
		\mathbb{E}\left[W_0(\mu_{\hat{\theta}_\lambda}, \nu) -W_0(\mu_{\theta^*}, \nu) \right] \lesssim \underbrace{\frac{2 K M_\lambda}{\sqrt{\underbar{m}}} + \frac{2 M_\lambda}{\sqrt{n}}}_{\mbox{Estimation error}} + \underbrace{4 d \lambda \log \left( \frac{8 \exp(2) R^2 }{\sqrt{d} \lambda} \right)}_{\mbox{Approximation error}},
		\label{eq:rate_Wlambda_full_samp}
	\end{equation}
	where $M_\lambda = M_{d} \max\left(R^2,\frac{R^{\lfloor d/2 \rfloor + 1}}{\lambda^{\lfloor d/2 \rfloor}} \right)$,
	and $M_{d}$ is a constant that depends only on $d$. 
\end{prop}

\begin{proof}
	Gathering the results on the approximation error for the regularized OT cost  in Proposition~\ref{prop:control_Wlambda} and the upper bounds from Proposition \ref{prop:bound_empiricalprocess_Wlambda_targetsampling} and Proposition \ref{prop:bound_second_empirical_process} on the empirical processes defined  in \eqref{eq:control_emp_processes}, we obtain the convergence rate claimed in \eqref{eq:rate_Wlambda_full_samp}.
\end{proof}

From the upper bound on the expected excess risk of $\hat{\theta}_\lambda$ established in the previous Proposition~\ref{prop:rate_wlambda_full_sampling}, we can propose a regularization choice in order to balance the estimation error and the approximation error. This regularization choice and the corresponding rate of convergence are given in the next Corollary.

\begin{cor}\label{cor:rate_Wlambda_regularization_policy}
	Suppose that all probability measures have compact supports included in $B(0,R)$. Make the additional assumption that for all the distributions $\mu_k$ and for $\nu$, at least $n$ samples are available. Then, choosing ${\lambda_n = n^{-1/(2 \lfloor d/2 \rfloor +2)}}$ we get
	\begin{equation}\label{eq:rate_Wlambda_full_samp_bis}
		\mathbb{E}\left[W_0(\mu_{\hat{\theta}_\lambda}, \nu) -W_0(\mu_{\theta^*}, \nu) \right] \lesssim  n^{-\frac{1}{2 \lfloor d/2 \rfloor +2}}\log(n),
	\end{equation}
	where $\lesssim$ hides a constant that depends on $R$ and $d$.
\end{cor}
\begin{proof}
	In order to drive the approximation term towards zero, the regularization parameter will converge towards $0$. And in this  case, $\lambda \log \left( \frac{8 \exp(2) R^2 }{\sqrt{d} \lambda} \right)\sim_{\lambda \rightarrow 0} \lambda \log(\lambda^{-1})$. Next as we have assumed that all the distributions have $n$ samples, the estimation term equals $\frac{M_\lambda}{\sqrt{n}} = \frac{M_d R^{\lfloor d/2 \rfloor + 1}}{\lambda^{\lfloor d/2 \rfloor}\sqrt{n}}$. To balance these two terms we set $\lambda_n = n^{\frac{-1}{2 \lfloor d/2 \rfloor +2}}$, and with this choice of regularization parameter, there exists a constant $M_{R,d} > 0$ such that for $n$ sufficiently large,
	\begin{equation}
		\mathbb{E}\left[W_0(\mu_{\hat{\theta}_\lambda}, \nu) -W_0(\mu_{\theta^*}, \nu) \right] \leq M_{R,d}n^{\frac{-1}{2 \lfloor d/2 \rfloor +2}}\log(n).
	\end{equation}
\end{proof}

\section{Proof of Lemma \ref{lem:bound_ctransform_regularized_euclidean}}
\label{proof:bound_ctransform_regularized_euclidean}

In this section we give a precise bound on the derivatives of a $c$-transform. Some results of the same flavor had already been established, for instance in \cite[Lemma 1, Lemma 2] {genevay2019sample}. The specificity of our result is to exploit the particular cost function $c(x,y)=\|x-y\|^2$ to give a precise description of the bound and to ensure that it is independent of the parameter $\theta$.\\

We precise the notations previously introduced in equation \eqref{eq:set_highly_regular_maps}. For a multi-index $\kappa = (\kappa_1,\ldots, \kappa_d) \in \mathbb{N}^d$, we denote by $|\kappa| = \sum_{i=1}^d \kappa_i$ and $D^{\kappa}$ the differential operator defined as follows
\begin{equation}\label{eq:diff_operator_notation}
	D^\kappa = \frac{\partial^{|\kappa|}}{\partial x_1^{\kappa_1}\ldots \partial x_d^{\kappa_d}}.
\end{equation}

	\begin{lem}\label{lem:bound_ctransform}
		Set $\lambda > 0$. We assume that $\mu$ and $\nu$ have compact supports included in $B(0,R)$. Denoting by $\psi:= \varphi_\mu^{c,\lambda}$ the $c$-transform of a given function $\varphi \in L^{\infty}(\mathcal{Y})$, for every multi-index $\kappa \in (\kappa_1, \ldots, \kappa_d)$, there exists a constant $C_{\mathcal{K}}$ that depends only on $\mathcal{K}$ such that
		$$ \sup_{y \in B(0,R)} | D^{\kappa} \varphi(y) | \leq C_{\mathcal{K}} \frac{R^{\mathcal{K}}}{\lambda^{\mathcal{K} -1 }}.$$
	\end{lem}
	
	Notice that for every $\theta \in \Sigma_K$, the probability measure $\mu_\theta = \sum_{k=1}^K \theta_k \mu_k$ has compact support included $B(0,R)$. Therefore, the upper bound established in the previous Lemma \ref{lem:bound_ctransform} holds for every measure $\mu_\theta$ with $\theta \in \Sigma_K$. 
	
	\begin{proof}
		We proceed by induction on $\mathcal{K} \geq 1$ with inductive hypothesis $H(\mathcal{K})$: there exists a constant $C_{\mathcal{K}}$ such that for every  multi-index $\kappa$ with $|\kappa| = \mathcal{K}$  $\sup_{y \in B(0,R)} | D^{\kappa} \varphi(y) | \leq C_{\mathcal{K}} \frac{R^{\mathcal{K}}}{\lambda^{\mathcal{K} -1 }}$.\\
		
		\textit{Base case.} As $\psi =  \varphi_\mu^{c,\lambda}$, 
		\begin{equation}\label{eq:relation_ctransform}
			\forall y \in B(0,R),~ \exp\left(-\frac{\psi(y)}{\lambda}\right) = \int_{\mathcal{X}} \exp\left(\frac{\varphi(x) - c(x,y)}{\lambda}\right)d\mu(x).
		\end{equation}
		As $\mu$ has compact support we can differentiate with respect to some $y_j$. From this we deduce that $\psi$ is differentiable with respect to $y_j$ and that for $y \in B(0,R)$,
		$$-\frac{1}{\lambda} \frac{ \partial \psi }{\partial y_j}(y) \exp\left(-\frac{\psi(y)}{\lambda}\right) = -\frac{1}{\lambda}  \int_{\mathcal{X}}  \frac{ \partial c }{\partial y_j}(x,y) \exp\left(\frac{\varphi(x) - c(x,y)}{\lambda} \right)  d \mu(x).$$ 
		Taking the absolute value on both sides of last equality gives 
		$$\left| \frac{ \partial \psi }{\partial y_j}(y) \right| \exp\left(-\frac{\psi(y)}{\lambda}\right)  \leq  \int_{\mathcal{X}} \left| \frac{ \partial c }{\partial y_j}(x,y) \right|\exp\left(\frac{\varphi(x) - c(x,y)}{\lambda} \right) d \mu(x). $$
		The cost $c$ being the squared euclidean distance, we have $ \left|\frac{ \partial c }{\partial y_j}(x,y)\right| \leq 4R.$ As $\psi$ is the $c$-transform of $\varphi$, $\int_{\mathcal{X}}\exp\left(\frac{\varphi(x) + \psi(y) - c(x,y)}{\lambda} \right) d \mu(x) = 1$. From this we deduce 
		$$ \left| \frac{ \partial \psi }{\partial y_j}(y) \right| \leq 4R,$$
		which concludes the base case.\\
		\textit{Induction case.} Set $\mathcal{K} \geq 2$, and suppose that $H(1), \ldots, H(\mathcal{K}-1)$ hold true. Set $\kappa \in \mathbb{N}^d$ a multi-index with $|\kappa|= \mathcal{K}$. The computation of $D^{\kappa}$ relies on Faà di Bruno's formula \cite{hardy2006} and the correspondence 
		\begin{equation}\label{eq:index_set}
			\kappa = (\kappa_1, \ldots, \kappa_d) \longleftrightarrow \{\underbrace{1,\ldots, 1}_{\kappa_1}, \underbrace{2,\ldots, 2}_{\kappa_2}, \ldots,\underbrace{d, \ldots, d}_{\kappa_d}\}.
		\end{equation}
		
		Throughout this induction step, we will always implicitly refer to this one to one relation \eqref{eq:index_set}. For instance, when mentioning a partition of $\kappa$, even if not explicitly stated, we refer to the right-hand side of relation \eqref{eq:index_set}. And when referring to the differential operator $D^B$ with $B$ a subset of $\{1,\ldots, 1, 2,\ldots, 2, \ldots,d, \ldots, d\}$ we identify $B$ with the corresponding multi index on the left-hand side of relation \eqref{eq:index_set}.\\
		
		Applying Faà di Bruno's formula on the right side of equation \eqref{eq:relation_ctransform}, as well as the differentiation theorem under the integral, gives
		\begin{align}\label{eq:faa_right}
			&D^{\kappa} \int_{\mathcal{X}} \exp\left(\frac{\varphi(x) - c(x,y)}{\lambda}\right)d\mu(x)\\
			& = \int_\mathcal{X}\left(\sum_{\pi \in \mathcal{P}(\kappa)}\left(\frac{-1}{\lambda}\right)^{|\pi|} \exp\left(\frac{\varphi(x)-c(x,y)}{\lambda} \right)\prod_{B \in \pi}D^B{c}(x,y)  \right)d \mu(x), \nonumber
		\end{align}
		
		where $\mathcal{P}(\kappa)$ denotes the collection of partitions of the right hand side of relation \eqref{eq:index_set}, and $|\pi|$ the number of sets that compose $\pi$.
		Applying the the same formula on the left hand side of equation \eqref{eq:relation_ctransform} yields
		\begin{align}\label{eq:faa_left}
			D^{\kappa}\exp\left(-\frac{\psi(y)}{\lambda}\right) & = \sum_{\pi \in \mathcal{P}(\kappa)}\left(-\frac{1}{\lambda}\right)^{|\pi|}\exp\left(-\frac{\psi(y)}{\lambda}\right)\prod_{B \in \pi}D^B\psi(y) \\
			& = \left(-\frac{1}{\lambda} D^\kappa \psi(y) + \sum_{\pi \in \widetilde{\mathcal{P}}(\kappa)}\left(-\frac{1}{\lambda}\right)^{|\pi|}\prod_{B \in \pi}D^B\psi(y) \right)\exp\left(-\frac{\psi(y)}{\lambda}\right), \nonumber
		\end{align}
		with $\widetilde{\mathcal{P}}(\kappa)$ the collection of partitions of $\kappa$, when relying on correspondence \eqref{eq:index_set}, \textit{without} the partition composed of the full set. Reminding that both quantities \eqref{eq:faa_left} and \eqref{eq:faa_right} are equals enables us to derive
		
		\begin{align*}
			D^{\kappa} \psi(y) & = \sum_{\pi \in \widetilde{\mathcal{P}}(\kappa)}\left(-\frac{1}{\lambda}\right)^{|\pi|-1}\prod_{B \in \pi}D^B\psi(y)\\
			& + \int_\mathcal{X}\left(\sum_{\pi \in \mathcal{P}(\kappa)}\left(-\frac{1}{\lambda}\right)^{|\pi|-1} \prod_{B \in \pi}D^B{c}(x,y)  \right)\gamma(x,y)d \mu(x), 
		\end{align*}
		where $\gamma(x,y)= \exp\left(\frac{\varphi(x) + \psi(y)-c(x,y)}{\lambda} \right)$. Then, taking the absolute value on both sides of last equality gives
		\begin{align}\label{eq:triangular_ineq}
			\left| D^{\kappa} \psi(y) \right|& \leq \sum_{\pi \in \widetilde{\mathcal{P}}(\kappa)}\lambda^{1-|\pi|}\prod_{B \in \pi}\left| D^B\psi(y) \right|\\
			& + \int_\mathcal{X}\left(\sum_{\pi \in \mathcal{P}(\kappa)}\lambda^{1-|\pi|} \prod_{B \in \pi} \left|D^B{c}(x,y) \right| \right)\gamma(x,y)d \mu(x). \nonumber
		\end{align}
		We begin by controlling the first sum of the right-hand side of last inequality. As $\pi$ is not the partition built from one block, for every block $B\in \pi$, $|B| \leq |\kappa|-1 = \mathcal{K}-1$. We can thus apply the inductive hypothesis to every factor of $\prod_{B \in \pi}\left| D^B\psi(y) \right|$ to derive 
		$$ \lambda^{1-|\pi|} \prod_{B \in \pi}\left| D^B\psi(y) \right| \leq  \lambda^{1-|\pi|} \prod_{B \in \pi} C_{|B|}\frac{R^{|B|}}{\lambda^{|B|-1}} .$$
		As $\pi$ is a partition of $\kappa$ composed of $|\pi|$ blocks, we have 
		$$
		\prod_{B \in \pi} \frac{R^{|B|}}{\lambda^{|B|-1}} = \frac{R^{|\kappa|}}{\lambda^{|\kappa|-|\pi|}}.
		$$
		From these previous computations we deduce that up to a multiplicative constant that depends only on $|\kappa|=\mathcal{K}$ the following inequality holds true
		$$
		\sum_{\pi \in \widetilde{\mathcal{P}}(\kappa)}\lambda^{1-|\pi|}\prod_{B \in \pi}\left| D^B\psi(y) \right| \lesssim \lambda^{1-|\pi|}\frac{R^{|\kappa|}}{\lambda^{|\kappa|-|\pi|}} = \frac{R^{\mathcal{K}}}{\lambda^{\mathcal{K}-1}}.
		$$
		We now control, the integral on the right-hand side of inequality \eqref{eq:triangular_ineq}.
		For $\pi$ a partition of $\kappa$ and $B$ a block of $\pi$, as $c$ is the squared euclidean distance, for all $x,y \in B(0,R),~ |D^{B}(x,y)| \leq 4R$. From this we deduce 
		$$ 
		\sum_{\pi \in \mathcal{P}(\kappa)}\lambda^{1-|\pi|} \prod_{B \in \pi} \left|D^B{c}(x,y) \right| \leq \sum_{\pi \in \mathcal{P}(\kappa)}\lambda^{1-|\pi|} (4R)^{|\pi|}.
		$$
		The biggest term in the sum over the partitions of $\kappa$ is when considering the partition composed of singletons. In this case, $|\pi|=|\kappa| = \mathcal{K}$. Hence, up to a multiplicative constant that depends only on $\mathcal{K}$,
		$$
		\int_\mathcal{X}\left(\sum_{\pi \in \mathcal{P}(\kappa)}\lambda^{1-|\pi|} \prod_{B \in \pi} \left|D^B{c}(x,y) \right| \right)\gamma(x,y)d \mu(x) \lesssim \frac{R^{\mathcal{K}}}{\lambda^{\mathcal{K}-1}} \int_\mathcal{X}\gamma(x,y) d\mu(x) = \frac{R^{\mathcal{K}}}{\lambda^{\mathcal{K}-1}}.
		$$
		To derive the last equality, remind that $\gamma(x,y)= \exp\left(\frac{\varphi(x) + \psi(y)-c(x,y)}{\lambda} \right)$; and as a consequence of equality \eqref{eq:relation_ctransform}, $\int_\mathcal{X}\gamma(x,y)d \mu(x)=1$. We now have upper bounded all the terms of the left hand side of inequality \eqref{eq:triangular_ineq}. We thus derive that for all $y \in B(0,R)$, 
		$$ \left| D^{\kappa} \psi(y) \right| \lesssim \frac{R^{\mathcal{K}}}{\lambda^{\mathcal{K}-1}}.$$
		As $\kappa$ is an arbitrary multi index with weights $|\kappa| = \mathcal{K}$, it shows $H(\mathcal{K})$.
	\end{proof}
	
	\paragraph{Conclusion of the proof of Lemma \ref{lem:bound_ctransform_regularized_euclidean}}
	\begin{proof}
		As  $\psi = \varphi_\mu^{c,\lambda}$, The application of Proposition 12 from \cite{feydy2019sinkhorndiv} ensures that a $c$-transform inherits the Lipschitz constant of the cost function. The cost function being $c(x,y) = ||x-y||^2$ with $x,y \in B(0,R)$, we have that $\psi$ is $4R$-Lipschitz. And as $\psi(0) = \rho(0) = 0$, we can write,
		$$\forall y \in B(0,R),~ \|\psi(y)\| \leq 4R||y||.$$
		Hence $\|\psi\|_{\infty} \leq 4R^2$. 
		Finally, Lemma \ref{lem:bound_ctransform} gives a control of the derivatives of $\psi$. Indeed, for every $\kappa \in \mathbb{N}^d$ such that $|\kappa| = \mathcal{K}$,
		$$ \forall y \in B(0,R),~ |D^{\kappa}\psi(y)| \lesssim \frac{R^{\mathcal{K}}}{\lambda^{1-\mathcal{K}}}.$$
		From this we deduce 
		$$
		\|\psi\|_{\mathscr{K}}\lesssim \max\left(R^2,\frac{R^\mathscr{K}}{\lambda^{\mathscr{K}-1}}\right),
		$$ 
		as claimed in Lemma \ref{lem:bound_ctransform_regularized_euclidean}.
	\end{proof}

\section{Extension of the results to the Sinkhorn divergence $S_\lambda$}\label{sec:sinkdiv}

We now study the collection of estimators $(\hat{\theta}_\lambda^{S(\ell)})_{\lambda > 0}$ defined by 
\begin{equation}\label{eq:estimate_Sinkdiv_algo}
	\hat{\theta}_\lambda^{S(\ell)} := \argmin_{\theta \in \Sigma_K} S_\lambda^{(\ell)}(\hat{\mu}_{\theta}, \hat{\nu}).
\end{equation}
When not taking into account the computational cost, we use the notation $\hat{\theta}_\lambda^{S}$ for the estimator defined by 
\begin{equation}\label{eq:estimate_Sinkdiv}
	\hat{\theta}_\lambda^{S} := \argmin_{\theta \in \Sigma_K} S_\lambda(\hat{\mu}_{\theta}, \hat{\nu}).
\end{equation}

In the next Section \ref{subsec:information_fisher_sinkhorn_approximation}, we introduce two additional assumptions in order to exploit the approximation result established in \cite{chizat2020sinkhorn} between the Sinkhorn divergence $S_\lambda(\mu, \nu)$ and the Wasserstein distance $W_0(\mu, \nu)$. Indeed, Theorem 1 in \cite{chizat2020sinkhorn}, that we remind in Theorem \ref{thm:control_Sinkhorn_Divergence} of this paper, allows to bound $|S_\lambda(\mu, \nu) - W_0(\mu, \nu)|$ with a constant that depends on $\lambda$, the standard Fisher information  $I(\mu)$, $I(\nu)$ of $\mu,\nu$, and $I(\mu,\nu)$ that is the Fisher information of the Wasserstein geodesic between $\mu$ and $\nu$ defined as in \cite{chizat2020sinkhorn}.

\subsection{Fisher information and approximation error of the Sinkhorn divergence $S_\lambda(\mu, \nu)$ }\label{subsec:information_fisher_sinkhorn_approximation}

We discuss  conditions that ensure a control of the approximation error between the Sinkhorn divergence $S_\lambda(\mu, \nu)$ and $W_0(\mu, \nu)$.

\begin{thm}\label{thm:control_Sinkhorn_Divergence} \cite[Theorem 1]{chizat2020sinkhorn}
	Suppose that $\mu$ and $\nu$ have bounded densities and supports. Then, it holds that
	\begin{equation}\label{eq:control_Sinkhorn_Divergence}
		|S_{\lambda}(\mu,\nu) - W_{0}(\mu,\nu)| \leq \frac{\lambda^2}{4}\max\{I(\mu, \nu), (I(\mu)+ I(\nu))/2\},
	\end{equation}
	where $I(\mu)$ refers to the standard Fisher information of $\mu$, and $I(\mu, \nu)$ is the Fisher information of the Wasserstein geodesic between $\mu$ and $\nu$ as defined  in \cite{chizat2020sinkhorn}.
\end{thm}
First, we introduce sufficient conditions to ensure that the Fisher information of $\mu_\theta$ can be upper bounded without dependence on $\theta$.

\begin{hyp}\label{hyp:fisher_component_mixture}
	The probability distributions $\mu_1, \ldots, \mu_K$ have finite Fisher information with respective densities $f_1, \ldots, f_K$ w.r.t. the Lebesgue measure, and all the components $\mu_1, \ldots, \mu_K$ have disjoint supports $\mathcal{X}_1, \ldots, \mathcal{X}_K$.
\end{hyp}

\begin{prop}\label{prop:bound_fisher_mixture}
	Suppose that assumption \ref{hyp:fisher_component_mixture} holds. Then, one has that
	\begin{equation}
		\forall \theta \in \Sigma_K,~ I(\mu_\theta) \leq \max_{k \in \{1, \ldots, K \}} I(\mu_k). \label{eq::bound_fisher_mixture}
	\end{equation}
\end{prop}

\begin{proof}
	For simplicity, we consider the case $d=1$. Let $\theta \in \Sigma_K$. Then, using the assumption that the components $\mu_k$  have disjoint supports, we decompose the Fisher information of $\mu_\theta$ as follows
	\begin{align*}
		I(\mu_\theta)& = \int_{\mathcal{X}} \left( \frac{f_{\theta}'(x)}{f_\theta (x)} \right)^2 f_{\theta}(x)dx  = \sum_{k=1}^K \theta_k \int_{\mathcal{X}_k} \left(\frac{f_{\theta}'(x)}{f_\theta(x)} \right)^2 f_k(x)dx \\
		& =  \sum_{k=1}^K \theta_k \int_{\mathcal{X}_k} \left(\frac{\theta_k f_{k}'(x)}{\theta_k f_k(x)} \right)^2 f_k(x)dx  = \sum_{k=1}^K \theta_k I(\mu_k) \leq \max_{k \in \{1, \ldots, K \}} I(\mu_k),
	\end{align*}
	which proves Inequality \eqref{eq::bound_fisher_mixture} for $d=1$. The case $d > 1$ can be treated analogously.
\end{proof}

Next, in order to bound the Fisher information of the Wasserstein geodesic between $\mu_\theta$ and $\nu$ with a constant independent of $\theta$, we adapt the assumptions of Proposition 1 from \cite{chizat2020sinkhorn} to our needs.

\begin{hyp}\label{hyp:bound_fisher_information_geodesic}
	The probability distribution $\nu$ is absolutely continuous with respect to the Lebesgue measure. Moreover, there exist two constants $m>0$ and $L>0$ such that for all $\theta \in \Sigma_K$ the Brenier potential $\varphi_\theta$ between $\mu_\theta$ and $\nu$ has a $L$-Lipschitz continuous Hessian satisfying $m \Id \leq \nabla^2 \varphi_\theta $.
\end{hyp}

\begin{prop}\label{prop:bound_fisher_information_geodesic}
	Suppose that Assumptions \ref{hyp:fisher_component_mixture} and \ref{hyp:bound_fisher_information_geodesic} hold. Then, we have the following inequality
	\begin{equation}\label{eq:bound_fisher_information_geodesic}
		\forall \theta \in \Sigma_K,~ I(\mu_\theta, \nu) \leq \frac{2}{m}\left(\max_{k \in \{1, \ldots, K \}} I(\mu_k) + \frac{L^2}{3m^2} \right).
	\end{equation}
\end{prop}

\begin{proof}
	A straight application of Proposition 1 from \cite{chizat2020sinkhorn} gives that
	$$\forall \theta \in \Sigma_K,~ I(\mu_\theta, \nu) \leq \frac{2}{m}\left(I(\mu_\theta) + \frac{L^2}{3m^2} \right).$$
	Then, using $I(\mu_\theta) \leq \max_{k \in \{1, \ldots, K \}} I(\mu_k)$ from Proposition \ref{prop:bound_fisher_mixture} yields inequality \eqref{eq:bound_fisher_information_geodesic}.
\end{proof}
As a consequence of  Theorem \ref{thm:control_Sinkhorn_Divergence}, we have the following result.

\begin{cor}\label{cor:uniform_control_approximation_S_lambda} Suppose that Assumptions \ref{hyp:fisher_component_mixture} and \ref{hyp:bound_fisher_information_geodesic} hold. Then, we have that there exists a constant $M_I > 0$ such that
	\begin{equation}\label{eq:approximation_bound_SinkDiv_cor}
		\forall \theta \in \Sigma_K,~|S_{\lambda}(\mu_\theta,\nu) - W_{0}(\mu_\theta,\nu)| \leq M_I \lambda^2,
	\end{equation}
	where 
	\begin{equation}\label{eq:constant_approximation_Sinkhorn_divergence}
		M_I =  \frac{1}{4}\max\left\{\frac{2}{m}\left(\max_{k \in \{1, \ldots, K \}} I(\mu_k) + \frac{L^2}{3m^2} \right), \frac{\max_{k \in \{1, \ldots, K \}} I(\mu_k)+ I(\nu)}{2}\right\}.
	\end{equation}
\end{cor}
\begin{proof}
	The combination of the upper bounds on $I(\mu_\theta)$ and $I(\mu_\theta, \nu)$ established in Proposition \ref{prop:bound_fisher_mixture} and in Proposition \ref{prop:bound_fisher_information_geodesic} respectively, as well as the upper bound of Theorem \ref{thm:control_Sinkhorn_Divergence} yields~\eqref{eq:approximation_bound_SinkDiv_cor}.
\end{proof}

\subsection{Estimation error of the Sinkhorn divergence $S_\lambda$ and rate of convergence}

\begin{lem}\label{lem:decomposition_sinkdiv_error}
	For $\lambda>0$, the excess risk of the estimator $\hat{\theta}_\lambda^{S(\ell)}$ is upper bounded as follows:
	\begin{align}\label{eq:decomposition_sinkdiv_error}
		0 \leq 	W_0(\mu_{\hat{\theta}_\lambda^{S(\ell)}}, \nu) -W_0(\mu_{\theta^*}, \nu) \leq & ~ 2\sup_{\theta \in \Sigma_K}|W_0(\mu_{\theta}, \nu) - S_\lambda(\mu_{\theta}, \nu)|\\
		& + 2 \sup_{\theta \in \Sigma_K}|S_\lambda(\mu_{\theta}, \nu) - S_\lambda(\hat{\mu}_{\theta}, \hat{\nu})| \nonumber \\ 
		& + 2\sup_{\theta \in \Sigma_K}|S_\lambda(\hat{\mu}_{\theta}, \hat{\nu}) - S_\lambda^{(\ell)}(\hat{\mu}_{\theta}, \hat{\nu})| \nonumber.
	\end{align}
\end{lem}

Mutatis mutandis, this is exactly the same proof as when considering the regularized transport cost $W_\lambda$.
\begin{proof}
	We begin with the decomposition
	\begin{align}\label{eq:sinkdiv_decomposition_fullsampling}
		W_0(\mu_{\hat{\theta}_\lambda^{S(\ell)}}, \nu) -W_0(\mu_{\theta^*}, \nu)~& = ~W_0(\mu_{\hat{\theta}_\lambda^{S(\ell)}}, \nu) -  S_\lambda(\mu_{\hat{\theta}_\lambda^{S(\ell)}}, \nu) +  S_\lambda(\mu_{\hat{\theta}_\lambda^{S(\ell)}}, \nu) - S_\lambda(\hat{\mu}_{\hat{\theta}_\lambda^{S(\ell)}}, \hat{\nu}) \nonumber \\ & + S_\lambda(\hat{\mu}_{\hat{\theta}_\lambda^{S(\ell)}}, \hat{\nu}) - S_\lambda^{(\ell)}(\hat{\mu}_{\hat{\theta}_\lambda^{S(\ell)}}, \hat{\nu}) + S_\lambda^{(\ell)}(\hat{\mu}_{\hat{\theta}_\lambda^{S(\ell)}}, \hat{\nu}) - S_\lambda^{(\ell)}(\hat{\mu}_{\theta^{*}}, \hat{\nu}) \nonumber \\ 
		& +  S_\lambda^{(\ell)}(\hat{\mu}_{\theta^*}, \hat{\nu}) -  S_\lambda(\hat{\mu}_{\theta^*}, \hat{\nu}) + S_\lambda(\hat{\mu}_{\theta^*}, \hat{\nu}) - S_\lambda(\mu_{\theta^*}, \nu)   \nonumber \\
		& +S_\lambda(\mu_{\theta^*}, \nu) -  W_0(\mu_{\theta^*}, \nu).
	\end{align}
	Let us focus on the right hand side of this last equation \eqref{eq:sinkdiv_decomposition_fullsampling}.
	The first and the last differences are controlled by the approximation error $\sup_{\theta \in \Sigma_K}|S_\lambda(\mu_{\theta}, \nu) - W_0(\mu_{\theta}, \nu)|$. The second and sixth differences can be upper by the estimation error ${\sup_{\theta \in \Sigma_K}|S_\lambda(\mu_{\theta}, \nu) - S_\lambda(\hat{\mu}_{\theta}, \hat{\nu})|}$. The third and fifth differences are upper bounded by the algorithm error ${\sup_{\theta \in \Sigma_K}|S_\lambda^{(\ell)}(\hat{\mu}_{\theta}, \hat{\nu}) - S_\lambda(\hat{\mu}_{\theta}, \hat{\nu})|}$.\\
	
	It only remains to control $ S_\lambda^{(\ell)}(\hat{\mu}_{\hat{\theta}_\lambda^{S(\ell)}}, \hat{\nu}) - S_\lambda^{(\ell)}(\hat{\mu}_{\theta^{*}}, \hat{\nu})$. However, $\hat{\theta}_\lambda^{S(\ell)}$ minimizes the function $\theta \mapsto S_\lambda^{(\ell)}(\hat{\mu}_{\theta}, \hat{\nu})$. Hence $S_\lambda^{(\ell)}(\hat{\mu}_{\hat{\theta}_\lambda^{S(\ell)}}, \hat{\nu}) - S_\lambda^{(\ell)}(\hat{\mu}_{\theta^{*}}, \hat{\nu}) \leq 0$.\\
	
	Going back to equation \eqref{eq:sinkdiv_decomposition_fullsampling} and substituting every difference of the right hand side by its appropriate bound we derive
	\begin{align*}
		W_0(\mu_{\hat{\theta}_\lambda^{S(\ell)}}, \nu) -W_0(\mu_{\theta^*}, \nu) \leq & ~ 2\sup_{\theta \in \Sigma_K}|W_0(\mu_{\theta}, \nu) - S_\lambda(\mu_{\theta}, \nu)|\\
		& + 2 \sup_{\theta \in \Sigma_K}|S_\lambda(\mu_{\theta}, \nu) - S_\lambda(\hat{\mu}_{\theta}, \hat{\nu})| \\ 
		& + 2\sup_{\theta \in \Sigma_K}|S_\lambda(\hat{\mu}_{\theta}, \hat{\nu}) - S_\lambda^{(\ell)}(\hat{\mu}_{\theta}, \hat{\nu})|,
	\end{align*}
	
	which is the result claimed in Lemma \ref{lem:decomposition_sinkdiv_error}.

\end{proof}

\begin{prop}\label{prop:estimation_sinkdiv}Set $\lambda >0$ and suppose that every probability measures have compact supports included in $B(0,R)$. Then, denoting by $\mathcal{E}$ the quantity introduced in equation \eqref{eq:constant_chizat}, the following inequalities hold true.
	\begin{enumerate}[(i)]
		\item If $n$ sample are available from $\nu$, then
		\begin{equation*}
			\mathbb{E}\left[ \sup_{\theta \in \Sigma_K}|S_\lambda(\mu_\theta, \nu) - S_\lambda(\mu_\theta, \hat{\nu}) \right] \lesssim \mathcal{E}(n,d).
		\end{equation*} 
		\item If for every probability measure $\mu_k$, $m_k$ observations are available, then
		\begin{equation*}
			\mathbb{E}\left[ \sup_{\theta \in \Sigma_K}|S_\lambda(\mu_\theta, \hat{\nu}) - S_\lambda(\hat{\mu}_\theta, \hat{\nu}) \right] \lesssim K \mathcal{E}(\underbar{m},d),
		\end{equation*}
		where $\underbar{m}=\min(m_1,\ldots, m_K)$.
	\end{enumerate}	
	
\end{prop}

\begin{proof}
	For concision, we only proof point $(i)$.
	Set $\theta \in \Sigma_K$, and using the definition of $S_\lambda$ we write
	\begin{align*}
		|S_\lambda(\mu_\theta, \nu) - S_\lambda(\mu_\theta, \hat{\nu})|& = |W_{\lambda}(\mu_\theta, \nu) - W_{\lambda}(\mu_\theta, \hat{\nu}) + \frac{1}{2}\left( W_\lambda(\hat{\nu}, \hat{\nu}) - W_\lambda(\nu, \nu) \right)| \\
		& \leq | W_\lambda(\mu_\theta, \nu) - W_\lambda(\mu_\theta, \hat{\nu})| \\
		& + \frac{1}{2}\left(|W_\lambda(\hat{\nu}, \hat{\nu}) - W_\lambda(\nu, \hat{\nu})| + |W_\lambda(\nu, \hat{\nu}) - W_\lambda(\nu, \nu) |\right).
	\end{align*}
	From last inequality we deduce 
	\begin{align}
		\mathbb{E}\left[ \sup_{\theta \in \Sigma_K}|S_\lambda(\mu_\theta, \nu) - S_\lambda(\mu_\theta, \hat{\nu})| \right] & \leq \mathbb{E}\left[ \sup_{\theta \in \Sigma_K}|W_\lambda(\mu_\theta, \nu) - W_\lambda(\mu_\theta, \hat{\nu})| \right] \\
		& + \frac{1}{2}\left( \mathbb{E}\left[|W_\lambda(\hat{\nu}, \hat{\nu}) - W_\lambda(\nu, \hat{\nu})|\right] + \mathbb{E} \left[ |W_\lambda(\nu, \hat{\nu}) - W_\lambda(\nu, \nu) |\right]\right). \nonumber
	\end{align}
	The first term is upper bounded, up to a multiplicative constant, by $\mathcal{E}(n,d)$ thanks to point (i) of Proposition \ref{prop:estimation_regfree}. To control the second term, we apply point (ii) of Proposition \ref{prop:estimation_regfree} in the case $K=1$ and $\mu_1 = \nu$. Hence the second term is also upper bounded by $\mathcal{E}(n,d)$.  The last term is upper bounded by the same quantity with a similar reasoning. From this we deduce
	$$
	\mathbb{E}\left[ \sup_{\theta \in \Sigma_K}|S_\lambda(\mu_\theta, \nu) - S_\lambda(\mu_\theta, \hat{\nu})| \right] \lesssim \mathcal{E}(n,d),
	$$
	as announced in Proposition \ref{prop:estimation_sinkdiv}.
\end{proof}

Under ad hoc assumptions, we can now control the expected excess risk of the estimator $\hat{\theta}_\lambda^{S(\ell)}$ defined by problem \eqref{eq:estimate_Sinkdiv_algo}.

\begin{thm}\label{thm:sinkdiv_excess_bound}
	Set $\lambda > 0$. Suppose that all probability measures, that are $\mu_1, \ldots, \mu_K$ and $\nu$, have compact supports; and that Assumptions \ref{hyp:fisher_component_mixture} and \ref{hyp:bound_fisher_information_geodesic} hold true. Assume that for all component $\mu_k$, as well as for $\nu$, at least $n$ observations are available. Then, the expected excess risk of the estimator $\hat{\theta}_\lambda^{S(\ell)}$ introduced in equation \eqref{eq:estimate_Sinkdiv_algo} is upper bounded as follows:
	$$ \mathbb{E}\left[W_0(\mu_{\hat{\theta}_{\lambda}^{S(\ell)}}, \nu) -W_0(\mu_{\theta^*}, \nu) \right]   \lesssim \mathcal{E}(n,d) + \lambda^2 + \frac{1}{\lambda \ell},$$
	where the quantity $\mathcal{E}(n,d)$ is defined by formula \eqref{eq:constant_chizat}.
\end{thm}

Thanks to the upper bound established in Theorem \ref{thm:sinkdiv_excess_bound}, we can propose a choice of the parameters $\lambda$ and $\ell$. For concision, we only address the case $d>4$.

\begin{cor} Suppose that all probability measures have compact supports included in $B(0,R)$ and that Assumptions \ref{hyp:fisher_component_mixture} and \ref{hyp:bound_fisher_information_geodesic} hold true. If at least $n$ samples are available for each probability measure, setting $\lambda_n = n^{-1/d}$ and $\ell_n = 64R^4n^{3/d}$, the estimator $\hat{\theta}_{\lambda_n}^{S(\ell_n)}$ defined in equation \ref{eq:estimate_Sinkdiv_algo} admits the following rate of convergence:
	
	$$ \mathbb{E}\left[W_0(\mu_{\hat{\theta}_{\lambda}^{S(\ell)}}, \nu) -W_0(\mu_{\theta^*}, \nu) \right]   \lesssim n^{-2/d}.$$
\end{cor}

\begin{proof}[Proof of Theorem \ref{thm:sinkdiv_excess_bound}]
	From Lemma \ref{lem:decomposition_sinkdiv_error},  we have the inequality
	\begin{align*}
		W_0(\mu_{\hat{\theta}_\lambda^{S(\ell)}}, \nu) -W_0(\mu_{\theta^*}, \nu) \leq & ~ 2\sup_{\theta \in \Sigma_K}|W_0(\mu_{\theta}, \nu) - S_\lambda(\mu_{\theta}, \nu)|\\
		& + 2 \sup_{\theta \in \Sigma_K}|S_\lambda(\mu_{\theta}, \nu) - S_\lambda(\hat{\mu}_{\theta}, \hat{\nu})| \\ 
		& + 2\sup_{\theta \in \Sigma_K}|S_\lambda(\hat{\mu}_{\theta}, \hat{\nu}) - S_\lambda^{(\ell)}(\hat{\mu}_{\theta}, \hat{\nu})|.
	\end{align*}
	Taking the expectation on both sides of this last inequality yields
	\begin{align*}
		\mathbb{E}\left[W_0(\mu_{\hat{\theta}_\lambda^{S(\ell)}}, \nu) -W_0(\mu_{\theta^*}, \nu)\right] \leq & ~ 2\sup_{\theta \in \Sigma_K}|W_0(\mu_{\theta}, \nu) - S_\lambda(\mu_{\theta}, \nu)|\\
		& + 2 \mathbb{E}\left[ \sup_{\theta \in \Sigma_K}|S_\lambda(\mu_{\theta}, \nu) - S_\lambda(\hat{\mu}_{\theta}, \hat{\nu})| \right] \\ 
		& + 2\sup_{\theta \in \Sigma_K}|S_\lambda(\hat{\mu}_{\theta}, \hat{\nu}) - S_\lambda^{(\ell)}(\hat{\mu}_{\theta}, \hat{\nu})|.
	\end{align*}
	
	Thanks to Assumptions \ref{hyp:fisher_component_mixture} and \ref{hyp:bound_fisher_information_geodesic}, we can exploit corollary  \ref{cor:uniform_control_approximation_S_lambda} to derive $$\sup_{\theta \in \Sigma_K}|W_0(\mu_{\theta}, \nu) - S_\lambda(\mu_{\theta}, \nu)| \leq M_I \lambda^2.$$
	
	Then, the triangular inequity yields
	\begin{align*}
		\mathbb{E}\left[ \sup_{\theta \in \Sigma_K}|S_\lambda(\mu_{\theta}, \nu) - S_\lambda(\hat{\mu}_{\theta}, \hat{\nu})| \right] \leq & ~ \mathbb{E}\left[ \sup_{\theta \in \Sigma_K}|S_\lambda(\mu_{\theta}, \nu) - S_\lambda(\mu_{\theta}, \hat{\nu})| \right] \\ & + \mathbb{E}\left[ \sup_{\theta \in \Sigma_K}|S_\lambda(\mu_{\theta}, \hat{\nu}) - S_\lambda(\hat{\mu}_{\theta}, \hat{\nu})| \right].
	\end{align*}
	
	Under the assumption that $n$ samples are available from each probability measure, Proposition \ref{prop:estimation_sinkdiv} ensures the control 
	$$
	\mathbb{E}\left[ \sup_{\theta \in \Sigma_K}|S_\lambda(\mu_{\theta}, \nu) - S_\lambda(\hat{\mu}_{\theta}, \hat{\nu})| \right] \lesssim K \mathcal{E}(n,d).
	$$
	Finally, to control the algorithm error, we write 
	\begin{align*}
		|S_\lambda(\hat{\mu}_{\theta}, \hat{\nu}) - S_\lambda^{(\ell)}(\hat{\mu}_{\theta}, \hat{\nu})|& = |W_\lambda(\hat{\mu}_{\theta}, \hat{\nu}) - \frac{1}{2}(W_\lambda(\hat{\mu}_{\theta}, \hat{\mu}_\theta) +  W_\lambda(\hat{\nu}, \hat{\nu})) \\
		&- W_\lambda^{(\ell)}(\hat{\mu}_{\theta}, \hat{\nu}) + \frac{1}{2}(W_\lambda^{(\ell)}(\hat{\mu}_{\theta}, \hat{\mu}_\theta) +  W_\lambda^{(\ell)}(\hat{\nu}, \hat{\nu})) | \\
		& \leq |W_\lambda(\hat{\mu}_{\theta}, \hat{\nu}) - W_\lambda^{(\ell)}(\hat{\mu}_{\theta}, \hat{\nu})| \\
		& +  \frac{1}{2}|W_\lambda(\hat{\mu}_{\theta}, \hat{\mu}_\theta) - W_\lambda^{(\ell)}(\hat{\mu}_{\theta}, \hat{\mu}_\theta)|  + \frac{1}{2}| W_\lambda(\hat{\nu}, \hat{\nu}) -  W_\lambda^{(\ell)}(\hat{\nu}, \hat{\nu})|.
	\end{align*}
	
	As we work with the squared euclidean distance as a ground cost function, from Proposition \ref{prop:sinkhorn_algorithm_error} we deduce that each term in the last inequality is upper bounded by $16 R^4 (\lambda \ell)^{-1}$. Hence, the algorithm error is controlled as follows:
	\begin{equation*}
		\sup_{\theta \in \Sigma_K}|S_\lambda(\hat{\mu}_{\theta}, \hat{\nu}) - S_\lambda^{(\ell)}(\hat{\mu}_{\theta}, \hat{\nu})| \leq \frac{32 R^4}{\lambda \ell}.
	\end{equation*}
	Gathering the pieces together, we derive \\
	$$
	\mathbb{E}\left[W_0(\mu_{\hat{\theta}_\lambda^{S(\ell)}}, \nu) -W_0(\mu_{\theta^*}, \nu)\right]  \lesssim \lambda^2 + \mathcal{E}(n,d) + \frac{1}{\lambda \ell},
	$$
	as claimed in Theorem \ref{thm:sinkdiv_excess_bound}.
	
\end{proof}

\end{document}